\documentclass[mnsc]{informs3_homepage}
\OneAndAHalfSpacedXI
\usepackage{mystylefile}

\usepackage{caption}
\usepackage{graphicx}
\numberwithin{equation}{section}
\usepackage{algorithm,algpseudocode}

\usepackage{natbib}
 \bibpunct[, ]{(}{)}{,}{a}{}{,}%
 \def\bibfont{\small}%
\TheoremsNumberedThrough     % Preferred (Theorem 1, Lemma 1, Theorem 2)
\ECRepeatTheorems

\EquationsNumberedBySection % (1.1), (1.2), ...
\providecommand{\lowh}{\underline{h}}

\MANUSCRIPTNO{}
\begin{document}
%%%%%%%%%%%%%%%%

%\setcounter{page}{1}
%\DOI{10.1287/xxxx.0000.0000}%

\RUNAUTHOR{Mohsen Bayati, Junyu Cao, Wanning Chen}

 \RUNTITLE{Speed Up the Cold-Start Learning in Two-Sided Bandits with Many Arms}
\RUNTITLE{Speed Up the Cold-Start Learning in Two-Sided Bandits with Many Arms}

% \TITLE{Bundling Information Goods of Decreasing Value}
\TITLE{Speed Up the Cold-Start Learning in Two-Sided Bandits with Many Arms\thanks{Alphabetical author order}
}


\ARTICLEAUTHORS{%
	\AUTHOR{Mohsen Bayati}
	\AFF{
		Graduate School of Business, Stanford University, \EMAIL{bayati@stanford.edu}}
    \AUTHOR{Junyu Cao}
	\AFF{
		McCombs School of Business, The University of Texas at Austin, \EMAIL{junyu.cao@mccombs.utexas.edu}}
  \AUTHOR{Wanning Chen}
	\AFF{Foster School of Business, University of Washington, \EMAIL{wnchen@uw.edu}}

	} % end of the block

\ABSTRACT{Multi-armed bandit (MAB) algorithms are efficient approaches to reduce the opportunity cost of online experimentation and are used by companies to find the best product from periodically refreshed product catalogs. However, these algorithms face the so-called \emph{cold-start} at the onset of the experiment due to a lack of knowledge of customer preferences for new products, requiring an initial data collection phase known as the \emph{burn-in period}. During this period, standard MAB algorithms operate like randomized experiments, incurring large burn-in costs which scale with the large number of products. We attempt to reduce the burn-in by identifying that many products can be cast into \emph{two-sided products}, and then naturally model the rewards of the products with a matrix, whose rows and columns represent the two sides respectively. Next, we design \emph{two-phase} bandit algorithms that first use subsampling and low-rank matrix estimation to obtain a substantially smaller targeted set of products and then apply a \textsf{UCB} procedure on the target products to find the best one. We theoretically show that the proposed algorithms lower costs and expedite the experiment in cases when there is limited experimentation time along with a large product set. Our analysis also reveals three regimes of long, short, and ultra-short horizon experiments, depending on dimensions of the matrix. Empirical evidence from both synthetic data and a real-world dataset on music streaming services validates this superior performance.
}%

% \HISTORY{This paper was first submitted on April 12, 1922 and has been with the authors for 83 years for 65 revisions.}

\maketitle
%%%%%%%%%%%%%%%%%%%%%%%%%%%%%%%%%%%%%%%%%%%%%%%%%%%%%%%%%%%%%%%%%%%%%%

%\tableofcontents

\section{Introduction}\label{sec:intro}

Online experimentation has become a major force in assessing different versions of a product, finding better ways to engage customers and generating more revenue. These experiments happen on a recurring basis with periodically refreshed product catalogs. One family of online experimentation is the so-called multi-armed bandit (MAB) algorithms, where each arm models a version of the product. At each time period, the decision-maker tests out a version, learns from the observed outcome, and proceeds to the next time period. It is well-known that MAB algorithms balance exploration and exploitation to handle the lack of prior information at the onset of an experiment. They can dynamically allocate experimentation effort to versions that are performing well, in the meantime allocating less effort to versions that are underperforming. As a result, MAB algorithms can identify the optimal product under relatively low opportunity costs incurred from selecting suboptimal products. Consequently, fewer customers are subjected to potentially inferior experiences. We want to highlight that what is treated as product is not confined to merchandise in the e-commerce setting. It includes broader categories from ads campaign product in online advertisement to drug product in drug development.

At the onset of the experimentation, MAB algorithms face what is known as the \emph{cold-start problem} when there is a large product set: many versions of a product are brand new to the customers. It is not yet clear to the algorithms how to draw any useful inferences before gathering some information. Therefore, many MAB algorithms involve an initial data collection phase, known as the \emph{burn-in period} \citep{burning}\footnote{This term is well-known in MCMC literature as data-collection period \citep{brooks2011handbook} and hence practitioners also use it for data-collection period in MAB experiments.}. During this period, MAB algorithms operate as nearly equivalent to randomized experiments, incurring costs that scale with the number of products. Under such circumstances, the cumulative cost, or regret, of current practices is huge due to a costly and lengthy burn-in period\footnote{It should be noted that, while standard MAB algorithms are recognized for \emph{addressing} cold-start problems in certain studies \citep{ye2023cold} through effective exploration and exploitation, this capability is contingent on a relatively short \emph{burn-in period} compared to the total experimentation horizon. As we have discussed, given a substantial number of arms and a constrained time horizon, standard MAB algorithms cannot manage the cold-start problem effectively.}. 

Our goal is to come up with a framework that can be cost-effective, which is achieved by first identifying a two-sided structure of the products. In many applications, because the products have two major components, the bandit arms can be viewed as versions of a \emph{two-sided product}. To be more specific, let us provide two examples that will reappear in the later empirical studies.

\begin{example}
Consider an online platform that wants to select one user segment and one content creator segment that interact the most to target an advertisement campaign \citep{bhargava2022creator,lops2011content,geng2020online}. Such advertisement campaign is a two-sided product where user segment and content creator segment are its two sides. There are many user segments and many content creator segments. In each time period, the platform tests a pair of user segment and content creator segment and observes the amount of interaction between them, based on which the platform decides which pair to test next. The number of products scales quickly: ten user segments and ten content creator segments already give rise to one hundred pairs, and this is just an underestimate of how many segments there are for each side.
\label{ex:ads}
\end{example}

\begin{example}\label{ex: homepage}
Consider an online platform that wants to design a homepage with a headline and an image. The platform's goal is to select a pair of headline and image that attracts the most amount of user engagement. The homepage design is a two-sided product with headline and image as two sides. Similar to Example \ref{ex:ads}, the number of pairs can easily be over several hundreds.
\end{example}

Besides the two examples above, there are plenty of other two-sided products: Airbnb experiences (activities and locations), Stitch Fix personal styling (tops and pants), Expedia travel (hotels and flights), car sales (packages and prices), drug development (composition of ingredients and dosages) and so on. 

The task of finding the best version of a two-sided product can be translated into finding the maximum entry of a partially observed reward matrix, which can help accelerate the learning process. That is, we can model different choices of one side to be rows of the matrix, different choices of the other side to be the columns, and the unknown rewards for the corresponding row-column pairs (e.g. the amount of interaction in Example~\ref{ex:ads} and the amount of user engagement in Example~\ref{ex: homepage}) to be the matrix entries. Mathematically, suppose the first side of the product has $d_r$ number of choices and the second side has $d_c$ number of choices, then this matrix contains the rewards of $d_rd_c$ versions of the two-sided product as entries. As aforementioned, current algorithms will randomize on the initial $\Omega(d_rd_c)$ number of periods, leading to a very costly burn-in period as $d_r$ and $d_c$ are often very large. Furthermore, when the time horizon $T$ is rather small compared to the order $d_rd_c$, a large cumulative regret that scales with $d_rd_c$ would arise because of the costly burn-in period. 

By casting the rewards into a matrix and considering its low-rank structure, we can expedite the experiment and reduce the regret by orders of magnitudes. To understand such a structure, we note that the rewards for different versions of a two-sided product depend on the interactions between some latent features of the two sides. Low-rank simply means that we only need as few as $\rrank$ latent features to explain each side, such that $\rrank\ll\min(d_r,d_c)$, where $\rrank$ is the rank of the matrix. Because of such low-rank property and with the help from low-rank matrix estimators, we only need to collect very few data to have an estimate of the rewards for all arms of interest in the burn-in period that is accurate enough to help screen out majority of the sub-optimal arms, and hence requiring substantially less downstream exploration which greatly reduces the burn-in cost.
\subsection{Our Contributions}
\label{sec:contribution}
\textbf{Algorithm.} Specifically, we have designed a new bandit algorithm called the Low-Rank Bandit (\textsf{LRB}). The algorithm can efficiently choose from a large set of two-sided products given a short time horizon. It is split into two phases: ``pure exploration phases" and ``targeted exploration + exploitation phases". At each period, the algorithm follows a schedule to enter either of the two phases. If it is in a pure exploration phase, a forced sample is selected uniformly at random and is used for the low-rank matrix estimation (i.e., the forced-sample estimates). The forced-sample estimates are used to screen out arms that are far from the optimal arm. The remaining arms form a targeted set whose size is much smaller than the initial product set. If the algorithm is in a targeted exploration + exploitation phase, we choose a product based on an off-the-shelf stochastic multi-armed bandit algorithm (e.g. \textsf{UCB}, Thompson Sampling) applied on the current targeted set as a subroutine. For illustration purpose, we use \textsf{UCB} as the subroutine throughout the paper.

Moreover, we propose Submatrix-Sampled Low-Rank Bandit (\textsf{ss-LRB}) that adds a subsampling pre-step to our regular \textsf{LRB} algorithm when we are extremely sensitive about time. Given an ultra-short horizon (formally defined in Section \ref{sub:short_horizon}) in the face of a huge number of arms, we can instead sample a submatrix of smaller dimensions and apply \textsf{LRB} to it. The intuition behind this is as follows. Even though the best entry in the submatrix might be suboptimal, because the horizon is ultra-short, the cost we save from not exploring the much larger matrix compensates for the regret we incur from picking a suboptimal entry.

\textbf{Theory.} We establish non-asymptotic dependent and independent regret bounds of \textsf{LRB} as functions of a parameter we term as ``filtering resolution" and demonstrate how to optimize such bounds. Simply put, the filtering resolution controls the number of arms that will enter the targeted exploration + exploitation phase. Specifically, we show that our algorithm achieves a strictly better regret than the typical $O(\sqrt{d_r d_c T})$ for standard non-contextual bandits when $T$ is small. By leveraging a structure of the reward distribution in the low-rank matrix characterized by a ``shrinkage rate", we show a further improvement on the regret. For the \textsf{ss-LRB} algorithm we propose, we derive the optimal subsampling ratio and an improved regret bound, both of which depend on the form of a ``subsampling cost" function. Our theory sheds light on customized strategies given different time horizon lengths and the product set sizes, distinct from most existing bandit theoretical analysis which study dependence of the regret on the time horizon asymptotically and may not perform well when the time horizon is short.

\textbf{Empirics.} Empirical evidence from synthetic data validates the superior performance of our algorithm. In addition, we illustrate the practical relevance of our algorithms by proposing a data-driven approach to select experiment-dependent hyperparameters. We evaluate such an approach on an advertisement targeting problem for the music streaming service, where we want to learn the best combination of user group and creator group to target an advertisement campaign. We show that our algorithm significantly outperforms existing bandit algorithms in finding the best combination efficiently under various time horizon lengths. 

\textbf{A reduction model for contextual bandits.}
We propose a new modeling framework to transform a stochastic linear bandit setting (a general class of bandit problems that contains contextual bandits) to the scope of our low-rank bandit problem. This allows us to empirically showcase \textsf{LRB}'s effectiveness of reducing the burn-in in contextual settings by comparing it with a standard linear bandit method called \textsf{OFUL} \citep{abbasi2011improved}.
\subsection{Related Literature}\label{sec:lit}
We will first give an overview of the bandit literature since our work addresses the setting of online decision making under bandit feedback. Next, because our work evolves around low-rank models and matrix-shaped data, we will introduce relevant literature to those topics as well.

Our work focuses on the regret minimization objective for the non-contextual bandit problems, where algorithms such as Thompson Sampling \citep{thompson1933likelihood,agrawal2012analysis,gopalan2014thompson,russo2018tutorial} and Upper Confidence Bound \citep{lai1987adaptive,katehakis1995sequential,kaelbling1993learning,auer2002finite} are some canonical works. Like ours, these algorithms explore and exploit without contextual information, i.e., they make no assumptions on dependence of the arm rewards. There is also a large literature on bandit problems with contextual information, when the arm rewards are parametric function of the observed contexts. Linear bandit \citep{auer2002nonstochastic,dani2008stochastic,rusmevichientong2010linearly,li2010contextual,chu2011contextual,agrawal2014bandits, abbasi2011improved} and its extension under generalized linear models (GLM) \citep{filippi2010parametric,li2017provably,kveton2020randomized} have been under the spotlight of many works. As the name of linear bandit suggests, the arm rewards are linear functions of the observed contexts. We refer the readers to \cite{lattimore2020bandit} for more information on bandit algorithms. 

Classical theoretical analysis of the non-contextual bandits focuses on asymptotic behavior of the regret bound under a long enough experimentation horizon. That is, the number of arms is generally assumed to be much smaller than the horizon. Specifically, in the two-sided bandit scenario we are considering, we have $d_rd_c$ number of arms, which leads to $O(\sqrt{d_rd_cT})$ regret upper bound by standard non-contextual bandit analysis.

However, as aforementioned, under a short horizon, existing bandit algorithms perform very poorly when there are many arms due to the linear dependence of the regret incurred at the beginning of the experiment on the number of arms (i.e., each arm needs to be pulled at least once). Such a cold-start problem is commonly seen in online experimentation when there are many arms and a short experimentation horizon. Examples of existing works to overcome the cold-start problem involve combining clustering and MAB algorithms \citep{miao2022context,keskin2024data}, warm starting the bandit problem using historical data \citep{banerjee2022artificial}, or using an integer program to constrain the action set \cite{bastani2022learning}.

A separate stream of literature studying many-armed or infinitely-armed bandit problems \citep{berry1997bandit,bonald2013two} has addressed this cold-start problem by using subsampling to reduce the opportunity cost \citep{wang2009algorithms, carpentier2015simple,chaudhuri2018quantile,bayati2020unreasonable}. The main differentiating factor between our work and this line of literature is that the inherent matrix structure of the two-sided products
allows us to make a more informed selection of the arms to use for targeted exploration + exploitation. An alternative solution to tackle infinitely-armed bandit problems is to assume that the reward function is a linear function of actions that belong to a compact set, e.g., as in \citep{mersereau2009structured,rusmevichientong2010linearly}. In such cases, pulling \emph{any action} provides information for reward function of \emph{all actions}.

Low-rank matrix estimators have become established tools for solving offline cold-start problems in recommendation systems \citep{lika2014facing,zhang2014addressing,volkovs2017dropoutnet}, and they are building blocks for many applications that have inherent low-rank models \citep{athey2021matrix, farias2022uncertainty, agarwal2023causal}. For a survey paper on low-rank matrix estimators, we refer the readers to \cite{davenport2016overview} and \cite{yuxin} for a detailed discussion.

Incorporating the low-rank matrix structure in bandit settings has been seen in many other works, which all consider different settings than ours. For example, \cite{katariya2017stochastic, trinh2020solving} consider only rank-1 matrices whereas our algorithm works for general low-rank matrices; \cite{kveton2017stochastic} allow a decision-maker to noisily choose and observe every entry of a $\rrank \times \rrank$ submatrix (where $\rrank$ is the rank of the underlying unknown matrix) in each period, but we only allow the selection of one entry each time and we assume the rank $\rrank$ is unknown, which is a harder problem. \cite{lu2018efficient} provide an ensemble sampling based algorithm which lacks theoretical guarantees. \cite{jun2019bilinear} address the contextual bilinear bandit setting where reward function is a bilinear function of two feature vectors and an unknown low-rank parameter matrix. They propose an \textsf{OFUL}-based algorithm that combines with low-rank matrix estimation. \cite{lu2021low} propose a similar algorithm for low-rank generalized linear models. The \textsf{ESTR} algorithm in \cite{jun2019bilinear} and \textsf{LowESTR} algorithm in \cite{lu2021low} use a novel way of exploiting the subspace so as to reduce the problem to linear bandits after exploration. They consider a more general setting and their algorithms can be adapted to solve our problems, by taking the canonical basis vectors as the two feature vectors. We will show in Appendix \ref{app:comparison} that, when narrowing down to the specific non-contextual setting we focus on, our algorithm has better theoretical and empirical performance. \cite{lu2021low} in addition propose an exponentially weighted average forecaster-based algorithm, which is statistically efficient but not computationally efficient as they have put it. \cite{kallus2020dynamic} impose a low-rank structure to their underlying parameter matrix for the dynamic assortment personalization problem in which a decision-maker picks a best subset of products for a user. \cite{hamidi2019personalizing} also impose a low-rank structure on the arm parameter set to expedite the learning process, but in their setting, contexts are observed which is different from ours. \cite{nakamura2015ucb} proposed a heuristic \textsf{UCB}-like strategy of collaborative filtering for a recommendation problem that selects the best user-item pairs, which they called the direct mail problem, but no theoretical guarantee for that approach is known. \cite{sam2023overcoming} study a class of Markov decision processes that exhibits latent low-dimensional structure with respect to the relationship between large number of states and actions. \cite{zhu2022learning} consider learning the Markov transition matrix under the low-rank structure. \cite{jain2022online} propose an explore-then-commit (ETC) approach for the online low-rank matrix completion problem, where the algorithm recommends one item per user at each round. \cite{zhou2024stochastic} propose low-rank tensor bandit algorithms for tensors that are at least three-dimensional. For more low-rank bandit works, we refer the readers to the literature sections of these related works. 

Though the low-rank structure is utilized in all the above-mentioned works, we want to emphasize that our work focuses on a short-horizon learning problem with many arms, and the purpose of utilizing the low-rank property is to efficiently filter suboptimal arms in order to construct a more focused targeted set, under very limited experimentation time.

There are recent papers such as \citep{gupta2021small,gupta2022data} and \citep{allouah2023optimal,besbes2023big} that focus on practical small data regime with limited sample size: the former group studies how to optimize decison-making in a large number of small-data problems, while the latter group studies classical pricing and newsvendor problems and showed how a small number of samples can provide valuable information, leveraging structure of the problems. Even though we focus on the burn-in period with small data, our bandit formulation with two-sided products and low-rank structure is completely different from their offline set-ups. Our problem is also different from the recent work of \cite{xu2021learning} that studies how one can learn across many contextual bandit problems, because in our model the contexts are hidden and will be estimated through learning of the low-rank structure.

Finally, recent works such as \cite{bajari2021multiple} and \cite{johari2022experimental} also look at online experimentation for two-sided settings, but from a completely different angle. Specifically, their objective is to remove bias due to the spillover effect by using two-sided randomization. 

\subsection{Organization of the Paper}
The remainder of the paper is organized as follows. We describe the problem formulation in Section \ref{sec:problem}. We present the \textsf{LRB} and our main result on the algorithm's performance in Section \ref{sec:alg}. In Section \ref{sec:subsamp}, we present the submatrix sampling technique to further enhance \textsf{LRB}'s performance when the time horizon is ultra-short. Finally, empirical results on simulated data as well as our evaluation on real music streaming data for the task of advertisement targeting are presented in Section \ref{sec:data}. It also includes numerical comparison against linear bandits in contextual settings.

\section{Model and Problem}
\label{sec:problem}

In this section, we formulate our problem as a stochastic multi-armed bandit with (two-sided) arms. These arms are represented as row-column pairs of a matrix, and the arm mean rewards are represented as entries of the matrix. We explain what it means for the reward matrix to have a low-rank structure, paving our way to understand the Low-Rank Bandit algorithm. We also introduce some necessary notation and assumptions that will be used throughout the paper. We use bold capital letters (e.g. $\bB$) for matrices and non-bold capital letters for vectors (e.g., $V$, except for $T$, which denotes the total time horizon). For a positive integer $m$, we denote the set of integers $\{1,\ldots,m\}$ by $[m]$. For a matrix $\bB$, $\bB_{ij}$ refers to its entry $(i,j)$ and $\bB_j$ refers to its $j$-th row. $\|\bB\|_{\infty}$ denotes the largest entry of a matrix $\bB$, i.e., $\|\bB\|_{\infty} =\max_{(j,k)\in[d_r]\times [d_c]} |\bB_{jk}|$. $\|\bB\|_{2,\infty}$ denotes the largest $l_2$ norm of all rows of a matrix $\bB$. 
 
\subsection{Rewards as a Low-Rank Matrix}
We work with a $d_r\times d_c$ matrix $\bB^*$ and our goal is to find the entry with the biggest reward as we sequentially observe noisy version of picked entries. As introduced earlier, the rows of this matrix correspond to different choices for the first side of the two-sided products and columns correspond to different choices for the second side. Entries of $\bB^*$ correspond to the mean rewards for different combinations of the two sides and we assume they are bounded. We assume the entries are bounded as in the canonical bandit literature. This assumption is standard and makes sure that the maximum regret at any time step is bounded. This assumption is realistic since user preferences are bounded in practice. 

\begin{assumption}[Boundedness] Assume  $\|\bB^*\|_{\infty}\leq b^*$. Without loss of generality, we further assume $b^*\leq 1$. 
\label{assump:parameter_set}
\end{assumption}

We use the two-sided product in Example \ref{ex:ads}, ads campaign, as a running example. In this case, the marketer needs to pick a pair of user segment and content creator segment to target the ads. Then we can model user segments as rows of matrix $\bB^*$ and content creator segments as columns of matrix $\bB^*$. The amount of interaction between different combinations of user segments and content creator segments are the entries of $\bB^*$. The goal under this scenario is to find the pair that interacts the most to target the ads by experimenting with different pairs. 

We model the ground truth mean reward matrix as a low-rank matrix. That is, let $\rrank$ be the rank of $\bB^*$, we have $\rrank\ll \min\{d_r, d_c\}$. This means the reward for the entry corresponding to row $j$ and column $k$ is a dot product of two $\rrank$-dimensional latent feature vectors. Therefore, our reward model is a contextual reward with unobserved features. In Section \ref{subsec:contextual-model}, we show an alternative contextual bandit representation of our reward function that is motivated by the standard linear bandit setting framework.

Motivated by the literature on low-rank matrix completion, we introduce the incoherence parameter, which is known to be crucial for reliable recovery of the matrix \citep{keshavan2009matrix,candes2012exact,chen2015incoherence, yuxin}. It quantifies how much information observing an entry gives about other entries in a matrix. 
\begin{definition}
\label{def:row-inc}
For a matrix $\bB^*\in \mathbb{R}^{d_r\times d_c}$ of rank $\rrank$ with singular value decomposition $\bB^* = \bU^*\bD^*\bV^{*\top}$, we denote the \emph{incoherence} parameter as $\mu:=\mu(\bB^*)$ such that $\|\bU^*\|_{2,\infty} \leq \sqrt{\mu \rrank/d_r}$ and $\|\bV^*\|_{2,\infty}\leq \sqrt{\mu \rrank/d_c}$.
\end{definition}
Intuitively, smaller $\mu$ means that we can get more information about other products from observing just a few entries, and the matrix completion problem is more tractable. Consequently, less experimentation cost is incurred to explore among a large product set. In contrast, matrices with bigger $\mu$ are those which have most of their mass in a relatively small number of elements and are much harder to recover under limited observations. 

\begin{definition}\label{definition: condition number}
    For a matrix $\bB$ of rank $\rrank$ with $\rrank$ non-zero singular values $\sigma_{\max} = \sigma_1 > \sigma_2> \cdots > \sigma_{\rrank} = \sigma_{\min}$, we denote the \emph{condition number} as $\kappa:=\kappa(\bB)$ such that $\kappa = \sigma_{\max}/\sigma_{\min}.$
\end{definition}

Both $\mu$ and $\kappa$ are assumed to be constants.

\subsection{Bandit under the Matrix-Shaped Rewards}
Consider a $T$-period horizon for the experimentation; at each time step $t$, the marketer has access to $d_rd_c$ arms (ads campaigns) and each arm yields an uncertain reward (amount of interaction). At time $t$, if the marketer pulls arm $(j_t,k_t)\in [d_r]\times [d_c]$, it yields reward $ y_t :=\bB^*_{j_tk_t} + \epsilon_{t}\,,$ where we assume the observation noises to be independent sub-Gaussian
random variables with sub-Gaussian norm at most $\sigma$ (see Definition 5.7 in \cite{vershynin2010introduction}). Natural examples of sub-Gaussian random variables include any centered and bounded random variable, or a centered Gaussian.

Our task is to design a sequential decision-making policy $\pi$ that learns the underlying arm rewards (parameter matrix $\bB^*$) over time in order to maximize the expected reward. Let $\pi_t = (j_t,k_t)$, an element of $[d_r]\times [d_c]$, denote the arm chosen by policy $\pi$ at time $t\in [T]$. We compare to an oracle policy $\pi^*$ that already knows $\bB^*$ and thus always chooses the best arm (in expectation), $\pi^* = (j^*,k^*) = \argmax_{(j,k)\in [d_r]\times[d_c]}\bB^*_{jk}$. Thus, if the arm $\pi_t = (j_t,k_t)$ is chosen at time $t$, expected regret incurred is
 $\mathbb{E}[\bB^*_{j^*k^*} -\bB^*_{j_tk_t}],$
where the expectation is with respect to the randomness of $\epsilon_t$ and potential randomness introduced by the policy $\pi$. This is simply the difference in expected rewards of $\pi_t^*$ and $\pi_t$. We seek a policy $\pi$ that minimizes the cumulative expected regret 
\[
\regret_T = \sum_{t=1}^T \mathbb{E}_\pi\left[\bB^*_{j^*k^*} -\bB^*_{j_tk_t}\right]\,.
\]
Note that throughout this paper we assume the matrix $\bB^*$ is deterministic which means our notion of regret is \emph{frequentist}.

Together with policy $\pi$, the observation model can be characterized by a trace regression model \citep{trevor}. For a given subset of cardinality $n$, $\mathcal{I} = \{t_1, \cdots, t_n\}$ of $[T]$, we define $Y = \mathfrak{X}^{\pi}_{\mathcal{I}}(\bB^*) + E\,,$ with observation operator $\mathfrak{X}^{\pi}_{\mathcal{I}}$ (defined below), vector of observed values $Y\in \mathbb{R}^{n}$ equal to $[y_{t_1}, y_{t_2}, \cdots,y_{t_n}]^\top$, and noise vector $E\in \mathbb{R}^{n}$ equal to $[\epsilon_{t_1}, \epsilon_{t_2}, \cdots,\epsilon_{t_n}]^\top$. The observation operator $\mathfrak{X}^{\pi}_{\mathcal{I}}(\cdot)$ takes in a matrix $\bB$ and outputs a vector of dimension $n$ (i.e., $\mathfrak{X}^{\pi}_{\mathcal{I}}:\mathbb{R}^{d_r\times d_c}\rightarrow \mathbb{R}^n$). Elements of the output vector are the entries of $\bB$ at $n$ observed locations, determined by $\pi_{t_1},\cdots,\pi_{t_n}$. Specifically, $\mathfrak{X}^{\pi}_{\mathcal{I}}(\bB)$ is a vector in $\mathbb{R}^n$, and for each $i$ in $[T]$, its $i$-th entry is defined by $\left[\mathfrak{X}^{\pi}_{\mathcal{I}}(\bB)\right]_i := \langle \bB, \bX_{t_i}^{\pi} \rangle\,,$ where \emph{design matrix} $\bX_{t_i}^{\pi}\in \mathbb{R}^{d_r\times d_c}$
has all of its entries equal to $0$, except for the $\pi_{t_i} = (j_{t_i}, k_{t_i})$-th entry, which is equal to $1$. The notation $\langle\cdot, \cdot\rangle$ refers to the trace inner product of two matrices that is defined by $\langle \bB_1, \bB_2\rangle:= \Tr(\bB_1\bB_2^{\top})\,.$ The observation operator will appear in the matrix estimator we use.

\section{Low-Rank Bandit Algorithm}
\label{sec:alg}
We will first give an overview of our two-phase Low-Rank Bandit (\textsf{LRB}) algorithm before describing it in details in Section \ref{subsec:description}. Then in Section \ref{SS. regret analysis}, we carefully incorporate a matrix element-wise error bound borrowed from the literature to guide our regret analysis. In Section \ref{sec:balance}, we further explain how the so-called ``filtering resolution" $h$ balances regrets from the two phases to achieve tightened bounds, with a discussion on how \textsf{LRB} can be effective under different horizon lengths.

Intuitively, our \textsf{LRB} takes advantage of the low-rank reward structure and yields a good estimation of all the arms of interest by only observing a few arms. Thus, we can speed up the cold-start given a relatively short horizon in the face of many arms. As aforementioned, most prior literature operates under a typical but unrealistic formulation of this problem: the set of arms is assumed to be ``small" relative to the time horizon $T$. In particular, in standard asymptotic analysis of the MAB setting, the horizon $T$ scales to infinity while the number of arms remains constant. However, as discussed earlier, the number of arms could be large relative to the time horizon of interest in many practical situations. In such a short horizon, standard MAB algorithms will incur at least $\Omega(d_r d_c)$ regret in the initialization stage, by forced-sampling each arm once. Leveraging the low-rank matrix structure of the arms, our forced-sampling regret becomes $\tilde O(\rrank^2 (d_r+d_c)/h^2)$. By carefully selecting the filtering resolution $h$, we can show that the total regret upper bound can be strictly smaller than that of a standard non-contextual bandit for specified horizon lengths.

As mentioned in Section \ref{sec:contribution}, Low-Rank Bandit algorithm is constructed using a two-phase approach --- it is split into ``pure exploration phases" (also called the first phase) and ``targeted exploration + exploitation phases" (also called the second phase). This approach partially builds on the ideas in \cite{goldenshluger2013linear} that are later extended by \cite{hamidi2019personalizing,bastani2020online}. 
Specifically, those studies centered around the contextual bandit problem. They built two sets of estimators (namely, forced-sample and all-sample estimators) in such a way that the former estimator is used to construct a targeted set that filters sub-optimal arms which fall outside the confidence region of the optimal reward, and the latter estimator is used to select the best arm in the targeted set. The first step which is a pure exploration procedure aims at ensuring that during a constant proportion of time periods, only a single arm is left in the targeted set, so that the algorithm can collect i.i.d. samples in the second phase, despite the fact that the second phase is pure exploitation. 

However, in our study, the two-phase procedure serves a different purpose. More specifically, in the first phase, we utilize the low-rank structure to filter sub-optimal arms that are $h/2$ away from the highest estimated reward by forced-sample estimators. While our first step seems similar to the aforementioned papers, our second phase is itself a \textsf{UCB} algorithm that both explores and exploits the targeted set to find the optimal arm. This distinction results in a different requirement for the first phase: the targeted set does not need to contain just one element; rather, it only needs to be much smaller than $d_rd_c$ to meaningfully reduce the exploration cost of a regular \textsf{UCB} algorithm. This fact results in a different selection criteria for the \resolution{} $h$. As we discuss in detail later, the optimal value of $h$ is determined by both the matrix size and the total time horizon.

Furthermore, we want to emphasize that the two-phase approach in the aforementioned studies cannot be directly applied to our problem. They focused on long enough experimentation time horizons so that their algorithms can take time to learn arms well and leave only the optimal choice in the targeted set for some fixed $h$. However, as aforementioned, we face a short horizon under the many-arm circumstances. In our analysis, we will show that the filtering resolution $h$ controls not only the size of the targeted set, but also the estimation error of using the forced samples. We need smaller $h$ to filter out more arms so as to have a smaller targeted set, which in turn requires more forced-sampling rounds to achieve a good enough estimation so as to be more confident about the suboptimality of the filtered arms. Thus, it is unlikely to construct a targeted set with fine-enough \resolution{} that only contains the optimal choice as the single arm, because forced-sample estimation is not accurate enough under limited experimentation time to achieve that.

Finally, we note that one cannot use the low-rank estimator to construct confidence sets for the \textsf{UCB} algorithm because existing bounds for the low-rank estimators do not work in presence of adaptively collected data. The only exception (to the best of our knowledge) is the all-sampling estimator of \cite{hamidi2019personalizing}, but as discussed above, their approach would not work when the targeted set has more than one arm.

\subsection{Description of the Low-Rank Bandit Algorithm}
\label{subsec:description}

The main idea of our two-phase Low-Rank Bandit algorithm is to reduce the initially large set of arms via a low-rank estimator and pass the smaller set of arms to a \textsf{UCB} algorithm.
Specifically, we \emph{forced-sample} arms at prescribed times and use a low-rank estimator which has a sharp entry-wise error bound to obtain an estimation of each entry of $\bB^*$. Then we use that to select a targeted set of arms that with high probability contains the optimal arm.

The Low-Rank Bandit takes as input the matrix row (column) dimension $d$, the forced-sampling rule $f_{\cdot}$, and a \resolution{} $h > 0$ (discussed further in Section \ref{subsec:h} and Section \ref{sec:phi1andphi2}). The complete procedure is in Algorithm \ref{alg:low-rank}.

\paragraph{Forced-Sampling Rule:} Our forced-sampling rule is $f:\mathbb{N}^+\rightarrow ([d_r]\times [d_c])\cup \{\emptyset\}$. At time $t$, the forced-sampling rule decides between forcing the arm $f_t\in [d_r]\times [d_c]$ to be pulled or exploiting the past data, indicated by $f_t = \emptyset$. By $\mathcal{F}_t$, we denote the time periods that an arm is forced to be pulled, i.e. $\mathcal{F}_t:=\{t'\leq t: f_{t'}\in [d_r]\times [d_c]\}$. The forced-sampling rule that we use is a randomized function that picks an arm $(j,k)\in [d_r]\times [d_c]$ with probability

\begin{equation}
\mathbb{P}\big(f_t = (j,k)\big) = \begin{cases}
\frac{1}{d_rd_c} & \text{ if } t\leq 2\rho\log(\rho), \\
\frac{\rho}{d_rd_c[t-\rho\log(\rho) +1]} & \text{ if } t>2\rho\log(\rho),
\label{forced-sampling}
\end{cases}
\end{equation}
and $f_t = \emptyset$ otherwise. We call $\rho$ the forced-sampling parameter and $2\rho\log(\rho)$ the initialization period which we denote as $t_0$. In Lemma \ref{lemma:forced-sampling-rounds} shown in Appendix \ref{appendix:element-wise}, we prove that the number of forced-sampling rounds $|\mathcal{F}_t|$ is the same order of $O(\rho\log(t))$ with high probability. The choice of $\rho$ will be discussed in Section \ref{subsec:h}.

\paragraph{Estimators:} At any time $t$, \textsf{LRB} maintains two sets of parameter estimates for $\bB^*$:
\begin{enumerate}
    \item the forced-sample estimates $\widehat\bB^F$ based only on forced samples;
    \item the \textsf{UCB} estimator $\widehat\bB^{UCB}_{jk}$ based on all samples observed for each arm $(j,k)$, where $\widehat\bB^{UCB}_{jk} = \bar x_{(j,k)}+\sqrt{2\log w(t)/n_{(j,k)}}$ is a combination of the empirical mean reward estimation $\bar x_{(j,k)}$ and a term $\sqrt{2\log w(t)/n_{(j,k)}}$ that represents the uncertainty of the estimates. $n_{(j,k)}$ is the number of times entry $(j,k)$ has been played so far. The function $w(t)$ is defined by $w(t) = 1+ t\log^2(t)$ for the empirical experiment and $w(t) =t$ for the analysis part for simplicity.
\end{enumerate}

\paragraph{Low-Rank Estimator.} As mentioned before, our forced-sample estimator is based on a low-rank estimator. We use nuclear-norm penalized least square as our example for a low-rank estimator. 

\begin{definition}[Nuclear norm penalized least square]
\label{def:nuclear-norm}
Given a regularization parameter $\lambda\geq 0$, the \emph{nuclear norm penalized least square estimator} is
\[
\widehat\bB = \argmin_{\bB\in \mathbb{R}^{d_r\times d_c}}\left(\frac{1}{|\mathcal{F}_t|}\|Y - \opr^{\pi}_{\mathcal{F}_t}(\bB)\|_2^2 + \lambda \|\bB\|_*\right).
\]
\end{definition}
However, any low-rank estimator that satisfies the following tail bound can be used as our forced-sample estimator. For simplicity of presentation, we set $d_r=d_c=d$. 

\begin{proposition} [Tail bound for low-rank estimators \citep{yuxin}] 
Fix any $\varsigma>0$. By taking $\lambda=C_{\lambda}\sigma \sqrt{1/(nd)}$ where $n$ denotes the number of i.i.d. samples which are sampled uniformly for some large enough constant $C_\lambda>0$, it holds that
\[
\mathbb{P}\left(\|\widehat\bB-\bB^*\|_{\infty}\geq C_\varsigma\sqrt{\kappa^3\mu \rrank}\frac{\sigma}{\sigma_{\min}}\sqrt{\frac{d\log d}{n}}\right)\leq \frac{1}{d^\varsigma}\,,
\] 
when $n\geq C\kappa^4 {\mu}^2 \rrank^2 d{\log^3 d}$, and $C$ and $C_\varsigma$ are positive constants which are independent of $d$ and $n$. 
\label{prop:low-rank-estimators}
\end{proposition}

\begin{remark}
\citep{yuxin} provide the tail bound for square matrices to facilitate a clear presentation. In their paper, they note that it is straightforward to extend their discussions to general rectangular matrices of size $d_r \times d_c$. Similarly, we focus on square matrices for simplicity of presentation, whereas our proposed framework is not limited to square matrices. Given a tail bound for rectangular low-rank matrices of $d_r \times d_c$, all subsequent discussions can be easily extended accordingly.
In addition, Theorem 1 in \citep{yuxin} proves for $\varsigma=3$, but their result can be easily generalized to any large enough $\varsigma\geq 3$.
\end{remark}

\paragraph{Execution:}  We follow a \emph{two-phase} process: if the current time $t$ is in $\mathcal{F}_t$, then we are in the ``pure exploration phases'', and arm $f_t$ is played; otherwise, we are in the ``targeted exploration + exploitation phases", and this phase consists of the following two steps. As a pre-processing step, we use the forced-sample estimates to find the highest estimated reward achievable across all $d_rd_c$ entries. We select the subset of arms $\mathcal{C}\subset [d_r]\times [d_c]$ whose estimated rewards are within $h/2$ of the maximum achievable (in step 6 of Algorithm \ref{alg:low-rank}). After this pre-processing step, we apply \textsf{UCB} on the targeted set $\mathcal{C}$ by selecting the arm that maximizes the \textsf{UCB} value (in step 7 of Algorithm \ref{alg:low-rank}).

\begin{algorithm}
\caption{Low-Rank Bandit}
\label{alg:low-rank}
\begin{algorithmic}[1]
\State 
Input: Forced-sampling rule $f$, filtering resolution $h$
\For{$t = 1, 2,\cdots, T$}
\If{$f_t\neq \emptyset$} ``pure exploration phase"
\State $\pi_t \leftarrow f_t$ (forced-sampling)
\Else ``targeted exploration + exploitation phases"
\State Define the set of near-optimal arms $\mathcal{C}$ according to the forced-sample estimator:
\begin{equation}\label{eq: candidate set}
    \mathcal{C} = \Bigg\{(j,k)\in [d_r]\times [d_c] \mid \widehat \bB_{jk}^F(\mathcal{F}_{t-1})\geq \max_{(l,m)\in [d_r]\times [d_c]}\widehat\bB_{lm}^F(\mathcal{F}_{t-1}) -\frac{h}{2}\Bigg\}.
\end{equation}
\State Choose the best arm within $\mathcal{C}$ according to the \textsf{UCB} estimator, i.e.,
$$\pi_t\leftarrow (j_t, k_t) = \argmax_{(j,k)\in \mathcal{C}} \ \widehat\bB^{UCB}_{jk}\,,
$$
where $\widehat\bB^{UCB}_{jk} = \bar x_{(j,k)}+\sqrt{\frac{2\log w(t)}{n_{(j,k)}}}$ is such that $\bar x_{(j,k)}$ is the average observed value for entry $(j,k)$, and $n_{(j,k)}$ is the number of times entry $(j,k)$ has been played so far.
\EndIf
\State Play arm $\pi_t = (j_t, k_t)$, observe $y_t = \bB^*_{j_tk_t} + \epsilon_{t}$.
\EndFor
\end{algorithmic}
\end{algorithm}

\subsubsection{Filtering resolution $h$.}
\label{subsec:h}
In our pre-processing step described above, we have used $h$ to filter the arm set as well as control how accurate the forced-sample estimator needs to be. As mentioned earlier, smaller $h$ includes fewer arms in the targeted set, and thus more accurate forced-sample estimation is needed to have the optimal arm in the targeted set. Vice versa, bigger $h$ leads to a larger targeted set that allows for more crude forced-sample estimation. 

Recall that the parameter $\rho$ controls the quality of the forced-sample estimator. As the next lemma shows, the smaller $h$ is, the bigger $\rho$ needs to be to ensure more forced samples. Its proof (in Appendix \ref{appendix:element-wise}) builds on Proposition \ref{prop:low-rank-estimators} and Lemma \ref{lemma:forced-sampling-rounds} (in Appendix \ref{appendix:element-wise}, which characterizes how many forced samples we can have given the forced-sampling parameter $\rho$). In our subsequent analysis, we assume $T=d^\beta$ for some $\beta\in \mathbb{R}$, thus smaller values of $\beta$ correspond to shorter horizons.

\begin{lemma}\label{prop:forced-samp-est}
For all $t\geq 2\rho\log(\rho)$ where
$\rho\geq 160 \max\{C, C_\varsigma^2\}\kappa^4\mu{^2}\rrank^2(\sigma/\sigma_{\min})^2 d\log^3(d) h^{-2}$ with $\varsigma=3\beta$ where $T=d^\beta$, 
 it holds  with probability at least $1-2t^{-3}$ that, $\|\widehat \bB^F - \bB^*\|_{\infty}\leq h/4\,.$
\end{lemma}
The above result states that the forced-sample estimator $\widehat \bB^F$ satisfies $\|\widehat \bB^F - \bB^*\|_{\infty}\leq O(1)$ with probability at least $1-O(1/t)$ for all arms when the forced-sampling parameter $\rho$ is big enough. To this end, we set
\begin{align}
    \rho(h;d)=\alpha \rrank^2 d \log^3(d) h^{-2}\,,
    \label{eq:rho}
\end{align}
where $\alpha=160 \max\{C, C_\varsigma^2\}\kappa^4\mu{^2}(\sigma/\sigma_{\min})^2$, so that the sample size requirement for the low-rank estimator is satisfied and we do not over-explore. Based on the forced-sampling parameter $\rho(h;d)$ derived in Eqn. \eqref{eq:rho}, we need $T\geq d$ (i.e., $\beta\geq 1$) to make the low-rank estimator effective. We discuss the case when $\beta<1$ in Section \ref{sec:subsamp}.

For ease of notation, we use $G(\cdot)$ to refer to the event $G(\mathcal{F}_t):=\left\{\|\widehat \bB^F(\mathcal{F}_{t-1}) - \bB^*\|_{\infty} \leq \frac{h}{4} \right\}\,.$

\paragraph{Range of $h$.} Filtering resolution $h$ has a selection range. The upper bound of this range is controlled by the range of the arm rewards (entries of $\bB^*$) and is 1 per Assumption \ref{assump:parameter_set}. To derive the lower bound, we notice that, given the time horizon $T$, we have $O(\rho\log T)$
number of forced samples per Lemma \ref{lemma:forced-sampling-rounds}. Now, since $\rho\log\rho\leq T/2$ is a required condition for Lemma \ref{lemma:forced-sampling-rounds}, we can let $\rho$ be such that $\rho\log T\leq T/2$, so as to derive the minimum value that $h$ can take, denoted by $\lowh(T)$:
\begin{equation}\label{eq: lowh value}
\lowh(T) = \sqrt{\frac{2\alpha \rrank^2 d\log^3 d \log T}{T}}.
\end{equation}
 This shows that $\lowh(T)$ is of order $\Omega(\sqrt{d/T})$.

We notice that the range of $h$ gets bigger with the time horizon, since $\lowh(T)$ decreases in the time horizon. Hence, for shorter time horizons, the best achievable accuracy of the low-rank estimator is less refined than that for longer time horizons, but more arms will be included in the targeted set to ensure good arms are not filtered out due to a lower quality estimation. Such intuition is formalized in Proposition~\ref{prop: monotonic h}. We work further with the range of $h$ when comparing our bound to the existing literature in Section \ref{subsec:bound-compare}.

\paragraph{Near-Optimal Arms.} An important step in our regret analysis is to show: when the estimation is accurate enough, the targeted set will be a subset of ``near-optimal" arms. We next explain what we mean by ``near-optimal" arms, which is defined according to $h$.

We introduce the following near-optimal set of arms $\mathcal{S}_{opt}^h(\mathcal{I}_r, \mathcal{I}_c)$, parameterized by an arbitrary $h$, where $\mathcal{I}_r\subseteq [d_r]$ denotes a subset of row indices and $\mathcal{I}_c\subseteq [d_c]$ denotes a subset of column indices. $\mathcal{I}_r\times \mathcal{I}_c\subseteq [d_r]\times [d_c]$ denotes the set of entries of the submatrix constructed from $\mathcal{I}_r$ and $\mathcal{I}_c$.

\begin{definition}[Near-optimal set]\label{Def. NO}
We define the \emph{near-optimal set} as the following:
\[
\mathcal{S}^h_{opt}(\mathcal{I}_r,\mathcal{I}_c)= \{(j,k)\in \mathcal{I}_r\times \mathcal{I}_c  \mid \bB^*_{jk}\geq \max_{(j',k')\in \mathcal{I}_r\times \mathcal{I}_c} \bB^*_{j' k'} - h\}\,.
\]
That is, it contains elements in $\mathcal{I}_r\times \mathcal{I}_c$ that are at most $h$ smaller than the best arm in $\mathcal{I}_r\times \mathcal{I}_c$.
\end{definition}

We next introduce the near-optimal function  $g(h; \mathcal{I}_r, \mathcal{I}_c)$, which characterizes the cardinality of the near-optimal set.
\begin{definition}[Near-optimal function] The \emph{near-optimal function} $g(h; \mathcal{I}_r, \mathcal{I}_c) = |\mathcal{S}_{opt}^h(\mathcal{I}_r,\mathcal{I}_c)|$ denotes the number of elements that are at most $h$ smaller than the best arm in $\mathcal{I}_r\times \mathcal{I}_c$.
\label{def:candidate}
\end{definition}

In such cases, for brevity, we often use shorter notation $\mathcal{S}_{opt}^h$ and $g(h)$ for $\mathcal{S}^h_{opt}([d_r],[d_c])$ and $g(h; [d_r], [d_c])$ respectively. The more general case with $\mathcal{I}_r\times \mathcal{I}_c\neq [d_r]\times [d_c]$ will be considered in Section \ref{sec:subsamp}. When entries of $\bB^*$ are sampled from certain distributions, we can describe the expected value of the near-optimal function in a closed form (Appendix~\ref{app:near-opt}). We also provide empirical evidence for how the near-optimal functions look like for low-rank matrices in Appendix~\ref{sec:ghandpsi}.

The following lemma shows that a targeted set is a subset of the near-optimal set when the forced-sample estimation is accurate enough. Its proof is in Appendix \ref{app:target-set}.
\begin{lemma}
If $G(\mathcal{F}_{t-1}) = 1$, then the largest entry $(j^*, k^*)$ belongs to the targeted set $\mathcal{C}$ defined in Eqn. \eqref{eq: candidate set} of Algorithm \ref{alg:low-rank}. Furthermore,  $\mathcal{C}\subseteq \mathcal{S}_{opt}$.
\label{lemma:candidates}
\end{lemma}

Our next section provides the regret analysis of the Low-Rank Bandit algorithm.

\subsection{Regret Analysis of Low-Rank Bandit}\label{SS. regret analysis}

We present both a \emph{gap-dependent} bound and a \emph{gap-independent} bound on our Low-Rank Bandit Algorithm. The key idea of our regret analysis is to decompose the total regret into the regret incurred from the pure exploration phase and that incurred from the targeted exploration + exploitation phase, and bound each respectively. The \emph{gap-dependent bound} depends on the term $\Delta$, which is defined to be the difference between the biggest element and the second biggest element of $\bB^*$, i.e., $\Delta := \min_{(j,k)\in [d_r]\times [d_c]} \left\{(\bB^*_{j^*k^*} - \bB^*_{jk})\ind[ \bB^*_{j^*k^*} - \bB^*_{jk}>0]\right\}\,,$ such that $\ind[\cdot]$ is the indicator function. The independent bound does not depend on $\Delta^{-1}$ which would be helpful when $\Delta$ is very small.  
We present the main results below and relegate the detailed derivations to Appendix \ref{SS. Regret Analysis}.

\begin{theorem}[Gap-dependent Bound]
For any $h$ such that $h\geq\lowh(T)$, the cumulative regret of Algorithm~\ref{alg:low-rank}, represented by $\regret_T(h)$, is upper bounded by 
\begin{align}
\regret_T(h)\leq C_1 \rrank^2 d \log^3 (d) h^{-2}\log T +\min\left\{hT,\frac{8g(h)\log T}{\Delta} + \Bigg(1 + \frac{\pi^2}{3}\Bigg)hg(h)\right\} + C_2,
\label{ineq:dep-bd}
\end{align}
where $C_1=b^*\alpha(2t_0 +12)$ and $C_2=100 b^*$.
\label{thm:regret}
\end{theorem}

\begin{theorem}[Gap-independent Bound]
For any $h\geq \lowh(T)$, the cumulative regret at time $T$, represented by $\regret_T(h)$, is upper bounded by
 \[
\regret_T(h)\leq C_1 \rrank^2 d\log^3 (d) h^{-2}\log T +  \min\left\{hT, 8\sigmaepsilon\sqrt{2Tg(h)\log T}\right\}  +C_2,
\]
 where $C_1=b^*\alpha(2t_0 +12)$ and $C_2=100 b^*$. 
\label{thm:regret-gapind}
\end{theorem}

As emphasized earlier, we use low-rank estimators rather than the empirical mean of the arms to select the targeted arms in the first phase so that we can learn the arms faster and lower the regret. Indeed, in order to select the optimal arm into the targeted set, we need to control the estimation error within the order of $O(h)$ (Lemma \ref{lemma:candidates}). When there are $d^2$ number of arms, low-rank estimators only use $\tilde O(\rrank^2 d/h^2)$ number of explorations to achieve the aforementioned estimation accuracy (shown earlier by Lemma \ref{prop:forced-samp-est}), whereas empirical mean estimator needs $\tilde O(d^2/h^2)$ number of explorations. This is because empirical mean estimator for each arm only focuses on information for that particular arm, ignoring helpful information from other arms. We elaborate further on this point in Remark~\ref{rmk:benefit_low_rank} after more in-depth analysis.

\subsection{Balancing Regrets of Two Phases with $h$}
\label{sec:balance}

Our algorithm can adjust $h$ to balance the cost of achieving the desired filtering resolution via forced-sampling and the cost of targeted exploration + exploitation after the filtering. Finer filtering resolution incurs more costs in the first phase due to more forced-sampling rounds to achieve more accurate estimates, but alleviates the burden in the second phase since more sub-optimal arms have been filtered out. Vice versa, less refined filtering resolution results in a larger targeted set and thus incurs more costs in the second phase, but saves costs in the first phase.

To explain further mathematically, recall the regret decomposition we have shown in the beginning of Section \ref{SS. regret analysis}. According to the proof of Theorem \ref{thm:regret-gapind}, the regret from the pure exploration phase can be bounded above by $\phi_1$ and the regret from the targeted exploration + exploitation phase can be bounded above by $\phi_2$, where:
\[
\phi_1(h;d,t)= C_1 \rrank^2 d\log^3 (d) h^{-2}\log (t) \quad \text{and} \quad 
\phi_2(h;d,t) = \min\{8\sigmaepsilon\sqrt{2tg(h)\log(t)}, ht\} \,.
\]
Note that we ignore the constant term $C_2$ for $\phi_2$ above (and in subsequent expressions we omit this term for simplicity).
The upper bound of the regret $\overline{\regret}_T(h)$ can be written as 
\begin{equation}\label{eq: regret decompose}
\overline{\regret}_T(h)=\phi_1(h;d,T) + \phi_2(h;d,T)\geq \regret_T(h).
\end{equation}

The quantity $\phi_1(h)$ increases in $h$. The bigger the $h$, the fewer rounds of forced-sampling are needed to achieve the forced-sample estimation accuracy described in Lemma \ref{prop:forced-samp-est}. Consequently, less regret is incurred in the first phase which leads to smaller $\phi_1(h)$. The quantity $\phi_2(h)$ depends on $h$ via the term $g(h)$, which increases in $h$. The bigger the $h$, the more arms are included in the targeted set. Consequently, more regret is incurred in the second phase to find the best arm in the targeted set, which makes $\phi_2(h)$ bigger. Next, we discuss how to choose $h$ optimally.

\subsubsection{The optimality of $h$.}
\label{sec:phi1andphi2}

Based on the interaction between $\phi_1$ and $\phi_2$, we characterize the best filtering resolution that optimizes the regret bound in the following lemma.

\begin{lemma}[Optimal filtering resolution]
\label{prop:optimal_h} Given the total time horizon $T$, let the optimal $h$ be defined as $h^*=\argmin_{h\in[\lowh(T), 2b^*]}\overline{\regret}_T(h)\,.$ We consider the following three cases:
\begin{enumerate}[label=(\arabic*)]
\item \label{optimal_h_scenario_3} If $\phi_1(h;d,T)<\phi_2(h;d,T)$ for any $h\in [\lowh(T), 2b^*]$, then $\tilde h = \lowh(T)$ is optimal up to a factor of $2$, i.e., $\overline{\regret}_T(\tilde h)\leq 2 \overline{\regret}_T(h^*)$. 

    \item \label{optimal_h_scenario_2} If $\phi_1(h;d,T) = \phi_2(h;d,T)$ for some $h = \tilde h\in [\lowh(T), 2b^*]$, then $\tilde h$ is optimal up to a factor of $2$.
\item \label{optimal_h_scenario_1} If $\phi_1(h;d,T)>\phi_2(h;d,T)$ for any $h\in [\lowh(T), 2b^*]$, then $\tilde h = 2b^*$ is optimal up to a factor of $2$. 

\end{enumerate}

Similar arguments hold for the gap-dependent bound in (\ref{ineq:dep-bd}).
\end{lemma}
The proof of Lemma \ref{prop:optimal_h} can be found in Appendix \ref{app:filteringresolution}. Although Lemma \ref{prop:optimal_h} only analyzes the regret \emph{upper bounds} $\phi_1$ and $\phi_2$ rather than the actual regret of the two phases, we empirically show that the actual regret follows the same behavior in numerical simulations. In Appendix \ref{app:phi1andphi2}, we plot the cumulative regret from the two phases for low-rank matrices that are $100\times 100$ with rank 3 with respect to a sequence of $h$, and observe that $h = 1$ with 225 number of forced samples gives the lowest cumulative regret. Further, the regret curves of the two phases cross at around $h=1.4$, and the cumulative regret at $h=1.4$ is within factor of 2 of the lowest cumulative regret.

In Section \ref{subsec:h}, we have discussed that the range of $h$ becomes bigger in the time horizon $T$ due to the functional form of $\lowh(T)$. Next, we show that the optimal filtering resolution in fact monotonically decreases in the time horizon $T$.

\begin{proposition}\label{prop: monotonic h}
    The optimal filtering resolution $\tilde h$ monotonically decreases in the time horizon $T$ when $T\geq 10$.
\end{proposition}

The proof of Proposition \ref{prop: monotonic h} is in Appendix \ref{sec:phi1andphi2}. This proposition matches with our earlier analysis that we work with finer filtering resolution when the time horizon is longer. 

Next, we show that our upper regret bounds improve on those in the existing literature by balancing $\phi_1$ and $\phi_2$ through $h$.

\subsubsection{Benefits of low-rank structure and the targeted exploration + exploitation.} \label{subsec:bound-compare}

We compare the gap-independent bound in Theorem \ref{thm:regret-gapind} against the upper regret bounds for standard non-contextual bandits to show that our bound achieves better rates in Proposition~\ref{prop: compare i.i.d. bandits}. This analysis also holds for comparison against gap-dependent bounds in the literature.

From this section, we focus on the order of $d$ and $T$ and treat $\rrank$ as a constant  when we discuss the regret order since $\rrank \ll d$. We first show in Proposition \ref{prop: compare i.i.d. bandits}\eqref{prop: compare i.i.d. bandits1} that our regret bound under any time horizon is no worse than the regret bound $\min\{O(T), \tilde{O}(d\sqrt{T})\}$ of those standard non-contextual bandit algorithms (e.g. \textsf{UCB}), which would treat each entry as an independent arm. Intuitively, when $T\geq d^2$, the regret bound takes the  form $\tilde{O}(d\sqrt{T})$ since $T\geq d\sqrt{T}$. When $T\leq d^2$, since standard non-contextual bandits like \textsf{UCB} need to randomly sample each of $d_rd_c$ arms at least once to initialize, it will incur $O(T)$ regret. Secondly, we show in Proposition \ref{prop: compare i.i.d. bandits}\eqref{prop: compare i.i.d. bandits2} how to optimize for the choice of $h$ to achieve a strictly smaller regret upper bound under time horizons $d^2\leq T< d^4$.
\begin{proposition}\label{prop: compare i.i.d. bandits}
Recall that $T=d^\beta$. 
For any structure of $g(h)$:
    \begin{enumerate}
        \item For all values of $\beta$, Algorithm \ref{alg:low-rank} achieves no worse regret than $\min\{O(T), \tilde{O}(d\sqrt{T})\}$.\label{prop: compare i.i.d. bandits1}
        \item In particular, when $2\leq \beta<4$, Algorithm \ref{alg:low-rank} can achieve regret order of $o(d\sqrt{T})$. More specifically, it can achieve the regret order of $\tilde{O}(d^{\frac{1}{3}} T^{\frac{2}{3}})$. \label{prop: compare i.i.d. bandits2}
    \end{enumerate}
\end{proposition}

The proof of Proposition \ref{prop: compare i.i.d. bandits} is relegated to Appendix \ref{append:proofs-for-bound-comparison}. We want to emphasize that the improvement of our algorithm stems from leveraging the low-rank structure. Without assuming any special structure, no algorithm is able to achieve a lower regret than the regret lower bound of $d\sqrt{T}$ of the standard non-contextual bandit algorithms like \textsf{UCB}. Below, we further discuss how we benefit from a low-rank estimator, in comparison to using the empirical mean estimator.

\begin{remark}[Benefit of low-rank estimator]
\label{rmk:benefit_low_rank}
As aforementioned, without the low-rank structure, we need $d^2/h^2$ rounds to achieve an optimality gap of $h$ for $d^2$ entries with an empirical mean estimator. Therefore, the order of regret bound on the first phase increases from $\rrank^2 d$ to $d^2$ (ignoring the logarithmic order), so the gap-independent bound in Theorem \ref{thm:regret-gapind} becomes $C_1' d^2 h^{-2}+\min\{hT, 8\sqrt{2T g(h)}\}$. Adapting the analysis for Proposition \ref{prop: compare i.i.d. bandits} to the empirical mean estimator, we obtain a regret order of $\tilde{O}(d^{\frac{2}{3}} T^{\frac{2}{3}})$, which is strictly higher than $\tilde{O}(d^{\frac{1}{3}} T^{\frac{2}{3}})$ derived in Proposition \ref{prop: compare i.i.d. bandits} with the low-rank estimator for the first phase.
\end{remark}

In Proposition \ref{prop: compare i.i.d. bandits}, we have not yet considered any structure of $g(h)$ in our discussion and have only worked with $g(h) \leq d^2$ for any $h$, so as to show how the sole use of low-rank estimation and the element-wise error bound already helps us reduce regret upper bound of the existing literature. 

We illustrate how to further tighten the regret bound by leveraging the structure of $g(h)$ determined by its shrinkage rate $\zeta$ defined below:
\begin{definition}[Shrinkage rate of near-optimal functions] 
\label{def:gh}
The shrinkage rate of the near-optimal function $g(h; \mathcal{I}_r, \mathcal{I}_c)$ is $\zeta=\sup\{\zeta'\geq 0\mid \mathbb{E}_{|\mathcal{I}_r|=|\mathcal{I}_c|=m}[\sqrt{g(h; \mathcal{I}_r, \mathcal{I}_c)}] \leq m h^{\gparam'/2}\,,\forall m\leq d \}.$
\end{definition}

The inequality in the set constraint naturally holds when $\gparam=0$ because $g(h; \mathcal{I}_r, \mathcal{I}_c)\leq m^2$. Since $h\leq 1$, then $g(h; \mathcal{I}_r, \mathcal{I}_c)$ shrinks faster for larger values of $\zeta$, which implies that the low-rank structure can filter out more arms. When $\mathcal{I}_r =[d] =\mathcal{I}_c$, Definition \ref{def:gh} gives $g(h;[d],[d]) \leq d^2h^{\gparam}$. 

Theorem \ref{theorem: g improve} then works with Definition~\ref{def:gh} to give a tighter regret bound for our algorithm. We defer its proof to Appendix \ref{append:proofs-for-bound-comparison}.

\begin{theorem}\label{theorem: g improve}
   For $\threc=(\frac{2\sqrt{2}\gparam (2\alpha \rrank^2)^{\frac{\gparam+4}{4}}}{C_1\rrank^2})^{\frac{4}{\gparam+2}}$, it holds that
\begin{enumerate}
\item When $T\geq \threc d^{\frac{\gparam+4}{\gparam+1}}$,
$\regret_T=\tilde{O}(d T^{\frac{2}{\gparam+4}}).$
\label{theorem: g improve1}
\item When $T\leq \threc d^{\frac{\gparam+4}{\gparam+1}}$,
$\regret_T=\tilde{O}(d^{\frac{1}{3}}T^{\frac{2}{3}}).$
\end{enumerate}
\end{theorem}

We next present Example \ref{example:uniform} in which we plug in empirical estimated values of $\zeta$ to show more concrete improvement.
\begin{example} \label{example:uniform}
Let the low-rank matrices $\bB =\bU\bV^{\top}$ where the entries of $\bU, \bV\in\mathbb{R}^{d\times \rrank}$ follow uniform distribution on [0,1]. From the empirical evidence in Appendix \ref{sec:ghandpsi}, $g(h)\leq d^2 h^3$ when $d=100$, i.e., the shrinkage rate $\zeta = 3$. Then from Theorem \ref{theorem: g improve}, when $T\geq \threc d^{\frac{7}{4}}$, the regret is of order $\tilde{O}(dT^{\frac{2}{7}})$, which is lower than $\min\{\tilde{O}(d\sqrt{T}), T\}$; when $d\leq T\leq \threc d^{\frac{7}{4}}$, the regret is of order $\tilde{O}(d^{\frac{1}{3}}T^{\frac{2}{3}})$, lower than the regret needed by standard non-contextual bandit like \textsf{UCB}, which is at least $\min\{d^2, T\}$ so as to pull each arm once when $T\leq d^2$.
\end{example}

\begin{remark}
The above example shows that we can construct the empirical $g(h)$ and thus fit $\zeta$ if we have access to the reward distribution from abundant related historical data. Consequently, we can obtain the optimal $h$ via the recipe provided in the proof of Theorem~\ref{theorem: g improve}. A more data-driven approach for obtaining the optimal $h$ is given in Section \ref{sec: data-driven heuristics}, where we use the optimal values found for a historical dataset using a pre-specified grid.
\label{rmk:data-driven-h}
\end{remark}

\begin{remark}[Benefit of targeted exploration + exploitation]
 \label{rmk:benefit_ucb}
We use Theorem \ref{theorem: g improve} to highlight the benefit of our targeted exploration + exploitation in the second phase, by comparing against a naive explore-then-commit (ETC) approach that we call \textsf{Low-rank ETC}. As the name suggests and similar to ours, \textsf{Low-rank ETC} undergoes the ``pure exploration" phases and uses the forced samples to construct a low-rank estimator. However, it commits to the best estimated entry in the second phase rather than conducting targeted exploration + exploitation. Consequently, the regret upper bound on the second phase changes from $\min\{8\sigmaepsilon\sqrt{2Tg(h)\log(T)}, hT\}$ to $hT$, resulting in a potentially higher gap-independent regret bound of $C_1 \rrank^2 d\log^3 (d) h^{-2}\log T +  hT$. In particular, when $g(h)$ has a nice structure (i.e., $g(h)$ is small), our approach is strictly better. As shown in Theorem \ref{theorem: g improve}\eqref{theorem: g improve1}, the regret of our \textsf{LRB} achieves an order of $\tilde{O}(d T^{\frac{2}{\gparam+4}})$ for $T\geq \vartheta d^{\frac{\gparam+4}{\gparam+1}}$, which improves from \textsf{Low-rank ETC} whose regret is of order $\tilde{O}(d^{\frac{1}{3}} T^{\frac{2}{3}})$. A visualization of such a comparison in upper bounds can be found in Section~\ref{sec:visual}. We additionally show the superior empirical performance of our approach against \textsf{\textsf{Low-rank ETC}} in Appendix \ref{app:benefit_filtering}.
\end{remark}

\subsubsection{Effectiveness of \textsf{LRB} under different horizons.}
\label{sub:short_horizon}
Proposition~\ref{prop: compare i.i.d. bandits}, Theorem~\ref{theorem: g improve} and Example~\ref{example:uniform} bring an important implication: our algorithm \textsf{LRB} can adapt to different time horizons, even when the time horizon is short. 

When $T>d^4$, Proposition~\ref{prop: compare i.i.d. bandits} shows that our algorithm performs no worse than those standard non-contextual bandits, and we call this time regime the \emph{long horizon}. Furthermore, when we know the near-optimal function $g(h)$ as in Example \ref{example:uniform}, \textsf{LRB} is strictly better in the long horizon. 

When $d^2\leq T< d^4$, our algorithm achieves a regret upper bound that is strictly better than that of standard non-contextual bandits even without specifying $g(h)$ (as shown in Proposition~\ref{prop: compare i.i.d. bandits}), and we call such a time regime the \emph{short horizon}. In such a \emph{short horizon} regime, we have seen that \textsf{LRB} no longer sticks with pure exploration phase but activates the targeted exploration + exploitation phase with a proper selection of the filtering resolution $h$ (shown in the proof of Proposition~\ref{prop: compare i.i.d. bandits}). Such flexibility of our algorithm can be very useful, since in real-world practices, the product set size and the experimentation time vary a lot by applications. While some have a small product set and the experimentation can effectively occur over a long time horizon, others have a huge product set and the experimentation needs to be completed in a very short time. 

The remaining time horizons, i.e. $T\leq d^2$, will be called the \emph{ultra-short} regime. In this regime, \textsf{UCB} is known to only perform forced-sampling (i.e. pure exploration). Our algorithm \textsf{LRB} also only does forced-sampling when $T\leq d$.
When $d\leq T\leq d^2$, Example~\ref{example:uniform} shows that our algorithm \textsf{LRB} can perform better than pure exploration when $g(h)$ has a good structure. However, as in the general case shown in Proposition~\ref{prop: compare i.i.d. bandits}, while our algorithm is no worse than standard non-contextual bandits, a strictly better performance is not guaranteed under this regime. We have thus proposed a subsampling strategy for enhancing performance when the time horizon is very limited in Section \ref{sec:subsamp}. It is shown through Theorem \ref{thm: subsampling regret} and Example~\ref{ex:regret-uniform-subsample} that the subsampling variation can be effective even under longer periods beyond the ultra-short horizons.

\section{Low-Rank Bandit with Submatrix Sampling}
\label{sec:subsamp}
In this section, we propose to perform a subsampling pre-step that can potentially incur less cost than a regular Low-Rank Bandit policy when the horizon is (and possibly beyond) ultra-short (defined in Section \ref{sub:short_horizon}). Under such a horizon, subsampling allows our regular $\textsf{LRB}$ to leverage the low-rank structure of a smaller submatrix to effectively filter out enough sub-optimal arms. Building on Section~\ref{subsec:bound-compare}, we summarize and visualize the benefits of low-rank structure, targeted exploration + exploitation, and subsampling at the end of this section.

\subsection{Description of the Submatrix-Sampled Low-Rank Bandit Algorithm}
The subsampling pre-step samples a submatrix of $m_r\times m_c$ and our \textsf{LRB} policy is used as a subroutine to explore and exploit the low-rank structure of that submatrix. We call this subsampling version of \textsf{LRB} as Submatrix-Sampled Low-Rank Bandit (\textsf{ss-LRB}), and the specific procedures are summarized in Algorithm \ref{alg:two-stage}. 

\begin{algorithm}[H]
\caption{Submatrix-Sampled Low-Rank Bandit (\textsf{ss-LRB}) Algorithm}
\label{alg:two-stage}
\begin{algorithmic}[1]
\State Decide the submatrix size $m_r, m_c$ as a function of $d_r, d_c, \rrank, T$.
\State Draw a set of $m_r$ number of row indices, denotes as $\mathcal{I}_r$, and a set of $m_c$ number of column indices, denoted as $\mathcal{I}_c$, uniformly at random (without replacement) from $[d]$ respectively.
\State Run Low-Rank Bandit algorithm (Algorithm \ref{alg:low-rank}) on the submatrix indexed by ($\mathcal{I}_r$, $\mathcal{I}_c$).
\end{algorithmic}
\end{algorithm}
For simplicity of presentation, we set $m_r=m_c=m$, but we can optimize for selecting non-square submatrices. Note that \textsf{ss-LRB} only estimates the submatrix rather than the entire matrix, which is a much easier and quicker task if the submatrix dimension is small (smaller sample complexity). To see this, we notice that 
the forced-sampling parameter in \textsf{ss-LRB} becomes $\rho(h;m)=\alpha \rrank^2 m \log^3(m) h^{-2}$,
and the lower bound for $h$ becomes $\lowh(T) =\sqrt{2\alpha \rrank^2 m\log^3(m) \log T/T}$ according to Eqn. \eqref{eq: lowh value}, such that $\lowh(T)$ is of order $\tilde{\Omega}(\sqrt{m/T})$ and the range of $h$ becomes larger. Thus, the selection range of $h$ expands, which makes case (2) in Lemma~\ref{prop:optimal_h} more likely to happen. Consequently, for a fixed time horizon, the filtering resolution can be finer for a submatrix sampling scheme, leading to a more accurate low-rank estimation.

Our motivation comes from \cite{bayati2020unreasonable}, who suggest that subsampling may help when the time horizon is ultra-short. Their work states that when the number of arms $k\geq \sqrt{T}$, it is optimal to sample a subset of $\Theta(\sqrt{T})$ arms and execute \textsf{UCB} on that subset. However, directly adopting random subsampling from \cite{bayati2020unreasonable} will demolish the helpful matrix structure. To preserve the potential advantage of working with a low-rank matrix, we integrate a slightly different subsampling strategy into our \textsf{LRB}, which is the aforementioned submatrix sampling.

Noticeably, we cannot ensure that the biggest element in this submatrix is necessarily the biggest element of the entire matrix. If not, then in each period, we would at least incur a subsampling cost defined as the difference between the biggest element of the entire matrix and the biggest element of a random submatrix. Such a subsampling cost is a function of the matrix size, the submatrix size and the rank of the full matrix.

\begin{definition}[Subsampling Cost]\label{def: gap}
Let $\mathcal{I}_r$ ($\mathcal{I}_c$ respectively) be an index set of $i$ where $i$ is drawn uniformly from $[d]$ ($[d]$ respectively).
The \emph{subsampling cost} function is defined as: 
\[
\psi(m; d) := \max_{j,k} \  \bB^*_{jk} - \mathbb{E}_{\mathcal{I}_r, \mathcal{I}_c: |\mathcal{I}_r|= |\mathcal{I}_c|= m }\left[\max_{(j',k')\in \mathcal{I}_r\times\mathcal{I}_c} \bB^*_{j'k'}\right].
\]
\end{definition}
When the arm rewards follow certain common distributions, the subsampling cost function has closed forms, which is presented in Appendix \ref{app:near-opt}. We also provide empirical evidence for how the subsampling cost functions look like for low-rank matrices in Appendix~\ref{sec:ghandpsi}.

\subsection{Regret Analysis under Submatrix Sampling}
\label{subsec:regret-sub}
In this subsection, we first give a general regret upper bound analysis for \textsf{ss-LRB} in Theorem \ref{thm:regret_sub}. Then, we impose an assumption on the subsampling cost function as in Assumption \ref{assumption: subsampling} which subsumes common distributions and is supported by empirical evidence for low-rank matrices. We show in Theorem \ref{thm: subsampling regret} and Example \ref{ex:regret-uniform-subsample} that \textsf{ss-LRB} indeed achieves a smaller expected cumulative regret upper bound than the regular \textsf{LRB} under a possibly big range of time horizons. That is, even if the biggest entry in the submatrix might be suboptimal, the cost we save from not exploring the much larger matrix compensates for the regret incurred from picking a suboptimal entry. As a result, such a submatrix sampling procedure may be desirable beyond the ultra-short horizons, and we specify when that would be in our subsequent analysis.

Recall that
$\phi_1(h;m,t)= C_1 \rrank^2 m\log^3 (m) h^{-2}\log t$,
where $C_1>0$ is a constant. 
Similarly, we define
\[
\phi_2(h;m,t) = \min\{8\sigmaepsilon\sqrt{2 t \log t}\cdot \mathbb{E}_{|\mathcal{I}_r|=|\mathcal{I}_c|=m}[\sqrt{g(h; \mathcal{I}_r, \mathcal{I}_c)}], ht\},
\]
where the expectation is taken over all submatrices with both the number of rows and columns equal to $m$.
For a fixed $m$, we use $h^*_{m}$ (which is characterized in Lemma~\ref{prop:optimal_h}) to denote the optimal $h$ that minimizes the regret upper bound relative to the largest entry in the submatrix (not considering the subsampling cost) earlier defined in Eqn.~\eqref{eq: regret decompose}: $\overline{\regret}_T(h;m)=\phi_1(h;m,T) + \phi_2(h;m,T).$ The following theorem gives the gap-independent bound for our algorithm with submatrix sampling.
\begin{theorem}
\label{thm:regret_sub}
The total regret of \textsf{ss-LRB}, represented by $\regret_T(h^*_m;m)$, is bounded by:
\begin{align}\label{bound:subsampling}
 \regret_T(h^*_m;m)&\leq \psi(m; d) T +\phi_1(h^*_{m};m,T)+\phi_2(h^*_{m};m,T)\,.
\end{align}
\end{theorem}
The first term $\psi(m; d) T$ describes the regret incurred from the difference between the biggest element in the submatrix and the biggest element in the entire matrix, per the definition of the subsampling cost function $\psi$ (Definition \ref{def: gap}). The remaining sum, $\phi_1(h^*_{m};m,T)+\phi_2(h^*_{m};m,T)$, bounds the regret of running Low-Rank Bandit on the submatrix when the value of the filtering resolution is $h^*_{m}$. The gap-dependent bound for the submatrix sampling algorithm can be constructed similarly, such that the first term $\psi(m; d) T$ stays the same, and for the remaining terms, we use the gap-dependent regret bound on the submatrix. In the extreme case where $m=d$, the regret is reduced to that in Theorem \ref{thm:regret-gapind}. Hence, \textsf{LRB} can be viewed as a special case of \textsf{ss-LRB}. We visualize the benefits of subsampling based on the regret upper bound~\ref{bound:subsampling} in Section~\ref{sec:visual}.

To bring the unknown subsampling cost function $\psi(m;d)$ into the analysis, we let $m=\eta d$, where $\eta$ is called the subsampling ratio which is a decision variable. We assume a structure of the subsampling cost in Assumption \ref{assumption: subsampling} such that the subsampling cost is equal to 0 when $\eta=1$, i.e., there is no subsampling cost if the submatrix is the entire matrix. Further, the subsampling cost increases when $\eta$ reduces.
\begin{assumption}[Subsampling cost]\label{assumption: subsampling} Suppose that
$\psi(\eta d;d)\leq c_s (1-\eta)^{\gamma_1} \eta^{\gamma_2} d^\iota$, where  $\gamma_1\geq 1$, $\gamma_2\leq 0$, and $c_s\geq 0$. 
\end{assumption}

We first provide two examples where the matrix entries (not necessarily low-rank) follow common distributions, and show that the subsampling costs satisfy Assumption~\ref{assumption: subsampling}. We provide empirical evidence in Appendix~\ref{sec:ghandpsi} to show that Assumption~\ref{assumption: subsampling} holds for a variety of low-rank matrices.

\begin{example}
\label{ex:submatrix-exponential}
For a matrix whose entries are exponentially distributed with mean $1$, the subsampling cost function is $\psi(\eta d;d)=-2\log(\eta)\,,$ as derived in Appendix \ref{app:near-opt}.
Since $\log(1/\eta^2)\leq 1/\eta^2-1=(1-\eta^2)/\eta^2=(1+\eta)(1-\eta/\eta^2\leq 2(1-\eta) \eta^{-2},$ then $c_s=2$, $\gamma_1=1$, $\gamma_2=-2$, and $\iota=0$.
\end{example} 
\begin{example}
    For a matrix whose entries are distributed approximately as uniform $[0,1]$, the subsampling cost function is derived in Appendix \ref{app:near-opt} as $\psi(\eta d;d) \approx 1/(\eta d)^2-1/d^2\leq (1/\eta^2-1)/d^2\leq 2(1-\eta)\eta^{-2}d^{-2}.$ Therefore, $c_s=2$, $\gamma_1=1$, $\gamma_2=-2$, and $\iota=-2$.
\label{rmk:sub-uniform-in-main}
\end{example}

Under Assumption \ref{assumption: subsampling}, we derive Theorem \ref{thm: subsampling regret} to provide regret upper bounds of the subsampling version of \textsf{LRB}, namely the \textsf{ss-LRB}. From the analysis, we show that it can be more efficient to work with a submatrix in (and possibly beyond) the ultra-short horizon regime, even though the first term grows linearly with time horizon $T$ in Eqn. \eqref{bound:subsampling}.

\begin{theorem}\label{thm: subsampling regret}
Suppose Assumption \ref{assumption: subsampling} holds. Recall that $\threc$ is defined in Theorem~\ref{theorem: g improve}. 
\begin{enumerate}
    \item When $T\geq \threc d^{\frac{\gparam+4}{\gparam+1}}$, $ \regret_T^{SS}=\tilde{O}(\min\{d^{\iota}T, dT^{\frac{2}{\gparam+4}}\}).$
\item When $T\leq \threc d^{\frac{\gparam+4}{\gparam+1}}$, $\regret_T^{SS}=\tilde{O}(\min \{
 d^\iota T,d^{\frac{1}{3}} T^{\frac{2}{3}}\}).$
\end{enumerate}
Moreover, to achieve this regret order, we select the subsampling ratio in the following way:
\begin{enumerate}[label=(\roman*)]
    \item When $T\geq \threc d^{\frac{\gparam+4}{\gparam+1}}$:
    \begin{enumerate}
        \item if $T^{\frac{\gparam+2}{\gparam+4}}\leq d^{1-\iota}$, we choose $\eta=d^{\iota-1} T^{\frac{\gparam+2}{\gparam+4}}$;
    \item  if $T^{\frac{\gparam+2}{\gparam+4}}\geq d^{1-\iota}$, we choose $\eta$ such that
    $(1-\eta)^{\gamma_1}\eta^{\gamma_2}  = d^{1-\iota} T^{-\frac{\gparam+2}{\gparam+4}}$.\label{thm: subsampling 1b}
    \end{enumerate}
    \item When $T\leq \threc d^{\frac{\gparam+4}{\gparam+1}}$:
    \begin{enumerate}
        \item if 
        $T\geq d^{1-3\iota}$, we choose $\eta$ such that $(1-\eta)^{\gamma_1}\eta^{\gamma_2}=d^{\frac{1}{3}-\iota}T^{-\frac{1}{3}}$;
        \item if 
        $T\leq d^{1-3\iota}$, we choose $\eta= d^{\frac{1}{4}\iota-1} T^{\frac{1}{4}}$.
    \end{enumerate}
\end{enumerate}
\end{theorem}

The proof of Theorem \ref{thm: subsampling regret} is relegated to Appendix \ref{appendix: optimal submatrix}. Compared with Theorem \ref{theorem: g improve}, the subsampling strategy has reduced the regret from $\regret_T$ in Theorem \ref{theorem: g improve} to $\regret_T^{SS}=\min\{\regret_T, \tilde{O}(d^{\iota}T)\}$. We notice that, while the optimal subsampling ratio $\eta$ depends on all parameters in Assumption \ref{assumption: subsampling} (i.e.,$\gamma_1$, $\gamma_2$, $\iota$, and $\gparam$), the improvement of the subsampling strategy depends solely on the value of $\iota$. Next, we provide an example to present more concrete improvement.
\begin{example}
\label{ex:regret-uniform-subsample}
Let the low-rank matrices $\bB =\bU\bV^{\top}$ where the entries of $\bU, \bV\in\mathbb{R}^{d\times \rrank}$ follow uniform distribution on [0,1]. From the empirical evidence in Appendix \ref{sec:ghandpsi}, $c_s\approx 2$, $\gamma_1\approx 1$, $\gamma_2\approx -0.5$, and $\iota\approx -\frac{1}{4}$. In this case, $\psi(\eta d; d)\leq 2 (1-\eta)\eta^{-\frac{1}{2}} d^{-\frac{1}{4}}$. Moreover, $\gparam\approx 3$ as discussed in Example \ref{example:uniform}, i.e., $\mathbb{E}_{|\mathcal{I}_r| = |\mathcal{I}_c|=m}\sqrt{g(h;\mathcal{I}_r, \mathcal{I}_c)}\leq mh^{\zeta/2}$. By applying Theorem \ref{thm: subsampling regret}, when $T\geq \threc d^{\frac{7}{4}}$, the regret is of order $\min\{\tilde{O}(d^{-\frac{1}{4}} T), \tilde{O}(d T^{\frac{2}{7}})\}$; when $T\leq \threc d^{\frac{7}{4}}$, the regret is of order $\min\{\tilde{O}(d^{-\frac{1}{4}}T), \tilde{O}(d^{\frac{1}{3}}T^{\frac{2}{3}})\}$. Comparing with Example \ref{example:uniform}, when $T\leq \threc d^{\frac{7}{4}}$, the regret can be reduced from $\tilde{O}(d^{\frac{1}{3}}T^{\frac{2}{3}})$ to $\tilde{O}(d^{-\frac{1}{4}}T)$.
\end{example}

The above example shows that \textsf{ss-LRB} performs better than \textsf{LRB} under the ultra-short horizons. In addition, Theorem \ref{thm: subsampling regret} implies that \textsf{ss-LRB} can possibly improve \textsf{LRB} under longer horizons when $\iota$ is small. The algorithm trades off between having a smaller subsampling cost with larger submatrix size and having a more efficient low-rank matrix estimation with smaller submatrix size. As the experimentation horizon increases, the algorithm wants to work with a larger submatrix so that it is more likely to include the largest element of the entire matrix to reduce the subsampling cost which grows linearly in time. Eventually, the algorithm will choose to work with the entire product set and utilize the low-rank structure on the entire matrix to filter out suboptimal arms when the time horizon is long enough. Such an intuition is implied by Theorem \ref{thm: subsampling regret}\eqref{thm: subsampling 1b}, where we see that, as $T$ becomes large enough, the optimal $\eta$ approaches 1.

\begin{remark}
    Similar to the comments we have made in Remark \ref{rmk:data-driven-h}, we see via Example \ref{ex:regret-uniform-subsample} that we can estimate $\psi(\eta d;d)$ if we are given  historical data and fit $c_s, \gamma_1, \gamma_2$ and $\iota$ so as to decide the optimal subsampling ratio $\eta$. A more data-driven approach elaborated in Section~\ref{sec: data-driven heuristics} is to use the optimal subsample size found for a historical dataset by searching over a pre-specified grid.
\end{remark}

\subsection{Visualizing the Benefits of Low-Rank Structure, Targeted Exploration + Exploitation and Subsampling}
\label{sec:visual}
We visualize the potential benefits of leveraging low-rank structure, using targeted exploration + exploitation phase, and subsampling, by plotting the sum of dominating terms (ignoring the constants) in the regret bound of several approaches, namely, $\psi(m;d)T + m/h^2 + \min\{\sqrt{T}\cdot\mathbb{E}_{|\mathcal{I}_r|=|\mathcal{I}_c|=m}[\sqrt{g(h;\mathcal{I}_r, \mathcal{I}_c)}], hT\}$ for \textsf{ss-LRB}, $d/h^2 + \min\{\sqrt{Tg(h)}, hT\}$ for \textsf{LRB}, $dh^{-2} + hT$ for \textsf{Low-rank ETC}, and $d\sqrt{T}$ for \textsf{UCB}, as they vary with $h$ in Figure~\ref{fig:upper_bound_comparison}.  We set $T = 3000$ and $d =100$ (more experiment details can be found in Appendix~\ref{app:visual_additional}). Note that the $y$-axis is the sum of the dominating terms, not the actual regret upper bounds or the actual regrets. This plot is intended to convey what we would expect from the theoretical bounds at a high level. We show with other empirical experiments that our methods outperform in terms of the actual regrets.

First of all, the curve for \textsf{UCB} hovers above all other approaches, showing that ignoring the low-rank structure can lead to substantially worse results.

Next, when comparing \textsf{LRB} and \textsf{Low-Rank ETC}, the gap between the two curves shows the value of bringing in the targeted exploration + exploitation in the second phase. We elaborate more on different ranges of $h$ under $T = 3000$ as in Figure \ref{fig:upper_bound_comparison_T3000} first. When $h$ is in the large range, the targeted set is too large to undergo exploration + exploitation efficiently. Hence, targeted exploration + exploitation does not bring in additional value comparing to committing to the best estimated arm under \textsf{Low-rank ETC}, leading to zero gap between the two curves. When $h$ is in the middle range, the value of having the targeted exploration + exploitation becomes evident with the increasing gap. When an appropriate $h$ is reached, we achieve the lowest value for \textsf{LRB} as the estimation accuracy and targeted set size are balanced. In other words, the time we spend in the pure exploration phase is balanced with the time we spend in the targeted exploration + exploitation phase. When $h$ is in the small range, more forced-sampling periods are needed to achieve the required estimation accuracy determined by $h$. With better low-rank estimation, the value that targeted exploration + exploitation brings in becomes smaller, as shown by smaller gap between the two curves as they go back up. Wasting many forced-sampling periods to learn a very accurate low-rank estimator is not necessary, as our \textsf{LRB} can afford a larger $h$ in the middle range so as to have more targeted exploration + exploitation periods and achieve an overall lower regret upper bound. When $T = 1500$ as in Figure~\ref{fig:upper_bound_comparison_T1500}, because of a shorter horizon, the range in which targeted exploration + exploitation brings in benefits shrinks. As we plot for even shorter horizons shown in Appendix~\ref{app:visual_additional}, \textsf{LRB} becomes ineffective and coincides with \textsf{Low-rank ETC}, which sets the stage for our \textsf{ss-LRB} under ultra-short horizons.

Finally, the gap between \textsf{ss-LRB} and \textsf{LRB} shows the benefit of subsampling. For \textsf{ss-LRB}, we plot the regret given by the optimal submatrix size
for each $h$, where the submatrix size\footnote{We treat \textsf{LRB} as a special case of \textsf{ss-LRB} by setting $m=d=100$.} is chosen from 5 to 100 with step size 5. The gap is zero when the full matrix size is the best submatrix size. In Figure \ref{fig:upper_bound_comparison_T1500}, the lowest value of \textsf{ss-LRB} is achieved when it subsamples, showing that, in shorter time horizons, it is better to subsample first rather than working with regular LRB directly. In Figure \ref{fig:upper_bound_comparison_T3000}, the lowest value of \textsf{ss-LRB} is achieved when it does not subsample, which is expected since the time horizon is relatively large. In both figures, when $h$ is very small, it is less costly to work with a submatrix than a full one due to a potentially smaller targeted set and fewer forced-sampling rounds. When $h$ is very large, \textsf{ss-LRB} chooses smaller submatrix size so that it spends less time in pure exploration phase and undergoes exploration + exploitation phase more efficiently.
\begin{figure}
    \centering
    \begin{subfigure}[b]{0.45\textwidth}
    \centering
        \includegraphics[width=\textwidth]{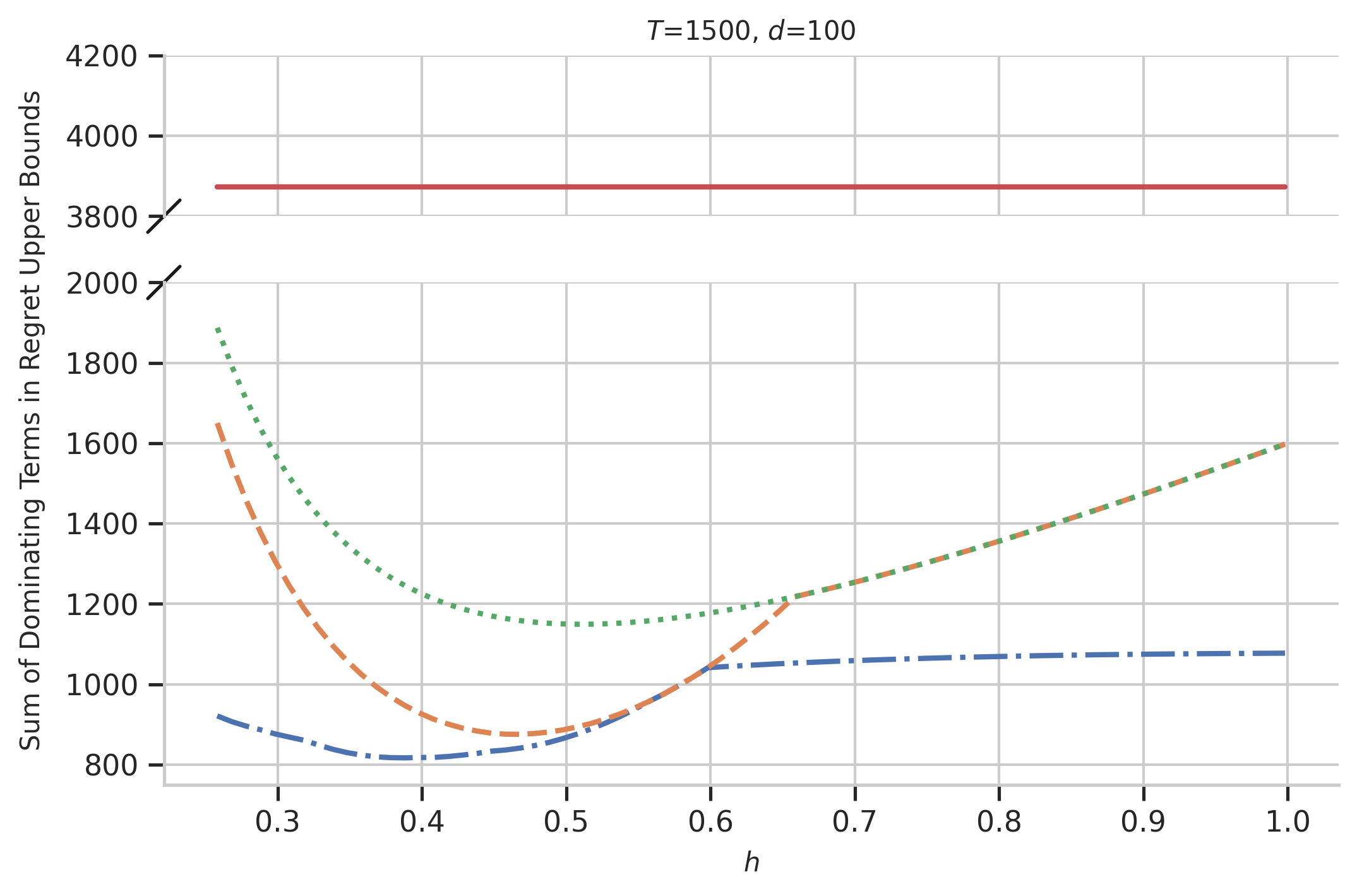}
            \caption{}
    \label{fig:upper_bound_comparison_T1500}
    \end{subfigure}
    \hfill
        \begin{subfigure}[b]{0.54\textwidth}
        \centering
        \includegraphics[width=\textwidth]{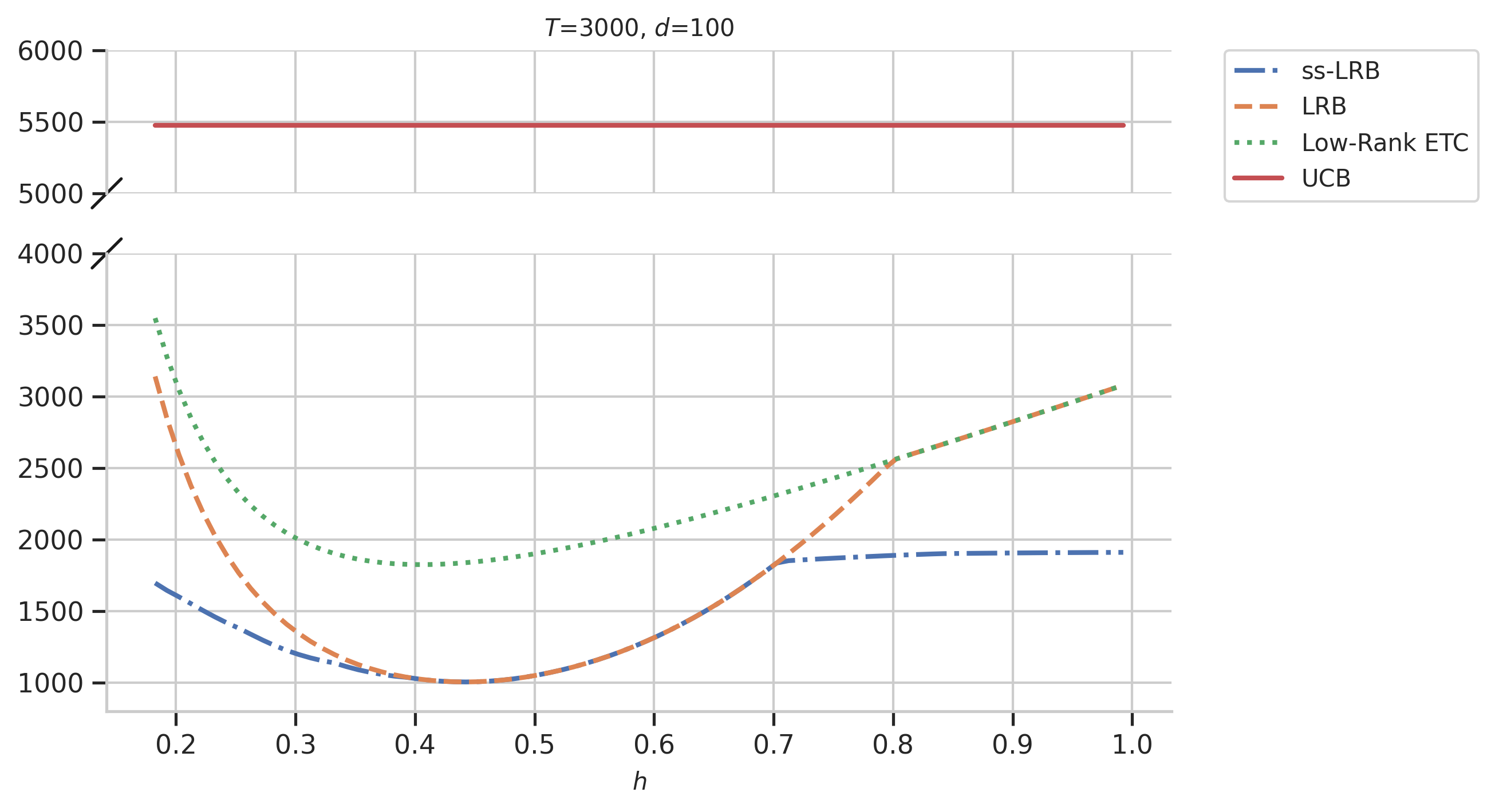}
            \caption{}
    \label{fig:upper_bound_comparison_T3000}
    \end{subfigure}
    \caption{Regret upper bound comparison among \textsf{ss-LRB}, \textsf{LRB}, \textsf{Low-rank ETC}, and \text{UCB} under different values of $h$ for time horizons (a) $T = 1500$ and (b) $T= 3000$.}
    \label{fig:upper_bound_comparison}
\end{figure}
\section{Empirical Results}
\label{sec:data}
We compare the performance of \textsf{LRB} and \textsf{ss-LRB} against existing algorithms in this section. First, we present two sets of empirical results evaluating our algorithm on both synthetic data and a real world dataset on NetEase Cloud Music App's impression-level data. Then, we consider a contextual bandit setting and show simulations supporting fast learning of \textsf{LRB} compared to an optimism-based contextual bandit algorithm.

\subsection{Synthetic Data}
\label{sec:synth}
\paragraph{Synthetic Data Generation.} In Figure \ref{fig:synth_fig}, we evaluate \textsf{LRB} and \textsf{ss-LRB} over 200 simulations, for two sets of parameters $T, d$: a) $d=100$, $\rrank{} = 3$, $T = 1000$ and b) $d=100$, $\rrank{} = 3$, $T = 2000$. In each case, we consider $d^2$ number of arms that form a $d^2$ matrix $\bB^*$ of rank $\rrank{}$. We compare the cumulative regret at period $T$. The error bars are 95\% confidence intervals.

\paragraph{Results.} We compare the Low-Rank Bandit algorithm and its submatrix sampling version against \textsf{ss-UCB} from \citep{bayati2020unreasonable}, which lowers cost from \textsf{UCB} notably due to the presence of many arms. Our results demonstrate that the Low-Rank Bandit and its submatrix sampling version significantly outperform the benchmark given different time horizon lengths. Specifically, in the shorter time horizon $T=1000$, we observe that subsampling improves the performance of \textsf{LRB}. In particular, \textsf{ss-LRB} with submatrix size equal to 40 achieves the lowest regret and reduces the \textsf{LRB} regret by 10\%. Further, by utilizing the low-rank structure of the subsample matrix, \textsf{ss-LRB} with submatrix size equal to 40 reduces \textsf{ss-UCB} regret by 21\%. In the longer time horizon $T=2000$, \textsf{LRB} algorithms perform better than the \textsf{ss-LRB} algorithm. In addition, \textsf{LRB} algorithm cuts the regret of \textsf{ss-UCB} by 31\%, which shows that the benefit from the low-rank structure outweighs the benefit from subsampling. We specify the algorithm inputs and other details in Appendix \ref{app:synth_param}.
\begin{figure}
    \centering
    \includegraphics[scale = 0.6]{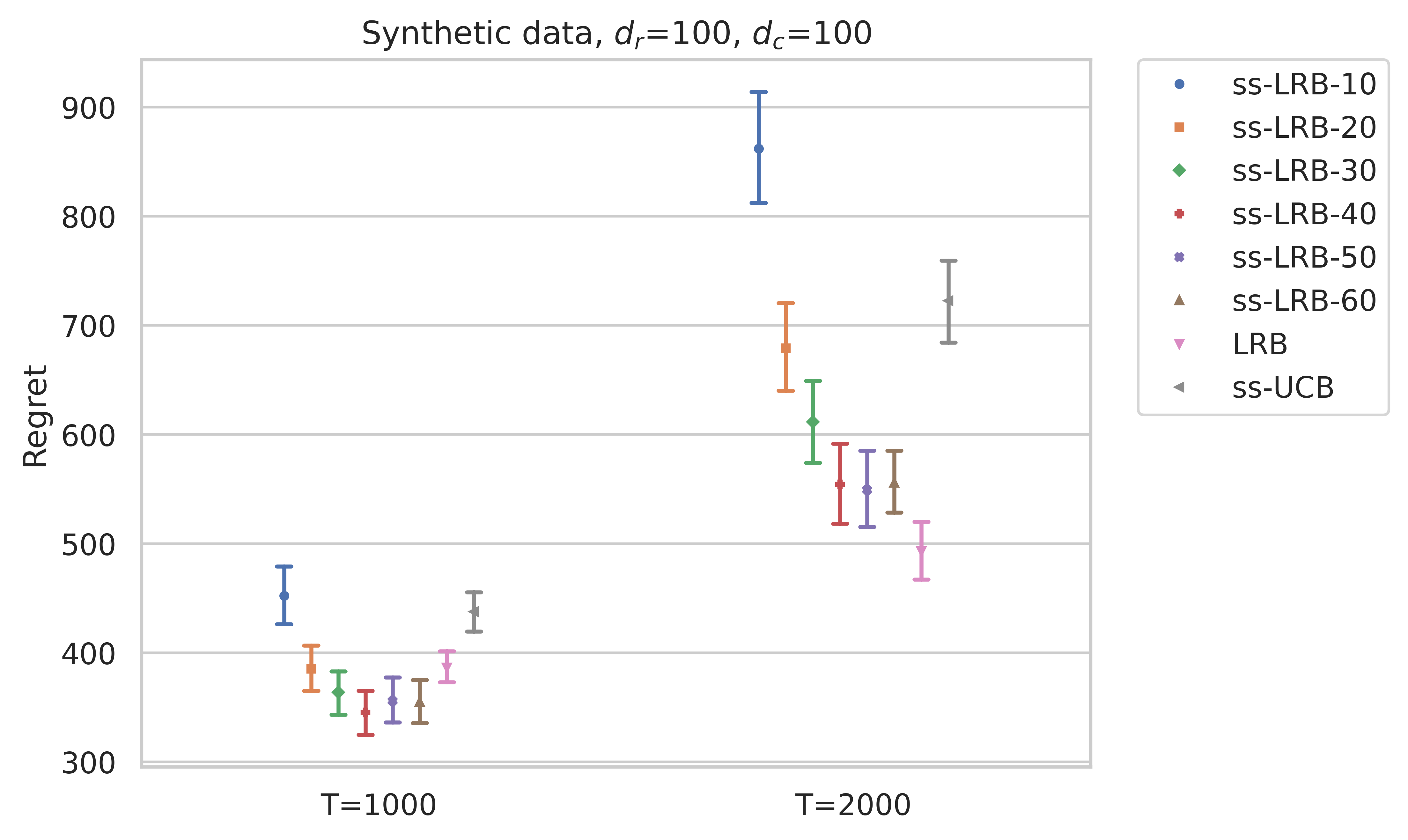}
    \caption{Distribution of cumulative regrets at $T=1000$ and at $T=2000$ for (1) \textsf{ss-LRB} Algorithm \ref{alg:two-stage} with submatrix size=10, (2) \textsf{ss-LRB} Algorithm \ref{alg:two-stage} with submatrix size=20, (3) \textsf{ss-LRB} Algorithm \ref{alg:two-stage} with submatrix size=30, (4) \textsf{ss-LRB} Algorithm \ref{alg:two-stage} with submatrix size=40, (5) \textsf{ss-LRB} Algorithm \ref{alg:two-stage} with submatrix size=50, (6) \textsf{ss-LRB} Algorithm \ref{alg:two-stage} with submatrix size=60, (7) \textsf{LRB} Algorithm \ref{alg:low-rank}, (8) \textsf{ss-UCB} with subsampling size=$\lfloor4\sqrt{T}\rfloor$.}
    \label{fig:synth_fig}
\end{figure}

\subsection{Data-Driven Approach to Select Experimentation Parameters for a Music Streaming App}
\label{sec: data-driven heuristics}
In a real-world setting, we can tune the experimentation parameters in a data-driven manner by using related historical data. To illustrate, we utilize a real-world dataset to simulate an advertising campaign setting as described in Example \ref{ex:ads}. That is, an advertiser needs to specify which pair of user group and content creator group to target for the advertisement campaign to be most effective\footnote{In Appendix~\ref{app:real-data}, we discuss other related modes of advertisement campaigns which are not necessarily two-sided products.}. The number of combinations of a user group and a creator group can be large, as we can distinguish user groups with many different information such as demographics, geographic information and behavior. Likewise, we can distinguish creator groups with many distinct features. Our algorithm (and other benchmark algorithms) would seek to find the best combination of a user group and a creator group to maximize overall interaction between the two groups. As such, we choose an application in the music streaming industry, where we have access to user and creator interaction data. We suppress the interaction data information and let different algorithms learn such information over time. We provide background on music streaming services in Appendix \ref{app:real-data}.

\paragraph{Our Problem.} The company has been providing services to older-aged users and now wants to expand their services to younger-aged ones. It has also attracted new creators to the platform and plans to run an experiment to find a younger-aged user group and a new creator group that interact the most to target an advertisement campaign. The experimentation parameters can be found via historical data on interaction between older-aged user groups and seasoned creator groups.

\paragraph{Bandit Formulation.} To simulate such a setting, we gather the data that correspond to the older-aged users and seasoned creators as a ``historical" dataset, and keep the data that correspond to the younger-aged users and new creators on the side as a ``test" dataset. For each dataset, we cluster the samples into 53 user groups and 39 creator groups, so that we have a 2067-armed stochastic bandit. We bucket users using different locations (provinces) and activity intensity levels. We bucket creators using the anonymous creator types and activity intensity levels. We obtain the same user groups and creator groups for the two datasets. For each (user group, creator group)-pair, we take the average of all interaction data between this user group and this creator group (including clicks, likes, shares, whether the user commented, whether the user viewed the comments and whether the user visited the creator's homepage) and treat it as reward for this pair. 

\paragraph{Selecting Experimentation Parameters.} The experimentation on finding a pair of younger-aged user group and new creator group needs input parameters such as the number of forced samples, denoted as $f$, the filtering resolution $h$ and the submatrix sampling size $m$ for different experimentation horizon length $T$. We work with a pre-specified grid of parameters (listed in Appendix~\ref{app:real_param}) and simulate a bandit experiment for each combination of the parameters using the historical dataset. We pick the set of parameters that achieves the lowest regret on the historical dataset averaged over 30 trials. Essentially speaking, we assume that the historical data and test data share similar reward distributions, hence the optimal parameters learned for the historical data can be used for the experiment we want to run for the test data.

\paragraph{Results and Evaluation.} In Figure \ref{fig:real_figure_final_reg}, we report the regret of using test data to conduct experiments using our algorithm (ss-)LRB versus the benchmark algorithm \textsf{ss-UCB}. We consider 300 runs. The error bars are 95\% confidence intervals. Our algorithm consistently reduces the regret of the benchmark \textsf{ss-UCB} by around 10\% across all time horizons.

In this real data experiment, without knowing whether the data generation produces low-rank matrix of rewards, we still see the effectiveness of our algorithm. Further, even though the historical data distribution may differ from the test data, our experiment results show that using the optimal parameters learned from the historical data still yields better performance than the benchmark. This shows that our data-driven approach of selecting the experiment parameters provides a practical solution to real-world scenarios, such that we can train the parameters on related historical datasets prior to starting an experiment. In practice, when no past data is available to tune these parameters, one can select parameters using the insights provided in Section \ref{sec:alg}-\ref{sec:subsamp}.
 \begin{figure}
     \centering
     \includegraphics[scale = 0.6]{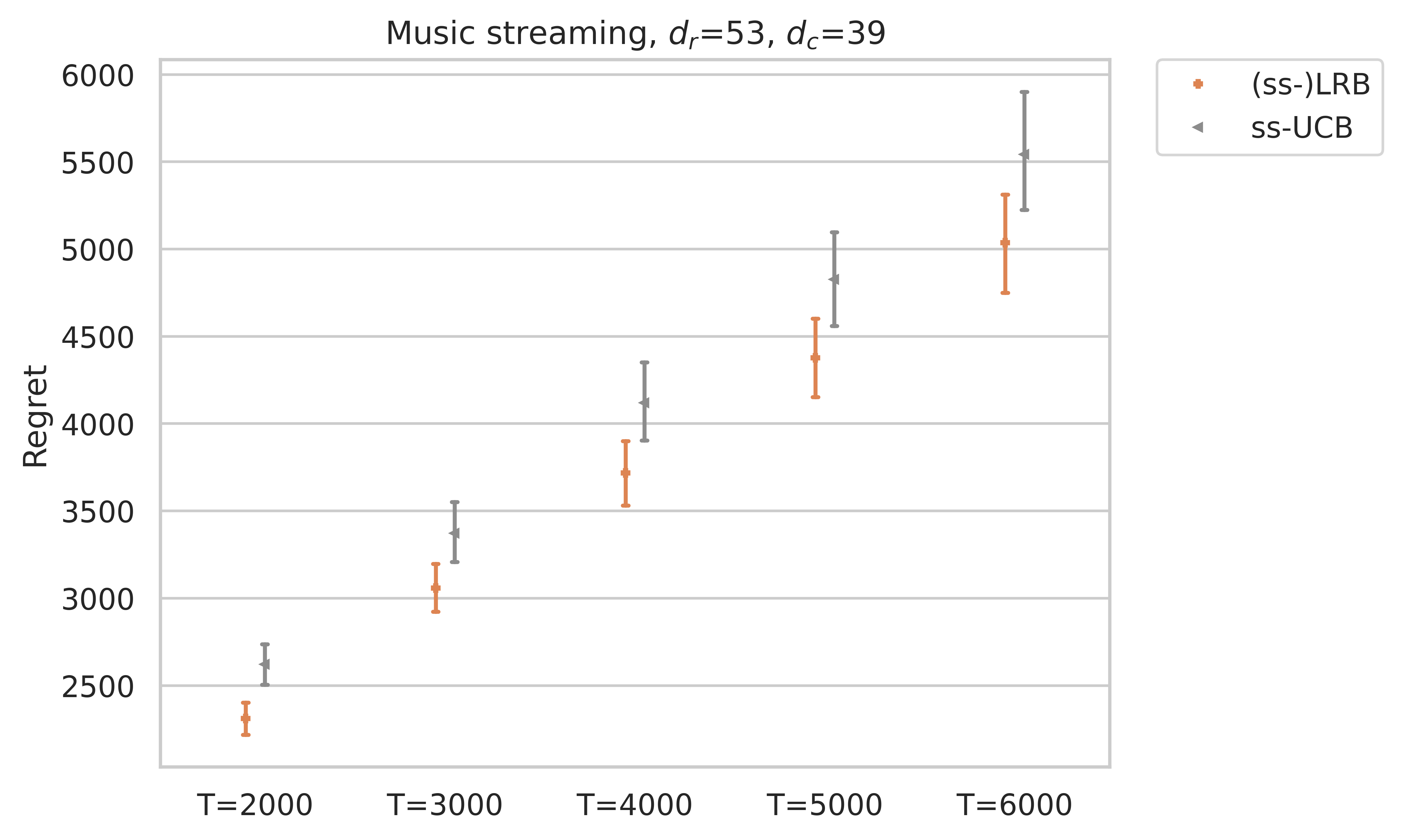}
     \caption{Distribution of the per-instance regret under different time horizons. Parameters of our algorithm (ss-)LRB are selected based on historical data; \textsf{ss-UCB} of \cite{bayati2020unreasonable} has sub-arm size=$4\sqrt{T}$.}
     \label{fig:real_figure_final_reg}
 \end{figure}

So far, we did not use any contextual information. In the next section, we show that \textsf{LRB} performs better than linear bandits which use contextual information, when the time horizon is short, because linear bandits take a long time to learn all the contextual information to initialize.

\subsection{Fast Learning in the Contextual Setting}\label{subsec:contextual-model}

In this section, we consider a classical contextual bandit setting and show how we use \textsf{LRB} to lower the cost in a synthetic experiment. We compare to a bandit algorithm, \textsf{OFUL}, that is designed to learn all the contextual information from \cite{abbasi2011improved}. This comparison is motivated by the observation that the regret of linear bandits at the initial rounds of an experiment increases almost linearly since the algorithm is essentially collecting initial data to perform a baseline estimate of the model parameters. Our algorithm deals with this so-called cold-start learning problem by identifying the low-rank structure behind the arm rewards, ignoring the learning of a large number of parameters. Even though our algorithm does not observe the contextual information, it learns latent feature vectors of the two-sided products quickly and incurs less cost. 

\paragraph{Contextual Bandit Setting.} 

We consider a bandit setting with $d_rd_c$ number of arms whose parameters are unknown. We index the $d_rd_c$ arms by $(j,k)$ for every $j\in [d_r]$ and every $k\in [d_c]$. We denote each arm feature with a $p$-dimensional vector $A^*_{jk}\in \mathbb{R}^p$. At time $t$, by selecting arm $(j,k)$, we observe a linear reward $A^{*\top}_{jk}X + \epsilon_t$, where $X\in \mathbb{R}^p$ is the (population-level) known context. The expected regret incurred at period $t$ is $\max_{A^{*}_{j'k'}}[A^{*\top}_{j'k'}X - A^{*\top}_{jk}X].$ We want to minimize the cumulative expected regret and find the arm that corresponds to the highest reward.

To match this bandit formulation with an example, let us consider the case of tailoring a homepage of an app to a user segment with a known context $X$ mentioned in Example \ref{ex: homepage}. The home page consists of a welcome text message and a picture background, and thus can be considered as a two-sided product. We have $d_r$ choices of welcome text messages and $d_c$ choices for picture backgrounds. $A_{jk}^*$ is the unknown ground truth arm feature of homepage indexed by $(j,k)$.

\paragraph{Low-Rank Reward Matrix.} The mean rewards can be shaped into the following matrix $\bB^*$: 
\[
\bB^* = \begin{bmatrix}\bA^*_{1}X | \bA^*_{2}X | \cdots | \bA^*_{d_c}X\end{bmatrix}\in \mathbb{R}^{d_r\times d_c}, \quad \text{where } \bA^*_{k} = \begin{bmatrix} A^{*\top}_{1k} \\ A^{*\top}_{2k} \\ \vdots \\ A^{*\top}_{d_rk}
\end{bmatrix} \in \mathbb{R}^{d_r\times p}.
\]
The low-rank structure of $\bB^*$ attributes to the modeling detailed next. Let $\bA^*_k = \bU\bV_{k}^{\top}$ where $\bU\in \mathbb{R}^{d_r\times \rrank}$ and $\bV_k\in \mathbb{R}^{p\times \rrank}$. Then $\bB^* = \begin{bmatrix}\bU\bV_{1}^{\top}X  | \bU\bV_{2}^{\top}X | \cdots | \bU\bV_{d_c}^{\top}X\end{bmatrix} = \bU \begin{bmatrix}\bV_{1}^{\top}X  | \bV_{2}^{\top}X | \cdots | \bV_{d_c}^{\top}X\end{bmatrix}.$ Under this design, the mean reward matrix $\bB^*$ is of rank $\rrank$ since $\bU$ is in $\mathbb{R}^{d_r\times \rrank{}}$ and $[\bV_{1}^{\top}X | \bV_{2}^{\top}X | \cdots | \bV_{d_c}^{\top}X]$ is in $\mathbb{R}^{\rrank{}\times d_c}$. We assume $\rrank$ is small, so the mean reward matrix is of low-rank. This is because we only need very few features to explain each welcome text message and each picture background.

In the example of designing homepage, each row of $\bU$ can be interpreted as the feature vector of a text message, and each column of $\begin{bmatrix}\bV_{1}^{\top}X  | \bV_{2}^{\top}X | \cdots | \bV_{d_c}^{\top}X\end{bmatrix}$ can be interpreted as the feature vector of a picture background. Welcome text messages are more straightforward and the interpretation would be the same across different user segments, so they are assumed to be independent of $X$. Picture backgrounds, on the other hand, are more subjective and up to each user segment's interpretation. Thus the representation depends on $X$.

\paragraph{LRB Policy.} To use our \textsf{LRB} method in Algorithm \ref{alg:low-rank}, we can treat the observed reward $A^{*\top}_{jk}X + \epsilon_t$ as the noisy observation of entry $(j,k)$ in $\bB^*$. Note that $A^{*\top}_{jk}X = \bB^*_{jk}$.

\paragraph{OFUL Policy.} To apply \textsf{OFUL} in this setting, we need the following transformation. We work with an arm set that is established by the known context $X$ and an unknown vector $\Theta\in \mathbb{R}^{d_rd_cp}$ that is constructed by using $A^*_{jk}$. The arm set is:
$\Big\{\tilde{A}_{jk}\in \mathbb{R}^{d_rd_cp} \mid j\in[d_r], k\in[d_c]\Big\}$ such that
$\tilde{A}_{jk} = [X_{11}^\top , X_{12}^\top, \cdots, X_{jk}^\top ,\cdots ,X_{d_rd_c}^\top]^\top\,,$ where $X_{jk} = X$ and, for $(j',k')\ne (j, k)$, $X_{j'k'}$ is the vector of all zeros in $\mathbb{R}^p$. The parameter $\Theta$ is defined by $\Theta = [A^{*\top}_{11}, A^{*\top}_{12},\cdots, A^{*\top}_{jk},\cdots, A^{*\top}_{d_r d_c}]^\top\,,$ which gives $\tilde{A}^{\top}_{jk}\Theta=A^{*\top}_{jk}X$. At time $t$, by pulling arm $(j,k)$ with known feature $\tilde{A}_{jk}$, it is equivalent to say that we observe a linear reward $\tilde{A}^{\top}_{jk}\Theta + \epsilon_t$.

\paragraph{Synthetic Experiment Set-up.} Our goal is to seek a policy that minimizes the cumulative expected regret over time horizon $T = 500$. We set $d_r = 8$, $d_c = 10$, $\rrank = 3$, $p = 7$. We evaluate different policies over 10 simulations. The error bars are 95\% confidence intervals. Other experiment details can be found in Appendix~\ref{app:context}.

\paragraph{Results.} In Figure \ref{fig:cumulative}, focusing on time $T=500$, the average cumulative regret of \textsf{OFUL} is 5662 and the average cumulative regret of \textsf{LRB} is 723. \textsf{LRB} lowers the regret of \textsf{OFUL} by nearly 87\%. This result shows that, while \textsf{OFUL} takes time to learn each dimension of $\Theta$ with $d_r\times d_c\times p = 8\times 10\times 7 = 560$ dimensions, \textsf{LRB} efficiently transfers knowledge learned for some arms to other arms by taking advantage of the low-rank structure and thus learns much quicker and incurs much smaller costs. We show in Appendix \ref{app:context} that our algorithm consistently outperforms \textsf{OFUL} by a big margin under all combinations of the experiment parameters we have tried.
\begin{figure}
         \centering
         \includegraphics[scale=0.6]{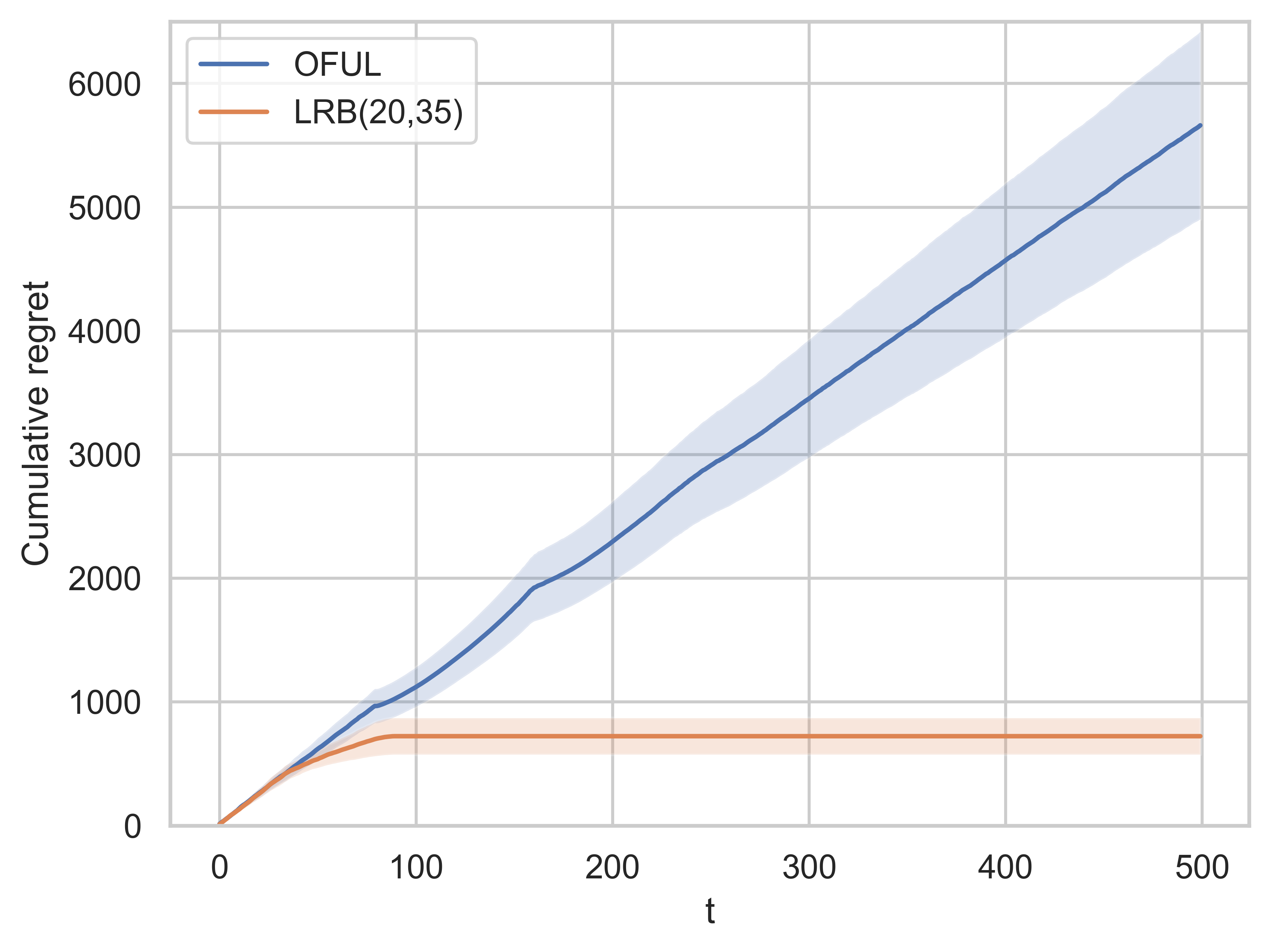}
        \caption{Cumulative regrets of \textsf{OFUL} and our algorithm \textsf{LRB} with number of forced samples equal to 20 and the filtering resolution equal to 35 under the contextual setting.}
        \label{fig:cumulative}
\end{figure}

It is possible to design a variation of \textsf{OFUL} as a benchmark that takes advantage of the special structure of the transformed arms (most dimensions of which are zeros). For example, the \textsf{OFUL} can learn a subset of dimensions in $\Theta$ or it can learn a sparse model, in similar spirit to Lasso Bandit. However, the objective here was to show the advantage of looking at the low-rank reward structure, compared to a standard linear bandit implementation.

%-----------------------CONCLUSION ------------------------%

\section{Conclusion}
This work introduced novel two-sided bandit algorithms, \textsf{LRB} and \textsf{ss-LRB}, to speed up the cold-start learning when there are many arms and the experimentation horizon is relatively short. Our theoretically and empirically efficient solution features leveraging the low-rank structure to shorten the \emph{burn-in} period or the exploration phase, a filtering mechanism to filter out suboptimal arms, and the submatrix sampling to adjust to ultra-short horizons.

\paragraph{Future Directions.}
Besides the low-rank estimator, prediction methods that leverage other problem structures may benefit from pairing with our targeted exploration + exploitation and/or subsampling strategy for enhanced performance as well. Such a heuristic may be of independent interest.
For example, when we want to leverage more dimensions to characterize a product, we can formulate a tensor version of our \textsf{LRB}. As a real-life application, Airbnb Experiences can be modeled as three-sided products by including a temporal side in addition to the content and the location of one activity as mentioned in Section \ref{sec:intro}. On a related note, for settings such as Stitch Fix personalized styling, we have a unique reward profile for each user and different user reward matrices might share similarity. To find the best outfit for each user, we can form a tensor version of our \textsf{LRB} that allows information-sharing among different users, or we can transfer latent contextual information learned across different users.
In addition, it would be interesting to see how our algorithm can be modified to tackle an online product assortment experiment: we have users and products as rows and columns, upon a user's arrival, we want to recommend a product (or an array of products) he/she likes the best. 

Our work gives new perspectives on how to leverage potential low-rank matrix structure for product bundling \citep{stigler1963united}, since the basic two-product bundling case considered in canonical works \citep{adams1976commodity,schmalensee1984gaussian,mcafee1989multiproduct} can be viewed as two-sided products. 
The entries of the matrix can be the expected rewards or purchase probabilities of bundles of two products. Through customer's purchase behaviors, the preference regarding each bundle can be learned, which ultimately informs the pricing scheme.

Finally, our work can be extended to use offline collected historical data (if any) to warm start our two-sided bandit under limited online experiment horizon. For example, we can leverage the offline data for low-rank matrix estimation\footnote{Depending on whether the offline samples are uniformly or non-uniformly sampled, we can keep using the entry-wise bound from \cite{yuxin} or switch to \cite{xi2023entry} who have derived an entry-wise bound for sampling a matrix under a non-uniform pattern.}, by adopting a meta-algorithm called Artificial Replay proposed by \cite{banerjee2022artificial} and previously cited in Section~\ref{sec:lit} to incorporate selected historical data into our bandit algorithm.

\newpage
%%%%%%%%%%%%%%%%%%%%%%%%%%%%%%%%%%%%%%%%%%%%%%%%%%%%%%%%%%%%%%%%%%%%%%
\bibliography{refs}
\bibliographystyle{ormsv080}

%%%%%%%%%%%%%%%%%%%%%%%%%%%%%%%%%%%%%%%%%%%%%%%%%%%%%%%%%%%%%%%%%%%%%%
\newpage

\renewcommand{\theHsection}{A\arabic{section}}
\begin{APPENDICES}

\section{Proof for Element-wise Guarantees of Forced-Sample Estimator}
\label{appendix:element-wise}

We first introduce a lemma that quantifies the forced sample size at time $t$. It states that with probability at most $1/t$, the forced samples at time $t$ will be bigger than $6\rho\log t$, and with probability at most $1/{t^3}$, the forced samples at time $t$ will be smaller than $\rho\log t/2$.

\begin{lemma}[Lemma 1 in \cite{hamidi2019personalizing}]
\label{lemma:forced-sampling-rounds}
The forced-sampling sets created by the forced-sampling rule (\ref{forced-sampling}) satisfy the following inequalities, for all $t\geq 2\rho\log(\rho)$, provided that $\rho\geq 24$,
\[
\mathbb{P}\big(|\mathcal{F}_t|\geq 6\rho\log (t)\big)\leq t^{-1}\quad \text{ and } \quad \mathbb{P}\big(|\mathcal{F}_t|\leq \rho\log (t)/2\big)\leq t^{-3}.
\]
\end{lemma}

\begin{proof}{Proof of Lemma \ref{lemma:forced-sampling-rounds}.}
Let $a:=2\rho\lfloor\log(\rho)\rfloor$. We have that
\begin{align*}
    \theta:&=\mathbb{E}[|\mathcal{F}_t|] = a+\sum_{i = a+1}^t\frac{\rho}{i - \rho\log(\rho)+1} \\
    &\leq a+\rho(\log(t-a)+1)\\
    &\leq 3\rho\log(t).
\end{align*}

Similarly, we get that 
\[
\theta\geq a+\rho\log(t-\rho\log(\rho))\geq \rho\log(t).
\]

Applying Chernoff inequality, we get that 
\begin{align*}
    \mathbb{P}(|\mathcal{F}_t|\geq 6\rho\log(t))&\leq \mathbb{P}(|\mathcal{F}_t|\geq 2\theta)\\
    &\leq \exp(-\frac{\theta}{3})\\
    &\leq \frac{1}{t}.
\end{align*}
The other inequality can be obtained by applying lower tail Chernoff bound:
\begin{align*}
    \mathbb{P}\Bigg(|\mathcal{F}_t|\leq \frac{\rho\log(t)}{2}\Bigg)&\leq \mathbb{P}\left(|\mathcal{F}_t|\leq \frac{\theta}{2}\right)\\
    &\leq \exp\left(-\frac{\theta}{8}\right)\\
    &\leq \frac{1}{t^3}.
\end{align*}
The last step follows because 
\[
\exp\left(-\frac{\theta}{8}\right)\leq \exp\left(-\frac{\rho\log(t)}{8}\right)\leq \exp(-3\log t) = \exp(\log(t^{-3})) = \frac{1}{t^3}.
\]

\QED
\end{proof}

Equipped with Proposition \ref{prop:low-rank-estimators} and Lemma \ref{lemma:forced-sampling-rounds}, we are ready to prove Lemma \ref{prop:forced-samp-est} below.
\begin{proof}{Proof of Lemma \ref{prop:forced-samp-est}.}
According to the condition that $\rho\geq 160 \max\{C, C_\varsigma^2\}\kappa^4\mu{^2}\rrank^2(\sigma/\sigma_{\min})^2 d\log^3(d) h^{-2}$, we have $\rho\geq 32 \max\{C, C_\varsigma^2\}\kappa^4\mu{^2}\rrank^2 (\sigma/\sigma_{\min})^2 d\log^3(d)(1+h^{-2})$ since $h^{-2}\geq 1/(2b^*)^2\geq 1/4$, which implies that $\rho\geq 2C\kappa^4\mu^2 \rrank^2 d\log^3 (d)$ and $\rho\geq 32 C_\varsigma^2\kappa^3\mu \rrank (\sigma/\sigma_{\min})^2 d\log (d)/h^2$.

By Lemma \ref{lemma:forced-sampling-rounds}, we have $|\mathcal{F}_t|\geq \frac{\rho\log t}{2}$ with probability at least $1-\frac{1}{t^3}$.

When $|\mathcal{F}_t|\geq \frac{\rho\log t}{2}$, we have $$|\mathcal{F}_t| \geq C\kappa^4 \mu^2 \rrank^2 d{\log^3 (d)},$$
where the condition holds because $\rho\geq 2C\kappa^4\mu^2 \rrank^2 d\log^3 (d)$. Thus, the sample complexity of Proposition \ref{prop:low-rank-estimators} holds. 

Furthermore, since $\rho\geq 32 C_\varsigma^2 \kappa^3\mu \rrank (\sigma/\sigma_{\min})^2 d\log d/h^2$, we therefore have
\[
|\mathcal{F}_t|\geq\frac{\rho\log t}{2} \geq \frac{32 C_\varsigma^2\kappa^3\mu \rrank(\sigma/\sigma_{\min})^2 d\log d}{2h^2},
\]
so $$\frac{h}{4}\geq C_\varsigma\sqrt{\kappa^3\mu \rrank}(\sigma/\sigma_{\min})\sqrt{\frac{d\log d}{|\mathcal{F}_t|}}.$$

Hence, by plugging $\varsigma=3\beta$ in Proposition \ref{prop:low-rank-estimators}, since $\beta$ satisfies that $d^{-3\beta}= T^{-3}\leq t^{-3}$, we can conclude that, with probability at least $1-2t^{-3}$, 
\[ 
\|\widehat \bB^F - \bB^*\|_{\infty}\leq h/4\,.
\]
\QED
\end{proof}

\section{Proofs for Cumulative Regret}
\label{SS. Regret Analysis}
As mentioned in Section \ref{SS. regret analysis}, our key idea of the regret analysis is to decompose the total regret into the regret incurred from the pure exploration phase and that incurred from the targeted exploration + exploitation phase, and bound each respectively. In the regret decomposition below, we involve the initialization period $t_0$, defined in Section \ref{subsec:description} for clarity. Note that, we further decompose the regret from the targeted exploration + exploitation phase into two parts with respect to the event $G(\mathcal{F}_{t})$. The detailed regret decomposition is as follows:
\begin{enumerate}[label=(\arabic*)]
    \item Pure exploration phase: Initialization ($t\leq t_0$), and other forced-sampling rounds; \label{part 1 step}
    \item targeted exploration + exploitation phase: \label{part 2 step}
    \begin{enumerate}
    \item Times $t>t_0$ when the event $G(\mathcal{F}_{t})$ does not hold; \label{part 2 step a}
    \item Times $t>t_0$ when the event $G(\mathcal{F}_{t})$ holds but the Low-Rank Bandit plays a suboptimal arm chosen by the \textsf{UCB} algorithm from the targeted set $\mathcal{C}$. \label{part 2 step b}
    \end{enumerate}
\end{enumerate}

We point out the important lemmas we have derived to bound each part in the regret decomposition below.

The cumulative expected regret from ``pure exploration phases'' (part \ref{part 1 step}) at time $T$ is bounded by at most $2b^*(t_0 + 6\rho\log T)$ (Lemma \ref{lemma:part-a}), i.e. $\tilde O(d\rrank^2/h^2)$. This follows from the fact that the worst-case regret at any time step is at most $2b^*$ (Assumption \ref{assump:parameter_set}), while there are only $t_0$ initialization samples and at most $6\rho \log T$ forced samples up to time $T$ (Lemma \ref{lemma:forced-sampling-rounds}).

Next, the cumulative expected regret of having sub-optimal arms as the best arm in the targeted set due to an inaccurate forced-sample estimator such that $G(\mathcal{F}_t)$ does not hold (part \ref{part 2 step a}) is bounded by at most $100b^*$ (Lemma \ref{lemma:part-b}).
This follows from the tail inequality for the forced-sample estimator (Lemma \ref{prop:forced-samp-est}), which bounds the probability that event $G(\mathcal{F}_t)$ does not hold at time $t$ by at most $10/{t^3}$. The result follows from summing this quantity over time periods $t_0<t\leq T$.

Finally, the cumulative expected regret of exploring inside the targeted set when the estimator is accurate enough such that $G(\mathcal{F}_t)$ holds (part \ref{part 2 step b}) is no larger than the regret of applying an off-the-shelf non-contextual bandit on the set of ``near-optimal" arms, since the targeted set will be a subset of the ``near-optimal" arms under an accurate enough forced-sample estimator (as the previous Lemma \ref{lemma:candidates} shows). Further, if event $G(\mathcal{F}_t)$ holds, then the set $\mathcal{C}$ (chosen by the forced-sample estimator) contains the optimal arm $(j^*,k^*) = \argmax_{(j,k)\in [d_r]\times [d_c]}\bB^*_{jk}$. Then the \textsf{UCB} subroutine will work with $g(h)$ number of near-optimal arms and pick the best among them as the arm for the current round. The derivations of the gap-dependent bound and the gap-independent bound differ on this \textsf{UCB} subroutine. 

The details for bounding each part can be found below.
\subsection{Part (1)}
\begin{lemma}
\label{lemma:part-a}
\begin{align}
\regret_T^{(1)}\leq 2b^*\mathbb{E}[t_0 + |\mathcal{F}_T|]\leq 2b^*(t_0 + 6\rho\log T).
\label{eq:regretA}
\end{align}
\end{lemma}

\begin{proof}{Proof for Lemma \ref{lemma:part-a}.}

Note that, for each suboptimal choice, the regret incurred at each step is at most $\bB^*_{j^*k^*} -\bB^*_{jk}$ for some $(j,k)$. In addition, since $\|\bB^*\|_{\infty,2}\leq b^*$, then we must have $\|\bB^*\|_{\infty}\leq b^*$. In turn, we have 
\[
\bB^*_{j^*k^*} -\bB^*_{jk}\leq |\bB^*_{j^*k^*} -\bB^*_{jk}|\leq |\bB^*_{j^*k^*}| +|\bB^*_{jk}|\leq 2b^*.
\]
Next, the number of times that each suboptimal arm is pulled is less than or equal to $t_0+|\mathcal{F}_T|$. Putting together and using Lemma \ref{lemma:forced-sampling-rounds}, we get
\begin{align}
\regret_T^{(1)}\leq 2b^*\mathbb{E}[t_0 + |\mathcal{F}_T|]\leq 2b^*(t_0 + 6\rho\log T).
\label{eq:regretA}
\end{align}

\QED
\end{proof}

\subsection{Part \ref{part 2 step a}}

\begin{lemma}
\label{lemma:part-b}
\begin{align}
\regret_T^{(2a)}\leq 100b^*.
\label{eq:regretB}
\end{align}
\end{lemma}

\begin{proof}{Proof for Lemma \ref{lemma:part-b}.}
Next, it follows from the definition of $G(\mathcal{F}_{t})$ and Lemma \ref{prop:forced-samp-est} that the number of times that $G(\mathcal{F}_{t}) = 0$ is controlled by
\[
\mathbb{E}\Bigg[\sum_{t = t_0+1}^T[1-G(\mathcal{F}_{t})]\Bigg] = \sum_{t = t_0+1}^T\mathbb{P}(G(\mathcal{F}_{t}) = 0)\leq \sum_{t = t_0+1}^T 2 t^{-3}\leq 15.
\]
Each round can be bounded by $2b^*$. Hence, $\regret_T^{(2a)}\leq 2b^*\mathbb{E}\Bigg[\sum_{t = t_0+1}^T[1-G(\mathcal{F}_{t})]\Bigg]$, we have
\begin{align}
\regret_T^{(2a)}\leq 100b^*.
\label{eq:regretB}
\end{align}
\QED
\end{proof}

\subsection{Part \ref{part 2 step b}}

\subsubsection{Gap-dependent bound for Part \ref{part 2 step b}}

As mentioned before, to bound part (2b) in the cumulative regret, we build on the standard proof for the \textsf{UCB} algorithm \citep[Chapter 7]{lattimore2020bandit}, with the help from Lemma \ref{lemma:candidates}. Under the condition for Lemma \ref{lemma:candidates}, the biggest entry is for sure in the targeted set $\mathcal{C}$. So we only incur a cost when some element in $\mathcal{S}_{opt}^h$ other than the biggest entry is selected. Further, since $g(h) = |\mathcal{S}_{opt}^h|$, we can use $g(h)$ as an upper bound on the number of arms in the targeted set $\mathcal{C}$. We first show the gap-dependent bound for part (2b).

\begin{lemma}
\begin{align}
    \regret_T^{(2b)} &\leq \min\left\{hT, \frac{8 g(h)\log T}{\Delta} + \Bigg(1 + \frac{\pi^2}{3}\Bigg)hg(h)\right\}.
    \label{eq:regretC}
\end{align}
\label{lemma:regretCgap}
\end{lemma}

\begin{proof}{Proof of Lemma \ref{lemma:regretCgap}.}
We start by finding a bound on the number of times each non-best entry in $\mathcal{S}^h_{opt}$ is pulled when $G(\mathcal{F}_{t}) = 1$. Let $\xi_{(j,k)}(n)$  denote the number of times we have played action $(j,k)\in \mathcal{S}^h_{opt}$ in rounds $1\leq t\leq n$ when $G(\mathcal{F}_{t}) = 1$. Let $\xi^A_{(j,k)}(n)$ denote the total number of times we have played action $(j,k)$ in rounds $1\leq t\leq n$. Let $a(y, T) = \sigmaepsilon\sqrt{2\log T/y}$. We know that, for some arbitrary $m\geq 0$, 
\begin{align*}
    \xi_{(j,k)}(T) = \sum_{t=1}^T\mathbb{I}\Big((j_t, k_t) = (j,k) \text{ and } G(\mathcal{F}_{t}) = 1\Big)&\leq m+\sum_{t=1}^T\mathbb{I}\Big((j_t, k_t) = (j,k) \text{ and }\xi_{(j,k)}(t-1)\geq m \Big) \\
    &\leq m+\sum_{t=1}^T\mathbb{I}\Big((j_t, k_t) = (j,k) \text{ and }\xi^A_{(j,k)}(t-1)\geq m \Big).
\end{align*}

Since $G(\mathcal{F}_{t}) = 1$, we must have $(j^*, k^*)$ in the targeted set $\mathcal{C}(t)$ by Lemma \ref{lemma:candidates}. Therefore, when $(j,k)$ is pulled at $t$ when $G(\mathcal{F}_{t}) = 1$, we must have
\[
\bar x_{(j,k)} + a(\xi^A_{(j,k)}(t), t-1)\geq \bar x^* + a(\xi^A_{(j^*,k^*)}(t), t-1)
\]

Therefore, we arrive at 
\[
\xi_{(j,k)}(T) \leq m+\sum_{t=1}^T\mathbb{I}\Big(\bar x_{(j,k)} + a(\xi^A_{(j,k)}(t), t-1)\geq \bar x^* + a(\xi^A_{(j^*,k^*)}(t), t-1) \text{ and }\xi^A_{(j,k)}(t-1)\geq m \Big).
\]
Denote by $\bar x_{(j,k), s}$ the random variable for the empirical mean after playing action $(j,k)$ a total of $s$ times, and $\bar x^*_s$ the corresponding quantity for the best entry. Then we must have 
\[
\xi_{(j,k)}(T) \leq m+\sum_{t=1}^T\mathbb{I}\Big(\max_{m\leq s < t}\bar x_{(j,k), s} + a(s, t-1)\geq \min_{0<s'<t}\bar x^*_{s'} + a(s', t-1) \text{ and }\xi^A_{(j,k)}(t-1)\geq m \Big).
\]
It follows that
\[
\xi_{(j,k)}(T) \leq m+\sum_{t=1}^T\sum_{s=m}^{t-1}\sum_{s'=1}^{t-1}\mathbb{I}\Big(\bar x_{(j,k), s} + a(s, t-1)\geq \bar x^*_{s'} + a(s', t-1)\Big).
\]
So, 
\[
\xi_{(j,k)}(T) \leq m+\sum_{t=1}^{\infty}\sum_{s=m}^{t}\sum_{s'=1}^{t}\mathbb{I}\Big(\bar x_{(j,k), s} + a(s, t)\geq \bar x^*_{s'} + a(s', t)\Big).
\]
Suppose that the event $\bar x_{(j,k), s} + a(s, t)\geq \bar x^*_{s'} + a(s', t)$ happens. Then we claim that at least one of the following events must happen:
\begin{enumerate}
    \item $\bar x^*_{s'}\leq \bB^*_{j^* k^*} - a(s', t)$,
    \item $\bar x_{(j,k), s}\geq \bB^*_{jk} + a(s, t)$,
    \item $\bB^*_{j^*k^*} < \bB^*_{j k}+2a(s,t)$.
\end{enumerate}
We prove this claim by contradiction. Suppose that none of the three events are true. Then we have 
\[
\bar x_{(j,k), s} + a(s, t)\geq \bar x^*_{s'} + a(s', t)>\bB^*_{j^*k^*} - a(s', t)+a(s', t) = \bB^*_{j^*k^*}
\]
and 
\[
\bB^*_{jk} + 2a(s, t) > \bar x_{(j,k), s} + a(s, t).
\]
Putting the above two inequalities together, we get $\bB^*_{j^*k^*} < \bB^*_{jk}+2a(s,t)$. Here we reached a contradiction as this is exactly event 3. So we conclude that at least one of the three events is true. 

Note that, if we take $m = 8\log T/(\bB^*_{j^*k^*} - \bB^*_{jk})^2$, then $2a(s,t)\leq \bB^*_{j^*k^*} - \bB^*_{jk}$ for any $s\geq m$. Hence, $\bB^*_{j^*k^*} - \bB^*_{j k} -2a(s,t)\geq 0$. So when $m= 8\log T/(\bB^*_{j^*k^*} - \bB^*_{jk})^2$, the probability that at least one of the three events above happens is $2t^{-4}$. This is based on the Chernoff-Hoeffding inequality that $\mathbb{P}\Big(\bar x_{(j,k), s}\geq \bB^*_{jk} + a(s, t)\Big)\leq t^{-4}$ and $\mathbb{P}\Big(\bar x^*_{s'}\leq \bB^*_{j^* k^*} - a(s', t)\Big)\leq t^{-4}$. 
Therefore, 
\begin{align}
    \mathbb{E}[\xi_{(j,k)}(T)] &\leq m+\sum_{t=1}^{\infty}\sum_{s=m}^{t}\sum_{s'=1}^{t}\mathbb{P}\Big(\bar x_{(j,k), s} + a(s, t)\geq \bar x^*_{s'} + a(s', t)\Big) \nonumber\\ 
    &\leq \frac{8\log T}{(\bB^*_{j^*k^*} - \bB^*_{jk})^2}  +\sum_{t=1}^{\infty}\sum_{s=m}^{t}\sum_{s'=1}^{t}2t^{-4} \nonumber\\ 
    &\leq \frac{8\log T}{(\bB^*_{j^*k^*} - \bB^*_{jk})^2} + 1 +\sum_{t=1}^{\infty}\sum_{s=m}^{t}\sum_{s'=1}^{t}2t^{-4} \nonumber\\ 
    &\leq \frac{8\log T}{(\bB^*_{j^*k^*} - \bB^*_{jk})^2} + 1 +2\sum_{t=1}^{\infty}t^{-2} \nonumber\\ 
    &\leq \frac{8\log T}{(\bB^*_{j^*k^*} - \bB^*_{jk})^2} + 1 + \frac{\pi^2}{3}. \label{eq:numpicks}
\end{align}
It follows that
\begin{align*}
    \regret_T^{(2b)} 
    &\leq\sum_{(j,k)\in \mathcal{S}_{opt}^h}(\bB^*_{j^*k^*} - \bB^*_{jk}) \mathbb{E}[\xi_{(j,k)}(T)]\\
    &\leq \min\left\{hT, \sum_{(j,k)\in \mathcal{S}_{opt}^h}\Bigg[\frac{8\log T}{\bB^*_{j^*k^*} - \bB^*_{jk}} + \Bigg(1 + \frac{\pi^2}{3}\Bigg)(\bB^*_{j^*k^*} - \bB^*_{jk})\Bigg]\right\}\\
    &\leq \min\left\{hT, \sum_{(j,k)\in \mathcal{S}_{opt}^h}\Bigg[\frac{8\log T}{\Delta} + \Bigg(1 + \frac{\pi^2}{3}\Bigg)h\Bigg]\right\}\\
    &\leq \min\left\{hT, \frac{8 g(h)\log T}{\Delta} + \Bigg(1 + \frac{\pi^2}{3}\Bigg)hg(h)\right\}.
\end{align*}
\QED
\end{proof}

\subsubsection{Gap-independent bound for Part (2b)}
Now we tackle the gap-independent bound for part (2b).
\begin{lemma} let $\Delta_{jk} = \bB^*_{j^*k^*} - \bB^*_{jk}$, for any given $(j,k)$,
\label{lemma:regretC}
\begin{align}
\regret_T^{(2b)}\leq \min\left\{hT,8\sqrt{2Tg(h)\log T} \right\}.
    \label{eq:regretC}
\end{align}
\end{lemma}

\begin{proof}{Proof for Lemma \ref{lemma:regretC}.} From (\ref{eq:numpicks}), we have the following inequality:
\begin{align*}
    \mathbb{E}[\xi_{(j,k)}(T)] &\leq \frac{8\log T}{(\bB^*_{j^*k^*} - \bB^*_{jk})^2} + 1 + \frac{\pi^2}{3}.
\end{align*}

Therefore, using the basic regret decomposition, we have
\begin{align*}
    \regret_T^{(2b)} &\leq \sum_{(j,k)\in \mathcal{S}_{opt}^h}(\bB^*_{j^*k^*} - \bB^*_{jk}) \mathbb{E}[\xi_{(j,k)}(T)] \\
    &= \sum_{(j,k)\in \mathcal{S}_{opt}^h, \Delta_{jk} \leq \chi}\Delta_{jk} \mathbb{E}[\xi_{(j,k)}(T)] +
    \sum_{(j,k)\in \mathcal{S}_{opt}^h,\Delta_{jk} \geq \chi}\Delta_{jk} \mathbb{E}[\xi_{(j,k)}(T)]\\
    &\leq \min\left\{hT, T\chi + \sum_{(j,k)\in \mathcal{S}_{opt}^h, \Delta_{jk} \geq \chi} \Bigg(\frac{8\log T}{\Delta_{jk}} + \Delta_{jk}(1 + \frac{\pi^2}{3})\Bigg)\right\} \\
    &\leq \min\left\{hT,T\chi + \frac{8g(h)\log T}{\chi} + (1 + \frac{\pi^2}{3}) \sum_{(j,k)\in \mathcal{S}_{opt}^h}\Delta_{jk}\right\} \\
    &\leq \min\left\{hT,4\sqrt{2Tg(h)\log(T)}+(1 + \frac{\pi^2}{3}) hg(h)\right\} ,
\end{align*}
where we choose $\chi = \sqrt{8g(h)\log(T)/T}$.

We further simplify the upper bound. 

If $g(h)\geq T$, we have 
\[
\min\left\{hT,4\sqrt{2Tg(h)\log T}+(1 + \frac{\pi^2}{3}) hg(h)\right\}=hT.
\]
Thus, in this case,
\[
\regret_T^{(2b)}\leq  \min\left\{hT,4\sqrt{2Tg(h)\log(T)}\right\}.
\]

Otherwise, if $g(h)\leq T$, we have 
\[
4\sqrt{2Tg(h)\log T} \geq 2(1 + \frac{\pi^2}{3})\sqrt{Tg(h)\log T}\geq 2(1 + \frac{\pi^2}{3}) g(h)\geq (1 + \frac{\pi^2}{3}) {h g(h)},
\]
where the last inequality holds because $h\leq 2$.
It implies that
\[
4\sqrt{2Tg(h)\log(T)}+(1 + \frac{\pi^2}{3}) hg(h) \leq 8\sqrt{2Tg(h)\log T}.
\]
In conclusion, we have
\[
\regret_T^{(2b)}\leq \min\left\{hT,8\sqrt{2Tg(h)\log T} \right\}.
\]

\QED
\end{proof}

\subsubsection{Targeted set as a subset of the near-optimal Set.}
\label{app:target-set}
Lemma \ref{lemma:candidates} used to bound part (2b) shows that the targeted set is a subset of the near-optimal set when the forced-sample estimation is accurate enough and is proved below.
\proof{Proof of Lemma \ref{lemma:candidates}.}
Since $G(\mathcal{F}_t) = 1$, then
\[
\|\widehat \bB^F(\mathcal{F}_{t-1}) - \bB^*\|_{\infty} \leq \frac{h}{4}\,.
\]
It follows that $|\widehat \bB_{jk}^F(\mathcal{F}_{t-1}) - \bB^*_{jk}|\leq \frac{h}{4}$. Then,
\begin{align*}
   \widehat \bB_{jk} - \widehat \bB_{j'k'} &= \widehat \bB_{jk} - \bB^*_{jk} - \widehat \bB_{j'k'}+ \bB^*_{j'k'} + \bB^*_{jk} - \bB^*_{j'k'}  \\
   &\leq \frac{h}{2} + \bB^*_{jk} - \bB^*_{j'k'}.
\end{align*}
Let $(j,k) = \argmax \widehat \bB_{jk}$ and $(j',k') = (j^*, k^*) = \argmax \bB^*_{jk}$. Since $\bB^*_{jk} - \bB^*_{j^*k^*}\leq 0$, then $\widehat \bB_{jk} - \widehat \bB_{j^*k^*}\leq \frac{h}{2}$. So $(j^*, k^*)\in \mathcal{C}$.

Let $(j',k')\in [d_r]\times [d_c]\setminus\mathcal{S}_{opt}^h$, then $\bB^*_{j'k'}\leq \bB^*_{j^*k^*} - h$.
\begin{align*}
   \widehat \bB_{jk} - \widehat \bB_{j'k'} & \geq \widehat \bB_{j^*k^*} - \widehat \bB_{j'k'}  = \widehat \bB_{j^*k^*} - \bB^*_{j^*k^*} - \widehat \bB_{j'k'} +  \bB^*_{j'k'} + \bB^*_{j^*k^*} - \bB^*_{j'k'} \\
   & \geq -\frac{h}{2} + \bB^*_{j^*k^*}- \bB^*_{j'k'} \geq -\frac{h}{2} + h = \frac{h}{2}.
\end{align*}
So $(j',k')\notin \mathcal{C}$.
\endproof{\QED}
\subsection{Proofs for Theorem \ref{thm:regret} and Theorem \ref{thm:regret-gapind}}
\begin{proof}{Proof of Theorem \ref{thm:regret}.}
Combining Lemmas \ref{lemma:part-a}-\ref{lemma:regretCgap}, and by plugging in $\rho(h;d_r, d_c)=\alpha \rrank^2 d \log^3 (d)h^{-2}$, the cumulative regret is as \eqref{ineq:dep-bd}. Specifically,
\begin{align*}
 \regret_T(h) &\leq \overbrace{2b^*(t_0 + 6\rho\log T)}^{\text{Regret from (1)}} +\overbrace{100b^*}^{\text{Regret from (2a)}} + \overbrace{\min\left\{hT,\frac{8g(h)\log T}{\Delta} + \Bigg(1 + \frac{\pi^2}{3}\Bigg)hg(h)\right\}}^{\text{Regret from (2b)}} \\
&=  b^*\Big[2t_0 + 12\alpha \rrank^2 d \log^3(d) h^{-2}\log T \Big]+ 100b^* + \min\left\{hT,\frac{8g(h)\log T}{\Delta} + \Bigg(1 + \frac{\pi^2}{3}\Bigg)hg(h)\right\}\\
&\leq b^*\Big[2t_0 \alpha \rrank^2 d \log^3 (d) h^{-2}\log T+ 12\alpha \rrank^2 d \log^3 (d) h^{-2}\log T\Big]\\
&\quad \quad +100b^* +\min\left\{ hT, \frac{8g(h)\log T}{\Delta} + \Bigg(1 + \frac{\pi^2}{3}\Bigg)hg(h)\right\}\\
&\leq C_1 \rrank^2 d \log^3 (d) h^{-2}\log T +\min\left\{hT,\frac{8g(h)\log T}{\Delta} + \Bigg(1 + \frac{\pi^2}{3}\Bigg)hg(h)\right\} + C_2,
\end{align*}
where $C_1=b^*\alpha(2t_0 +12)$ and $C_2=100 b^*$. 
\QED
\end{proof}

The proof for Theorem \ref{thm:regret-gapind} follows the same strategy. 
\begin{proof}{Proof of Theorem \ref{thm:regret-gapind}.}
Combining Lemmas \ref{lemma:part-a}-\ref{lemma:regretCgap}, and by plugging in $\rho(h;d_r, d_c)=\alpha \rrank^2 d \log^3 (d)h^{-2}$, the cumulative regret is as \eqref{ineq:dep-bd}. Specifically,
\begin{align*}
 \regret_T(h) &\leq \overbrace{2b^*(t_0 + 6\rho\log T)}^{\text{Regret from (1)}} +\overbrace{100b^*}^{\text{Regret from (2a)}} + \overbrace{\min\left\{hT,8\sigmaepsilon\sqrt{2Tg(h)\log(T)}\right\} }^{\text{Regret from (2b)}} \\
&\leq \underbrace{ C_1 \rrank^2 d\log^3 (d) h^{-2}\log T }_{\phi_1}+  \underbrace{\min\left\{hT, 8\sigmaepsilon\sqrt{2Tg(h)\log(T)}\right\} }_{\phi_2} +C_2
\end{align*}
where $C_1=b^*\alpha(2t_0 +12)$ and $C_2=100 b^*$. 
\QED
\end{proof}

\section{Proofs for Section \ref{sec:balance}}
\label{app:filteringresolution}

\subsection{Proofs for Section \ref{sec:phi1andphi2}}

\begin{proof}{Proof of Lemma~\ref{prop:optimal_h}.} First note that $\phi_1(h)$ is monotonically decreasing with $h$ and $\phi_2(h)$ is monotonically increasing with $h$ and $h$ takes on possible values between $\lowh(T)$ and $2b^*$.

Case \ref{optimal_h_scenario_3}: If $\phi_1(h)<\phi_2(h)$ for any $h\in (\lowh(T), 2b^*]$, we have $\phi_1(\lowh(T))<\phi_2(\lowh(T))\leq \phi_2(h^*)$, which implies that $\overline{\regret}_T(\lowh(T))=\phi_1(\lowh(T))+\phi_2(\lowh(T))\leq 2\phi_2(h^*)\leq 2\overline{\regret}_T(h^*)$.

Case \ref{optimal_h_scenario_2}: Assume $\tilde{h}$ satisfies that $\phi_1(\tilde{h}) = \phi_2(\tilde{h})$. If $h^*> \tilde{h}$, we have $\phi_1(h^*)>\phi_1(\tilde{h})\geq \phi_2(\tilde{h})$, which implies that $\overline{\regret}_T(\tilde{h})=\phi_1(\tilde{h})+\phi_2(\tilde{h})\leq 2\phi_1(h^*)\leq 2\overline{\regret}_T(h^*)$. Instead, if $h^*<\tilde{h}$, we have $\phi_2(h^*)>\phi_2(\tilde{h})\geq \phi_1(\tilde{h})$, which yields $\overline{\regret}_T(\tilde{h})=\phi_1(\tilde{h})+\phi_2(\tilde{h})\leq 2\phi_2(h^*)\leq 2\overline{\regret}_T(h^*)$.

Case \ref{optimal_h_scenario_1}: Assume the optimal $h$ is $h^*$. Then if $\phi_1(h)>\phi_2(h)$ for any $h\in (\lowh(T), 2b^*]$, we have $\phi_2(2b^*)<\phi_1(2b^*)\leq \phi_1(h^*)$, which implies that $\overline{\regret}_T(2b^*)=\phi_1(2b^*)+\phi_2(2b^*)\leq 2\phi_1(h^*)\leq 2\phi_1(h^*)+2\phi_2(h^*) = 2\overline{\regret}_T(h^*)$.
\QED
\end{proof}

\begin{proof}{Proof of Proposition \ref{prop: monotonic h}.}
Recall that 
\[
\phi_1(h;d,t)= C_1 \rrank^2 d\log^3 (d) h^{-2}\log t  \quad \text{and} \quad 
\phi_2(h;d,t) = 2\min\{8\sigmaepsilon\sqrt{2tg(h)\log(t)}, ht\}\,.
\]
For ease of notation, we let $\phi_1(h;d,t)=c_1 \log t/h^2$ and $\phi_2(h;d,t)=\min\{c_2\sqrt{t g(h)\log t},ht\}$, where $c_1$ and $c_2$ are constants that do not depend on $h$ and $t$. We prove that for any $T_1< T_2$, we have $\tilde{h}(T_1)\geq \tilde{h}(T_2)$, where $\tilde{h}(\cdot)$ is defined in Lemma \ref{prop:optimal_h}.

\underline{Scenario 1}: $\phi_1(h;d,T_1)<\phi_2(h;d,T_1)$ for any $h\in [\lowh(T_1), 2b^*]$. In this scenario, we  prove that $\phi_1(h;d,T_2)<\phi_2(h;d,T_2)$ for any $h\in [\lowh(T_1), 2b^*]$.

Fix any $h\in [\lowh(T_1), 2b^*]$. The condition $\phi_1(h;d,T_1)<\phi_2(h;d,T_1)$ implies that $c_1\log T_1/h^2< c_2\sqrt{T_1 g(h)\log T_1}$ and $c_1\log T_1/h^2< h T_1$. Thus, we have
\[
h^2\sqrt{g(h)}>\frac{c_1}{c_2}\sqrt{\frac{\log T_1}{T_1}}>\frac{c_1}{c_2}\sqrt{\frac{\log T_2}{T_2}}\,,
\]
where the last inequality holds because $T\geq 10$ and $T_1<T_2$. Similarly, we have
\[
h^3\geq c_1\frac{\log T_1}{T_1}>c_1\frac{\log T_2}{T_2}.
\]
Therefore, we have proven that $\phi_1(h;d,T_2)<\phi_2(h;d,T_2)$ for any $h\in [\lowh(T_1), 2b^*]$. According to the selection rule specified in Lemma \ref{prop:optimal_h}, we have $\tilde{h}(T_2)\leq \lowh(T_1)$.

\underline{Scenario 2}:  $\phi_1(h;d,T_1)=\phi_2(h;d,T_1)$ for  $h=\tilde{h}(T_1)\in [\lowh(T_1), 2b^*]$. In this scenario, we will show that for any $h>\tilde{h}(T_1)$, it holds that $\phi_1(h;d, T_2)<\phi_2(h;d, T_2)$. First, since $\phi_1(h;d,T)$ monotonically decreases in $h$ and $\phi_2(h;d,T)$ monotonically increases in $h$, we have $\phi_1(h;d, T_1)<\phi_2(h;d, T_1)$ for any $h>\tilde{h}(T_1)$.
Then, fix any $h\in [\tilde{h}(T_1), 2b^*]$. The condition $\phi_1(h;d,T_1)<\phi_2(h;d,T_1)$ implies that $c_1\log T_1/h^2< c_2\sqrt{T_1 g(h)\log T_1}$ and $c_1\log T_1/h^2< h T_1$. Similar to the proof for Scenario 1, we can show that $\phi_1(h;d,T_2)<\phi_2(h;d,T_2)$ for any $h\in [\tilde{h}(T_1), 2b^*]$. According to the selection rule specified in Lemma \ref{prop:optimal_h}, we have $\tilde{h}(T_2)\leq \tilde{h}(T_1)$.

\underline{Scenario 3}: $\phi_1(h;d,T_1)>\phi_2(h;d,T_1)$ for  $h\in [\lowh(T_1), 2b^*]$. In this case, $\tilde{h}(T_1)=2b^*$. Since $2b^*$ is the largest possible value of $\tilde{h}(T_2)$, we must have $\tilde{h}(T_2)\leq 2b^*=\tilde{h}(T_1)$.
\QED    
\end{proof}

\subsection{Proofs for Section \ref{subsec:bound-compare}}
\label{append:proofs-for-bound-comparison}
\begin{proof}{Proof of Proposition \ref{prop: compare i.i.d. bandits}.}
To show \eqref{prop: compare i.i.d. bandits1}, we split into two cases, (i) $T\leq d^2$ and (ii) $T\geq d^2$.

Case (i): When $T\leq d^2$, $\sqrt{T}\leq d$, hence $T\leq d\sqrt{T}$. Our algorithm is no worse than doing random sampling and the regret incurred is thus at most $O(T)$ and it is smaller than $\tilde{O}(d\sqrt{T})$ when $T\leq d^2$.

Case (ii): When $T\geq d^2$, we take $h=\rrank T^{-1/4}$ to show that our bound is no worse than $\tilde{O}(d\sqrt{T})$. Under such a choice of $h$, 
\[
\regret_T(h)=\tilde{O}(d\sqrt{T}), 
\]
since $\phi_1=\tilde{O}(d\sqrt{T})$
and $\phi_2\leq\tilde{O}(d\sqrt{T})$ due to $g(h)\leq d^2$. Note that this upper bound is rather conservative since the structure of $g(h)$ is unknown.

To show \eqref{prop: compare i.i.d. bandits2}, we ignore the order of $\rrank$ (since $\rrank \ll d$) and take $h= T^{-\jmath}$ where $\jmath$ will be specified later. Then $\phi_1$ is of order $\tilde{O}( d T^{2\jmath})$, and $\phi_2$ is bounded by  $\min\{\tilde{O}(\sqrt{Tg (T^{-\jmath})}), O(T^{1-\jmath})\}$. Thus, the regret is at most of order $\max\{ dT^{2\jmath}, T^{1-\jmath}\}$. Note that $T=d^\beta$ where $\beta\geq 1$ since $T\geq d$. Then, our regret is at most of order $\max\{d^{2\jmath\beta+1}, d^{(1-\jmath)\beta}\}$. Note that $\jmath^*= \text{argmin}_{\jmath} \max\{d^{2\jmath\beta+1}, d^{(1-\jmath)\beta}\} = \frac{\beta-1}{3\beta}$ optimizes the regret. Further, $h$ still satisfies the minimum filtering resolution requirement, i.e. $h = T^{-\jmath}\geq O(\sqrt{d/T})$, under the optimized $\jmath$, as long as $\beta\geq 1$. Thus, if we choose $\jmath = \frac{\beta-1}{3\beta}$ and the corresponding $h$, we have 
\[
\regret(h)\leq \tilde{O}(d^{\frac{2}{3} \beta+\frac{1}{3}}).
\]
When $\beta\leq 4$, it satisfies that 
\[
\regret(h)\leq \tilde{O}(d^{\frac{2}{3} \beta+\frac{1}{3}})=\tilde{O}(d^{\frac{1}{3}} T^{\frac{2}{3}})\leq \tilde{O}(d^{1+\frac{\beta}{2}})=\tilde{O}(d\sqrt{T}),
\]
where the order of $d\sqrt{T}$ is achieved if and only if $\beta=4$. When $\beta<4$, the regret order of our algorithm is $o(d\sqrt{T})$. Finally, we note that we analyze under $\beta\geq 2$ so that $d\sqrt{T}\leq T$. 
\QED
\end{proof}

\begin{proof}{Proof of Theorem \ref{theorem: g improve}.}
From Theorem \ref{thm:regret-gapind}, the regret bound with $h$ can be bounded by 
\[
\begin{aligned}
\regret_T(h)&\leq  C_1 \rrank^2 d\log^3 (d) h^{-2}\log T  + \min\{8\sqrt{2Tg(h)\log(T)}, hT\}  +C_2\\
&= C_1 \rrank^2 d\log^3 (d) h^{-2}\log T  + \min\{8\sqrt{2Td^2 h^{\gparam}\log(T)}, hT\}  +C_2.
\end{aligned}
\]
Let $\tilde{\alpha}=2\alpha \rrank^2$. By ignoring the logarithmic term, we select $h$ to minimize 
\[
\min_{h\geq \sqrt{\tilde{\alpha}\frac{d}{T}}}\  C_1 \rrank^2 d h^{-2} + \min\{8\sqrt{2Td^2 h^{\gparam}}, hT\}.
\]
We discuss two scenarios:
\begin{enumerate}
\item 
When $h^{1-\frac{\gparam}{2}}\geq 8\sqrt{\frac{2d^2}{T}}$, it holds that $8\sqrt{2Td^2 h^{\gparam}}\leq hT$, and we select $h$ such that 
\[
\min_{h\geq \sqrt{\tilde{\alpha}\frac{d}{T}}}\  d(C_1 \rrank^2  h^{-2} + 8\sqrt{2T h^{\gparam}})\,.
\]
The optimal solution without the constraint is $\tilde{h}=\left(\frac{C_1 \rrank^2}{2\gparam\sqrt{2T}}\right)^{\frac{2}{\gparam +4}}$. Then, when $\tilde{h}\geq \sqrt{\tilde{\alpha}\frac{d}{T}}$, i.e., 
\[
T\geq \threc d^{\frac{\gparam+4}{\gparam+2}}
\]
where $\threc=(\frac{2\sqrt{2}\gparam \tilde{\alpha}^{\frac{\gparam+4}{4}}}{C_1\rrank^2})^{\frac{4}{\gparam+2}}$, 
the regret is of order $\tilde{O}(d (\rrank^{\gparam} T)^{\frac{2}{\gparam+4}})$. Otherwise, when $T\leq \threc d^{\frac{4+\gparam}{2+\gparam}}$, then $h^*=\sqrt{\tilde{\alpha}\frac{d}{T}}$. In this case, the regret is of order $O(T)$.

\item
When $h^{1-\frac{\gparam}{2}}\leq 8\sqrt{\frac{2d^2}{T}}$, it holds that $hT\leq 8\sqrt{2Td^2 h^{\gparam}}$. We select $h$ to minimize
\[
\min_{h\geq \sqrt{\tilde{\alpha}\frac{d}{T}}}\  C_1 \rrank^2 d h^{-2} + hT.
\]
Without the constraint of $h$, the optimal $h^*$ is $\left(\frac{2C_1 \rrank^2 d}{T}\right)^{1/3}$, which is higher than the order of $\sqrt{\frac{d}{T}}$ under the condition that $T\geq d$. Therefore, the regret is $\tilde{O}(d^{\frac{1}{3}}  T^{\frac{2}{3}})$. 
\end{enumerate}
Therefore, combining two scenarios, when $T\geq \threc d^{\frac{\gparam+4}{\gparam+2}}$, the regret is of order
\[
\min\left\{\tilde{O}(d T^{\frac{2}{\gparam+4}}), \tilde{O}(d^{\frac{1}{3}}T^{\frac{2}{3}})\right\}.
\]
When $T\leq \threc d^{\frac{\gparam+4}{\gparam+2}}$, the regret is of order $\tilde{O}(d^{\frac{1}{3}}T^{\frac{2}{3}})$.

To simplify two scenarios, we conclude that,

-- When $T\geq \threc d^{\frac{\gparam+4}{\gparam+1}}$, the regret is of order $\tilde{O}(d T^{\frac{2}{\gparam+4}})$;

-- When $T\leq \threc d^{\frac{\gparam+4}{\gparam+1}}$, the regret is of order $\tilde{O}(d^{\frac{1}{3}}T^{\frac{2}{3}})$.

    \QED
\end{proof}

\section{Regret Analysis for \textsf{ss-LRB}}
\label{appendix: optimal submatrix}
\begin{proof}{Proof of Theorem \ref{thm: subsampling regret}.}
From Theorem \ref{theorem: g improve}, 

-- When $T\geq \threc d^{\frac{\gparam+4}{\gparam+1}}$, the regret omitting the subsampling cost is of order
$\tilde{O}( d T^{\frac{2}{\gparam+4}})\,.$
Without loss of generality, we ignore the constant term $c_s$ in the regret order. Then,  we choose $\eta$ to optimize
\[
\min_{0\leq \eta\leq 1}\ (1-\eta)^{\gamma_1}\eta^{\gamma_2} d^\iota T +\eta dT^{\frac{2}{\gparam+4}}.
\]
The first term decreases with $\eta$ while the second term increases with $\eta$.
Hence, when $d^\iota T\leq  dT^{\frac{2}{\gparam+4}}$, the optimal $\eta\leq \frac{1}{2}$ and the regret order is at least $d^\iota T$. Thus, we take $\eta$ such that $\eta d T^{\frac{2}{\gparam+4}}=d^\iota T$, which implies that $\eta=d^{\iota-1} T^{\frac{\gparam+2}{\gparam+4}}$ and the regret is of order $\tilde{O}(d^{\iota} T)$.

On the other hand, when $d^\iota T\geq  dT^{\frac{2}{\gparam+4}}$, taking $\eta$ such that $(1-\eta)^{\gamma_1}\eta^{\gamma_2} d^\iota T = d T^{\frac{2}{\gparam+4}}$ gives the regret order $\tilde{O}(dT^{\frac{2}{\gparam+4}})$.

Hence, when $T\geq \threc d^{\frac{\gparam+4}{\gparam+1}}$, the regret is of order
\[
 \tilde{O}(\min\{d^{\iota}T, dT^{\frac{2}{\gparam+4}}\}).
\]

-- When $T\leq \threc d^{\frac{\gparam+4}{\gparam+1}}$, the regret is of order
\[
\min_{0\leq \eta\leq 1}\ (1-\eta)^{\gamma_1}\eta^{\gamma_2} d^\iota T +(\eta d)^{\frac{1}{3}}T^{\frac{2}{3}}\,.
\]
To get the optimal value, we solve the optimization problem
\begin{equation}\label{eq: eta choose 2}
\min_{0\leq \eta\leq 1}\ (1-\eta)^{\gamma_1}\eta^{\gamma_2} d^\iota T +(\eta d)^{\frac{1}{3}} T^{\frac{2}{3}}\,.
\end{equation}

\begin{enumerate}
    \item When $d^\iota T \geq d^{\frac{1}{3}}T^{\frac{2}{3}}$, the optimal $\eta\geq 1/2$. We choose $\eta$ such that
    \[
    (1-\eta)^{\gamma_1}\eta^{\gamma_2} d^\iota T=d^{\frac{1}{3}} T^{\frac{2}{3}}.
    \]
    \item When $d^\iota T \leq d^{\frac{1}{3}}T^{\frac{2}{3}}$, the optimal $\eta\leq 1/2$. We choose $\eta$ such that
    \[
    (\eta d)^{\frac{1}{3}} T^{\frac{2}{3}}=d^{\iota} T.
    \]
\end{enumerate}

Therefore, we can select $\eta$ to achieve the regret order of $\min\{\tilde{O}(d^\iota T), \tilde{O}(d^{\frac{1}{3}} T^{\frac{2}{3}})\}$.

To summarize, when $T\leq \threc d^{\frac{\gparam+4}{\gparam+1}}$, the regret order is 
\[
\min \left\{ \tilde{O}(d^\iota T), \tilde{O}(d^{\frac{1}{3}} T^{\frac{2}{3}})\right\}.
\]

\QED
\end{proof}

\section{Empirical Evidence for Lemma~\ref{prop:optimal_h}}
\label{app:phi1andphi2}

We have mentioned in Section \ref{sec:phi1andphi2} that even though Lemma \ref{prop:optimal_h} analyzes the upper bounds on the regret from the two phases, the actual regret quantities from the two phases behave similarly. The regrets from the two phases are denoted as part (1) and part (2) in Appendix~\ref{SS. Regret Analysis}. We conduct the synthetic experiment with parameters $d_r = d_c = 100 = d$, $\rrank{} = 3$, $T = 2000$ (experiment parameter set b) described in Section \ref{sec:synth}) under 50 runs. From Figure \ref{fig:final_vs_h}, we can see that \textsf{LRB}(225, 1) with number of forced samples $f=225$ and filtering resolution $h=1$ gives the lowest regret 480. If we single out the curve with $f=225$ and draw part (1) and part (2) of the regret separately, we end up with the two curves in Figure \ref{fig:225_part_1_2_vs_h}. In this figure, we can see that the two curves intersect at around $h =1.4$. When $h=1.4$, the total regret is 650. It is within factor of two of the lowest regret 480, which is what Lemma \ref{prop:optimal_h} Case 2 entails.

\begin{figure}[ht]
\begin{subfigure}[b]{0.47\textwidth}
\centering
 \includegraphics[width=\textwidth]{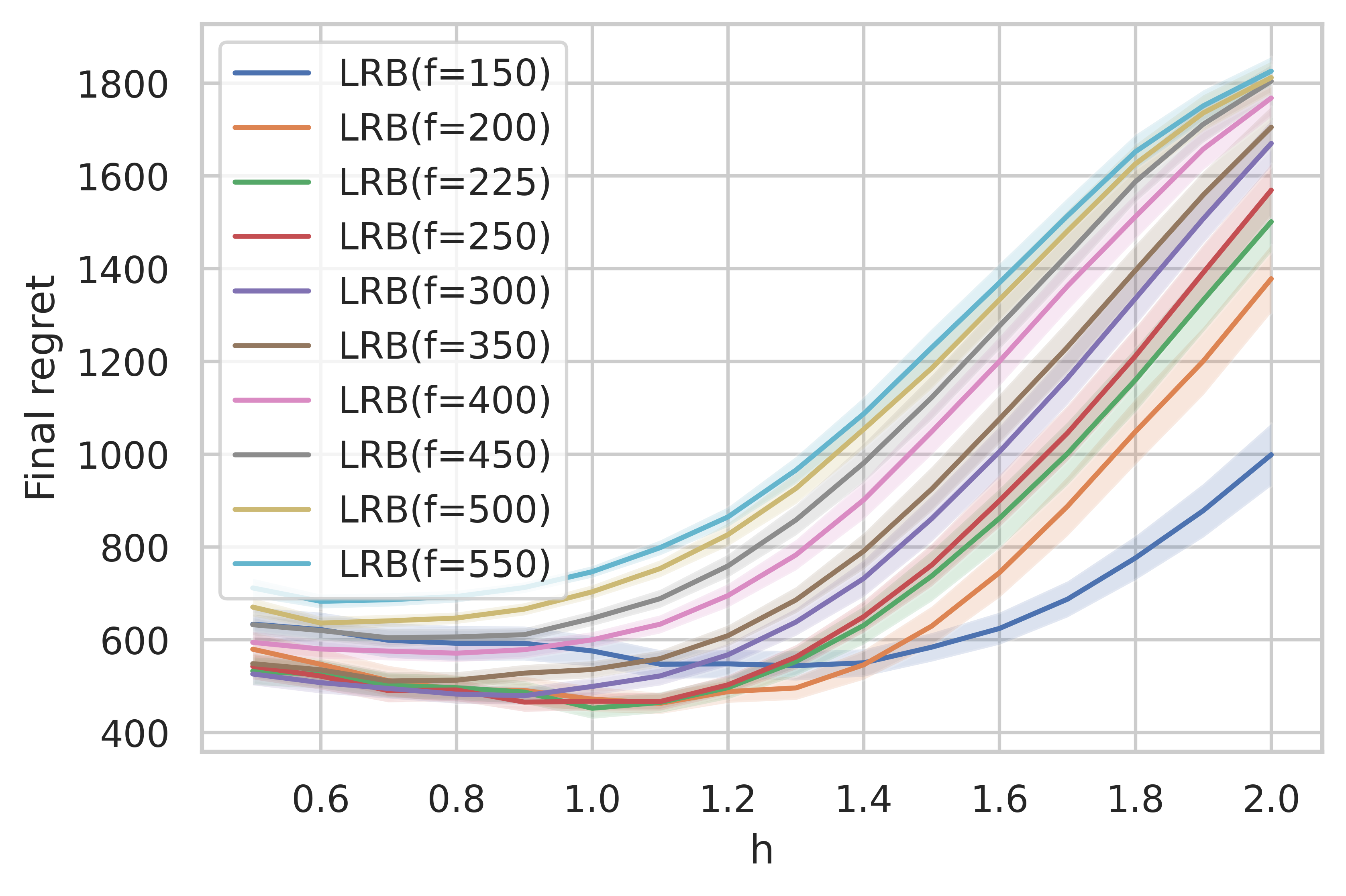}
\caption{}
 \label{fig:final_vs_h}
\end{subfigure}     
\hfill
\begin{subfigure}[b]{0.47\textwidth}
\centering
 \includegraphics[width=\textwidth]{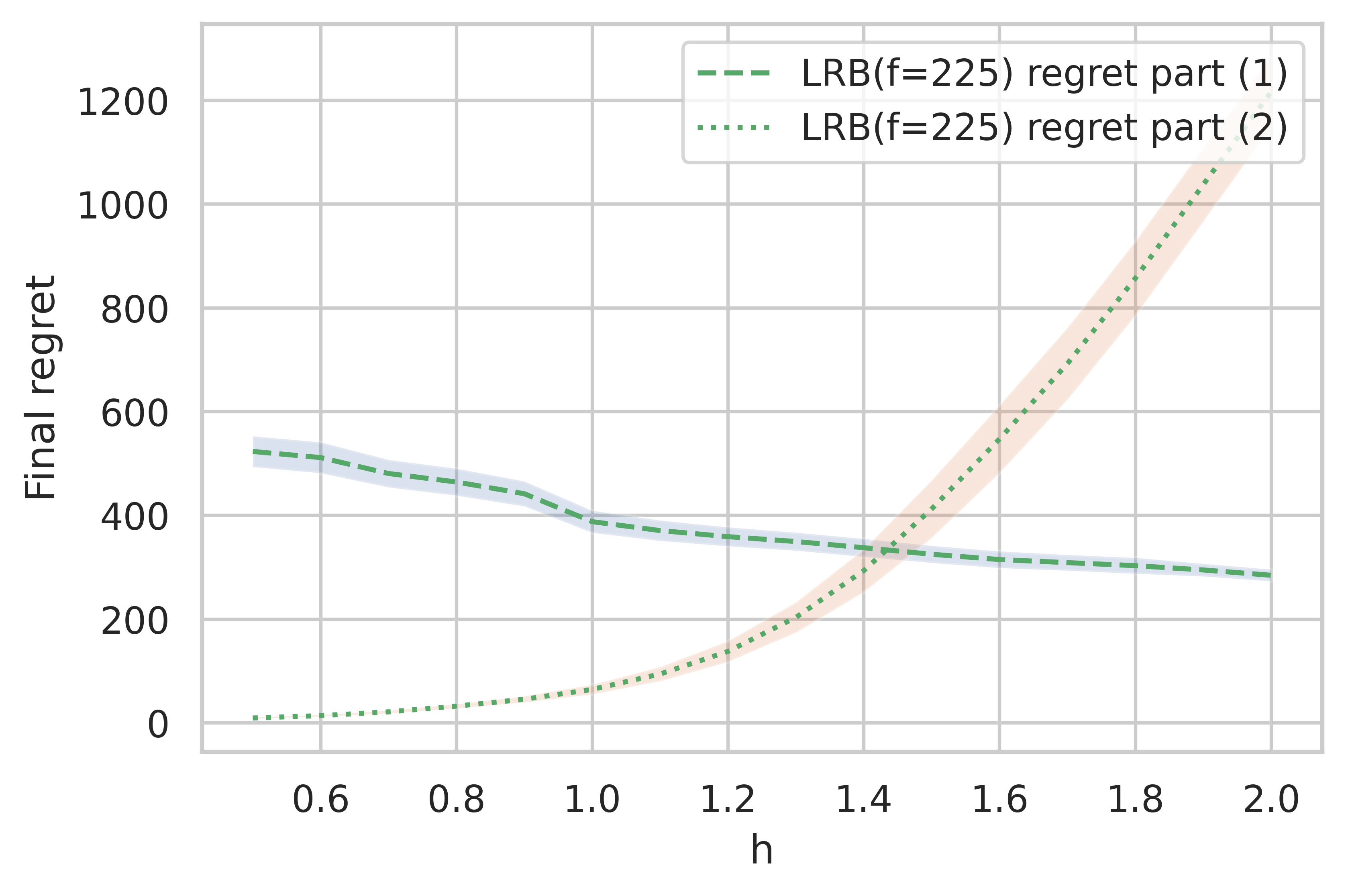}
\caption{}
\label{fig:225_part_1_2_vs_h}
\end{subfigure}
\caption{The total regret under the $h$ where the regrets from the two parts intersect is within factor of two of the minimum regret. Panel \ref{fig:final_vs_h} shows regrets of \textsf{LRB} under different combinations of number of forced samples $f$ and filtering resolution $h$; Panel \ref{fig:225_part_1_2_vs_h} shows regret part (1) and regret part (2) of \textsf{LRB}($f$=225) at different $h$.}
        \label{fig:phi1andphi2}
\end{figure}

\section{$g(h)$ and $\psi$ under Uniform Distributions}
\label{app:near-opt}

\begin{example}[Near-optimal function under uniform distribution]
\label{ex:near-opt}
If entries of $\bB^*$ are i.i.d. samples from uniform distribution over $[0,1]$, conditioning on the maximum value, the number of entries that are within $h$ of the maximum follows a Binomial distribution. We can therefore obtain the expected value of the near-optimal function by an integration which yields 
\[
\mathbb{E}[g(h)] = 1 + hd_rd_c - h^{d_rd_c}\,.
\]
\end{example}

We detail the derivation below.

Conditioned on the maximum value $z$, the number of entries that are within $h$ of the maximum follows a Binomial distribution with $h/z$ probability of success and $d_rd_c$ trials if $h<z$. If $h > z$, then all the entries will be within $h$ of the maximum. Since the probability that the maximum value is no bigger than $z$ is $z^{d_rd_c}$, then the probability density function for $z$ is $d_rd_c z^{d_rd_c-1}$. The expectation is calculated the following way:

\begin{align*}
  \mathbb{E}[g(h)] = &\int_{0}^{h} d_rd_c \cdot d_rd_c \cdot z^{d_rd_c -1}dz + \int_{h}^1 [1+ (n-1)\frac{h}{z}]\cdot d_rd_c \cdot z^{d_rd_c -1}dz \\
  =& 1+ hd_rd_c - h^{d_rd_c}.
\end{align*}

We provide closed form of the subsampling cost function when the arm reward means are independent and identically generated from common distributions such as exponential, uniform and gaussian. 

\begin{example}[Exponential distribution]\label{ex: exponential} 
For a matrix whose entry is sampled from an exponential distribution with mean 1, the median of the maximum value is $\log(d_r d_c)$. Therefore, the subsampling cost is
\[
\psi(\eta d_r, \eta d_c; d_r, d_c)= \log\left(\frac{d_r d_c}{\eta^2 d_r d_c}\right) = -2 \log(\eta).
\]
\end{example}

\begin{example}[uniform distribution]\label{ex:uniform-subsampling-cost}
For a matrix whose entry is sampled uniformly from [0,1], the expectation of the maximum value is $\frac{d_rd_c}{1+d_r d_c}$. Therefore, the cost of subsampling $\eta$-proportion of rows and columns, where $\eta$ is between (0,1], is
\[
\begin{aligned}
&\psi(\eta d_r, \eta d_c; d_r, d_c)=  \frac{d_r d_c}{1+ d_r d_c} - \frac{\eta^2 d_r d_c}{1+ \eta^2 d_r d_c} \\
\approx &(\eta^2 d_r d_c)^{-1} - (d_r d_c)^{-1} = (\eta^2 d_r d_c)^{-1}(1-\eta^2 ).
\end{aligned}
\]
\end{example}
Note that the approximation takes place when the submatrix is very close to the whole matrix.

\begin{example}[Gaussian distribution]
For a matrix whose entry is sampled from a standard Gaussian distribution, the expectation of the maximum value is $\log(d_r d_c)$. Therefore, the subsampling cost is
\[
\psi(\eta d_r, \eta d_c; d_r, d_c)=\sqrt{\log(d_r d_c)} - \sqrt{\log(\eta^2 d_r d_c)}.
\]
\end{example}

\section{Empirical $g(h)$ and Empirical $\psi$ for Low-Rank Matrices}
\label{sec:ghandpsi}
We generate the low-rank matrices $\bB =\bU\bV^{\top}$ where the entries of $\bU, \bV\in\mathbb{R}^{d\times \rrank}$ follow three different distributions: (i)\, standard Gaussian: $\mathcal{N}(0,1)$, (ii)\,Uniform: $U[0,1]$, (iii)\,Exponential: $\text{exp}(1)$. Here, we take $d=100$ and $\rrank = 3$.

\textbf{Empirical Near-Optimal Function $g(h)$.} We plot the square root of the near-optimal function $g(h)$ and fit an upper bound for it to show Definition~\ref{def:gh} holds in a log-log plot in Figure \ref{fig:g}. To illustrate, we take entries of $\bU, \bV$ following standard Gaussian in Figure \ref{fig:g_gaussian} as an example. When $d=100$, we plot $\log h$ versus $\log\widehat{\mathbb{E}}\sqrt{g(h)}$ with the blue solid line. The dashed red line is $\zeta/2\cdot \log(h) + \log d$ where $\zeta=3.4025$, which shows that $\log\widehat{\mathbb{E}}[\sqrt{g(h)}]\leq \zeta/2\cdot \log(h) + \log d$, i.e., $\widehat{\mathbb{E}}[\sqrt{g(h)}]\leq dh^{\zeta/2}$, which shows that Definition \ref{def:gh} holds. Likewise, when $m =70$, we take a submatrix of size $m\times m$ and then plot $\log h$ versus $\log\widehat{\mathbb{E}}_{|\mathcal{I}_r| = |\mathcal{I}_c|=m}\sqrt{g(h;\mathcal{I}_r, \mathcal{I}_c)}$ with the yellow solid line. The yellow dashed line is $\zeta/2\cdot \log(h) + \log m$ where $\zeta = 3.1906$, which shows that $\log\sqrt{g(h)}\leq \zeta/2\cdot \log(h) + \log m$, i.e., $\widehat{\mathbb{E}}_{|\mathcal{I}_r| = |\mathcal{I}_c|=m}\sqrt{g(h;\mathcal{I}_r, \mathcal{I}_c)}\leq mh^{\zeta/2}$. The lines are averaged over 100 trials. We summarize different $\zeta$ under different distributions for $d=100$ and different submatrix sizes $m = 30, 50, 70$ in Table \ref{tb:zeta}.

\begin{figure}[h] % "h" here means to place the figure here, where it's being declared
    \centering % Center the figure
    
    % Subfigure 1
    \begin{subfigure}[b]{0.28\textwidth} % "b" for bottom alignment, width set to 30% of text width
        \includegraphics[width=\textwidth]{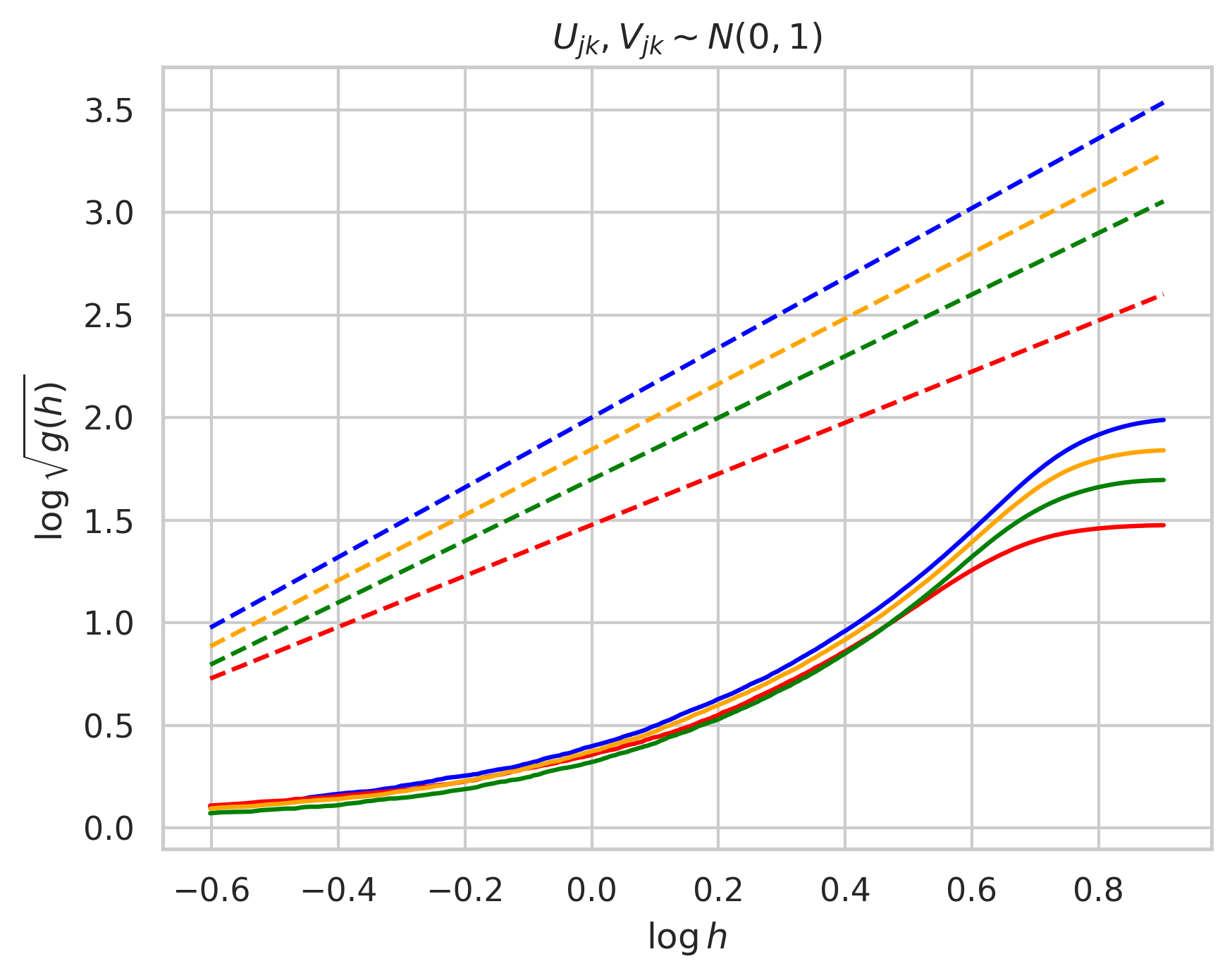}
        \caption{$\bU_{jk}, \bV_{jk}$ follow $\mathcal{N}(0,1)$}
        \label{fig:g_gaussian}
    \end{subfigure}
    \hfill 
    \begin{subfigure}[b]{0.28\textwidth}
        \includegraphics[width=\textwidth]{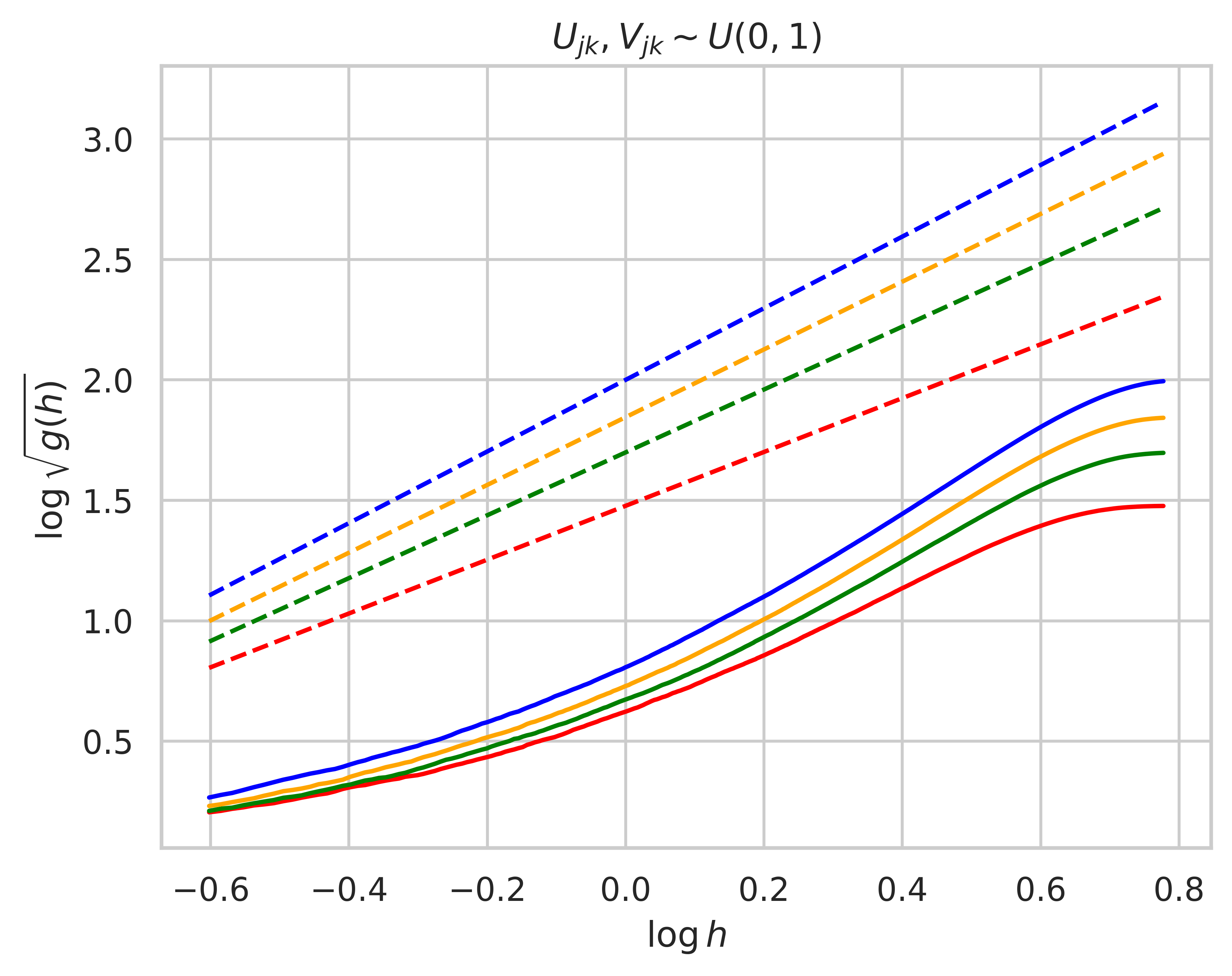}
        \caption{$\bU_{jk}, \bV_{jk}$ follow $U[0,1]$}
        \label{fig:g_uniform}
    \end{subfigure}
    \hfill % Fill out the horizontal space
    % Subfigure 3
    \begin{subfigure}[b]{0.4\textwidth}
        \includegraphics[width=\textwidth]{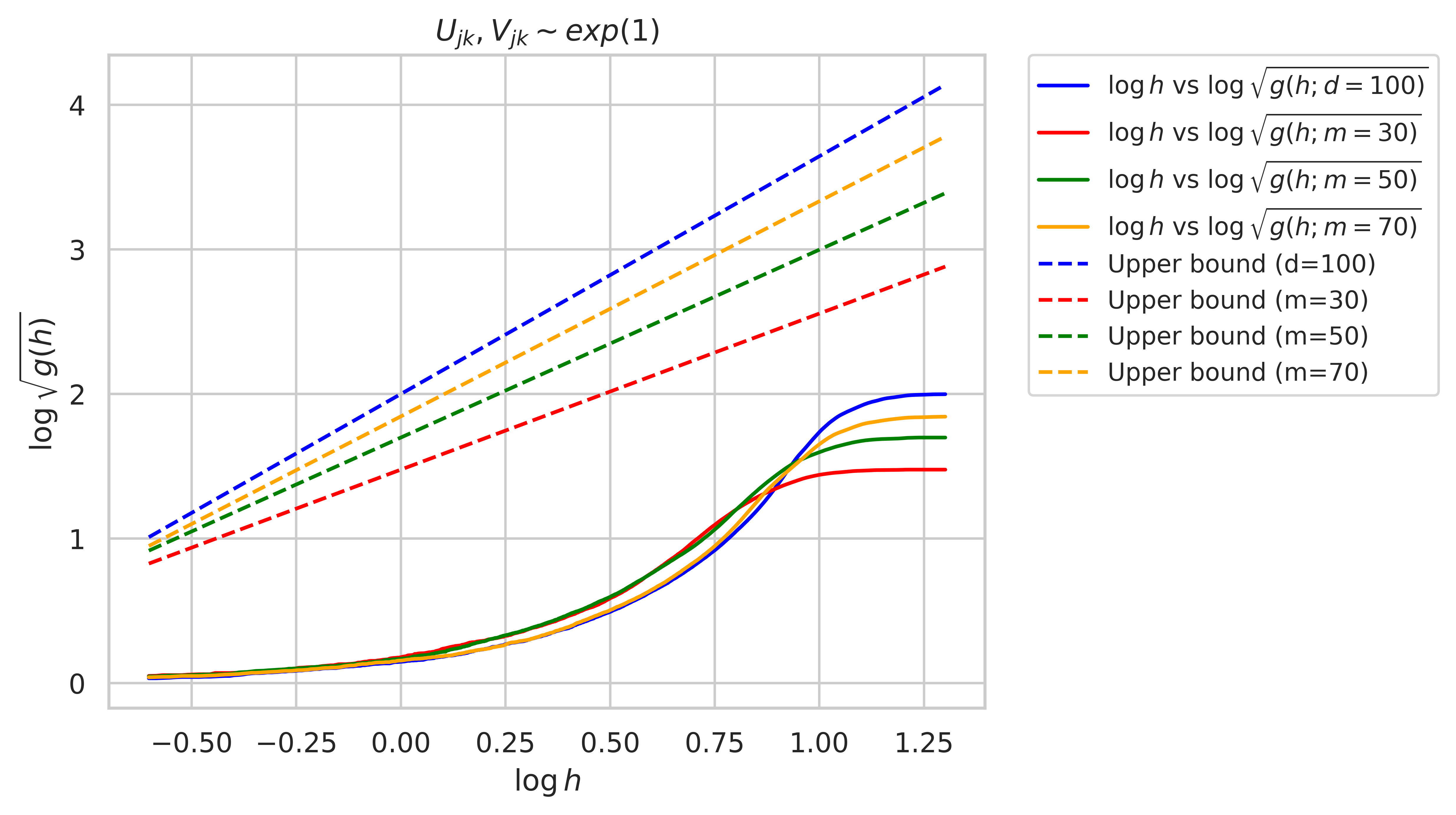}
        \caption{$\bU_{jk}, \bV_{jk}$ follow $\text{exp}(0,1)$}
        \label{fig:g_exp}
    \end{subfigure}

    \caption{Log-log plots of $\sqrt{g(h)}$ against $h$ in solid lines, and the upper bound $dh^{\zeta/2}$ described in Definition~\ref{def:gh} in dashed lines, where values of $\zeta$ can be found in Table \ref{tb:zeta}. We consider a matrix of size $d_r\times d_c$ and rank $\rrank$, where $d=100, \rrank=3.$} % Main caption for all subfigures
    \label{fig:g} % Main label for referencing the entire figure
\end{figure}

\begin{table}[ht]
\centering
\begin{tabular}{|c|c|c|c|c|}
\hline
 & $d=100$ & $m=70$ & $m=50$ & $m=30$ \\ \hline
$U_{jk}, V_{jk}\sim \mathcal{N}(0,1)$ & 3.4025 & 3.1906  & 3.0020 & 2.4897  \\ \hline
$U_{jk}, V_{jk}\sim U[0,1]$ &  2.9731 & 2.8125  & 2.6097 & 2.2344\\ \hline
$U_{jk}, V_{jk}\sim \text{exp}(1)$ & 3.2887 & 2.9764 & 2.5981 & 2.1589 \\ \hline
\end{tabular}
\caption{$\zeta$ in Definition \ref{def:gh} for different families of low-rank matrices $\bB = \bU\bV^{\top}$ where $\bU, \bV\in \mathbb{R}^{100\times 3}$.}
\label{tb:zeta}
\end{table}

\textbf{Empirical Subsampling Cost $\psi$.} We plot the subsampling cost function $\psi$ against the subsampling ratio $\eta$ and matrix dimension $d$, and fit a function to show that Assumption~\ref{assumption: subsampling} holds in Figure~\ref{fig:psi}. To illustrate, we take entries of $\bU, \bV$ following standard Gaussian in Figure \ref{fig:psi_gaussian} as an example. The subsampling cost $\psi$ is plotted as a function of subsampling ratio $\eta$ and matrix dimension $d$ in the blue scatter plot. We then regress $\log \psi$ on $\log(1-\eta)$, $\log\eta$ and $\log d$ to get estimates for $c_s$, $\gamma_1$, $\gamma_2$ and $\iota$. The $R^2$ of the regression is very close to 1. We then plot $c_s (1-\eta)^{\gamma_1} \eta^{\gamma_2} d^\iota$ with the red scatter plot. We can see that the blue and red scatter plots align fairly well. The dots are average of 10000 trials. We summarize different $c_s$, $\gamma_1$, $\gamma_2$ and $\iota$ under different distributions in Table~\ref{tb:psi_param}.

\begin{figure}[h] % "h" here means to place the figure here, where it's being declared
    \centering % Center the figure
    
    % Subfigure 1
    \begin{subfigure}[b]{0.3\textwidth} % "b" for bottom alignment, width set to 30% of text width
        \includegraphics[width=\textwidth]{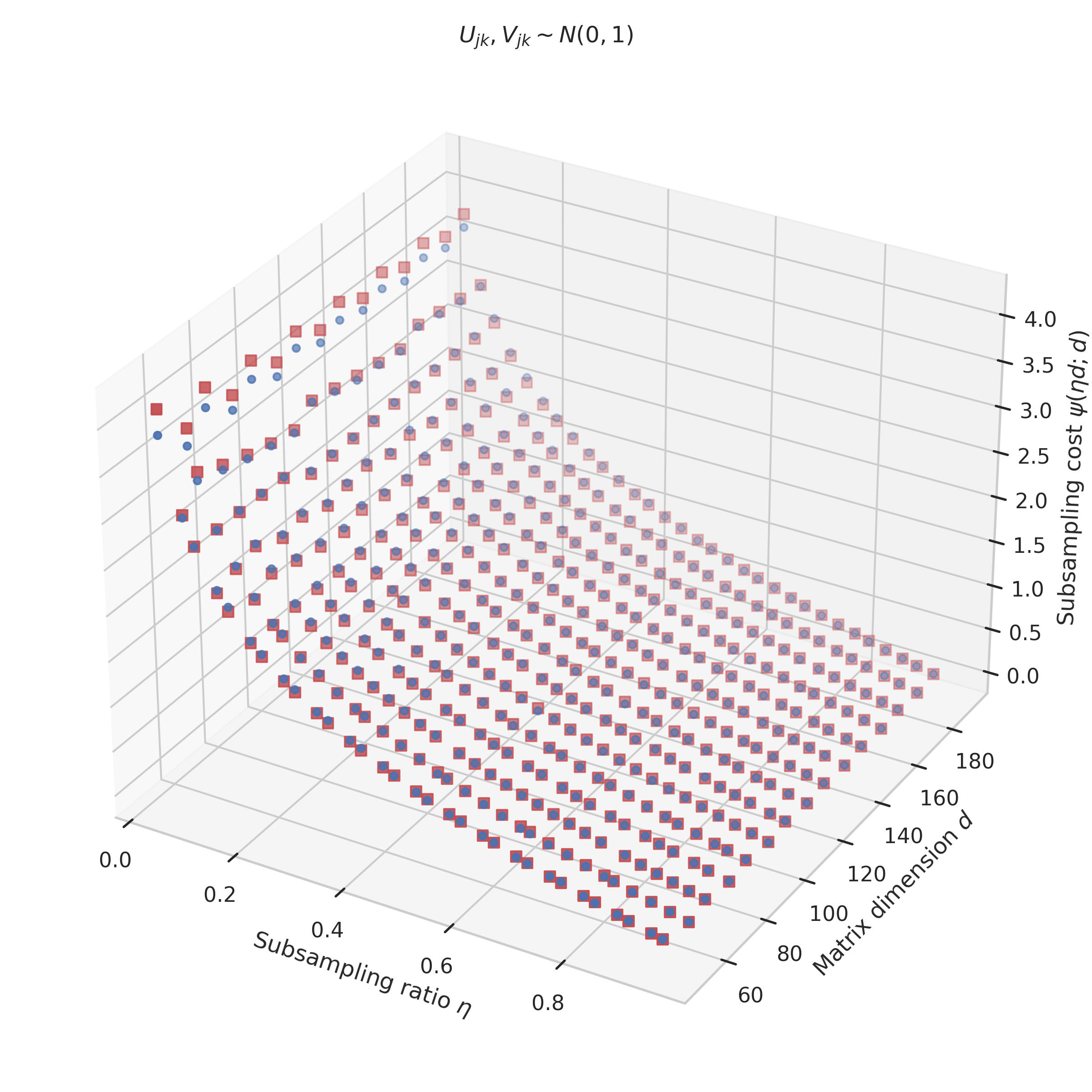}
        \caption{$\bU_{jk}, \bV_{jk}$ follow $\mathcal{N}(0,1)$}
        \label{fig:psi_gaussian}
    \end{subfigure}
    \hfill % Fill out the horizontal space to push the figures apart
    % Subfigure 2
    \begin{subfigure}[b]{0.3\textwidth}
    \includegraphics[width=\textwidth]{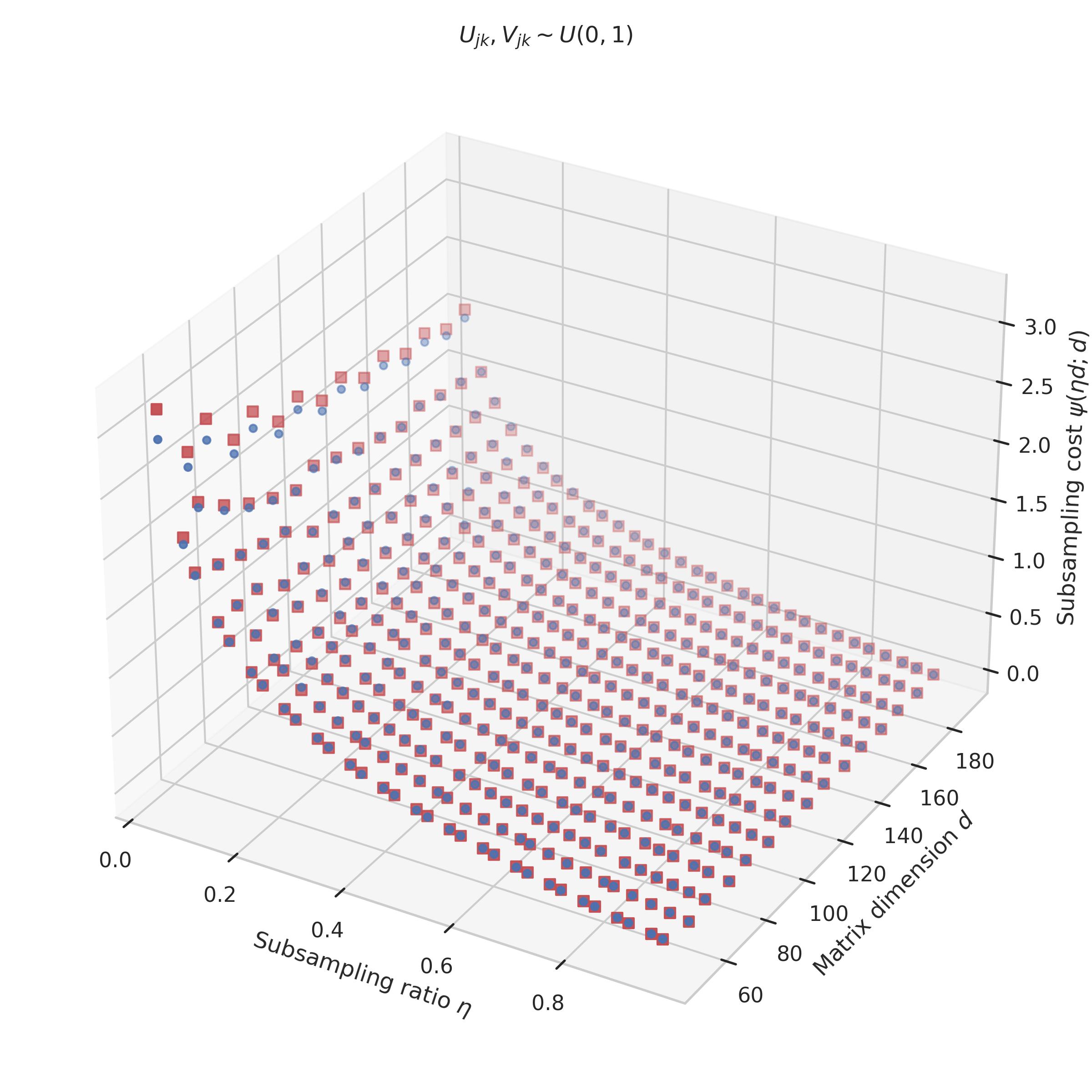}
    \caption{$\bU_{jk}, \bV_{jk}$ follow $U[0,1]$}
    \label{fig:psi_uniform}
    \end{subfigure}
    \hfill % Fill out the horizontal space
    % Subfigure 3 (with downward adjustment)
    \begin{subfigure}[b]{0.37\textwidth}
\includegraphics[width=\textwidth]{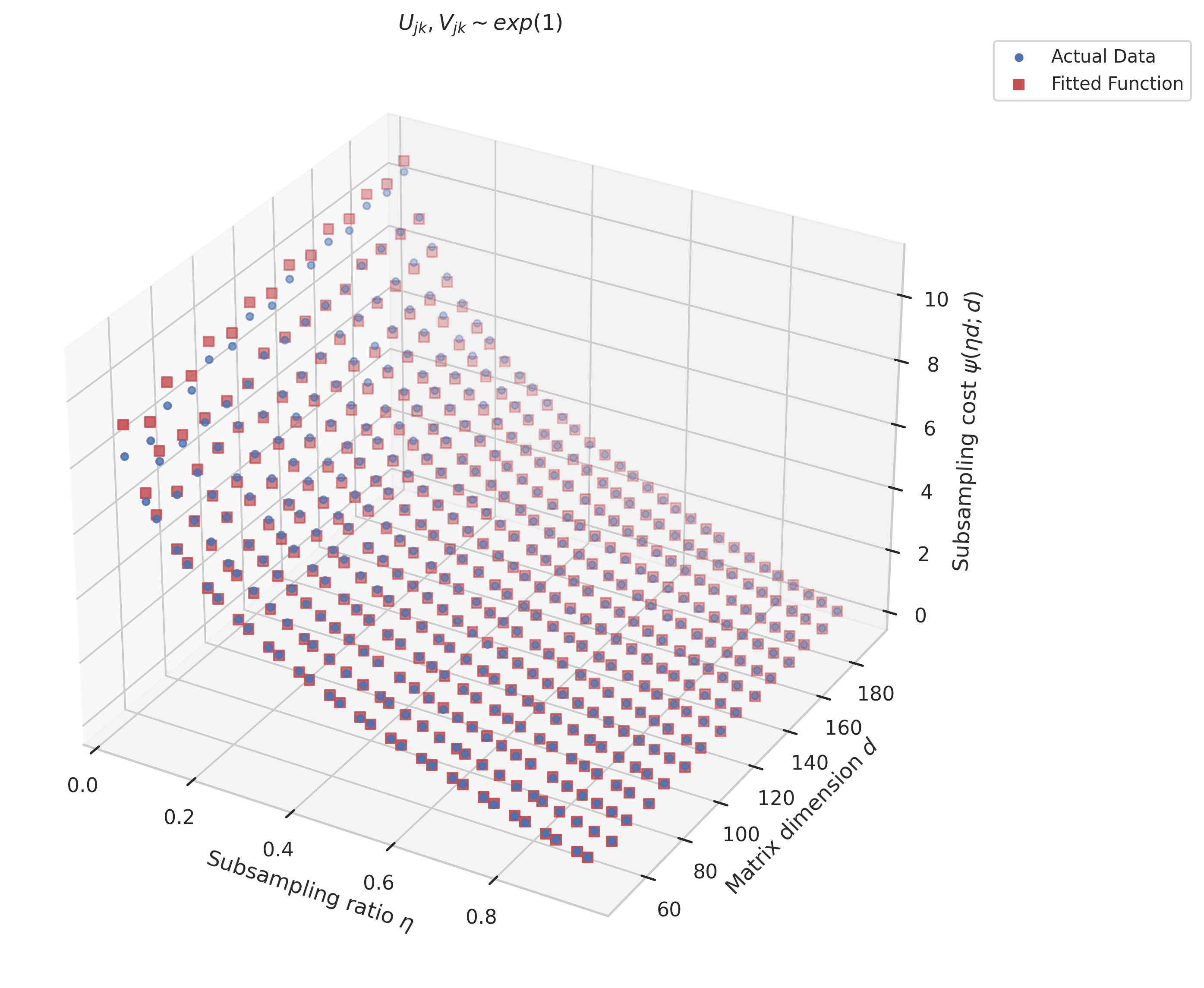}

        \caption{$\bU_{jk}, \bV_{jk}$ follow $\text{exp}(0,1)$}
        \label{fig:psi_exp}
    \end{subfigure}
\caption{Scatter plot of $\psi$ against $\eta$ and $d$ in blue dots, and the fitted function $c_s (1-\eta)^{\gamma_1} \eta^{\gamma_2} d^\iota$ presumed in Assumption~\ref{assumption: subsampling} in red dots, where values of $c_s$, $\gamma_1$, $\gamma_2$ and $\iota$ can be found in Table~\ref{tb:psi_param}. We consider a matrix of size $d_r\times d_c$ and rank $\rrank$, where $d=100, \rrank=3.$} % Main caption for all subfigures
    \label{fig:psi} % Main label for referencing the entire figure
\end{figure}
\begin{table}[h]
\centering
\begin{tabular}{|c|c|c|c|c|}
\hline
 & $c_s$ & $\gamma_1$ & $\gamma_2$ & $\iota$ \\ \hline
 $U_{jk}, V_{jk} \sim \mathcal{N}(0,1)$ & 1.3945 & 1.0478 & -0.3585 & -0.0073\\ \hline
 $U_{jk}, V_{jk} \sim U[0,1]$ & 1.9469 & 1.0491 & -0.4800 & -0.2571 \\ \hline
 $U_{jk}, V_{jk} \sim \text{exp}(1)$ & 2.0704 & 1.0513 & -0.2673 & 0.1688 \\ \hline
\end{tabular}
\caption{Estimates of $c_s, \gamma_1, \gamma_2, \iota$ in Assumption~\ref{assumption: subsampling} for different families of low-rank matrices $\bB = \bU\bV^{\top}$ where $\bU, \bV\in \mathbb{R}^{100\times 3}$.}
\label{tb:psi_param}
\end{table}

\section{Experiment Details}
\subsection{Experiment Details in Figure~\ref{fig:upper_bound_comparison} and Additional Experiments}
\label{app:visual_additional}
\paragraph{Data Generation Process.} We set the rank of the underlying reward matrix $\bB^*$ to be $\rrank=3$. We use the underlying model $\bB^* = \bU\bV^{\top}$, where $\bU$ and $\bV$ are random matrices of size 100 by 3 with entries drawn independently from $\mathcal{N}(0,1)$, to calculate the empirical near-optimal function $\mathbb{E}_{|\mathcal{I}_r|=|\mathcal{I}_c|=m}[\sqrt{g(h;\mathcal{I}_r,\mathcal{I}_c)}]$ and the subsampling cost function $\psi(m;d)$. We normalize the entries so that $\|\bB^*\|_{\infty}$ is bounded by 1. We set the smallest $h$ following the recipe in \eqref{eq: lowh value} by taking its dominating terms. The curves are averaged over 1000 trials.

In Figure~\ref{fig:app_upper_bound_comparison}, we plot the regret upper bounds under additional time horizons: $T=500$ and $T=1000$. We can see that the benefits \textsf{LRB} brings in over \textsf{Low-rank ETC} become smaller as the time horizon becomes shorter; eventually, they coincide. In both horizons, \textsf{ss-LRB} no longer touches \textsf{LRB}, showing that it is better to subsample for all range of $h$. 
\begin{figure}[H]
    \centering
    \begin{subfigure}[b]{0.45\textwidth}
    \centering
        \includegraphics[width=\textwidth]{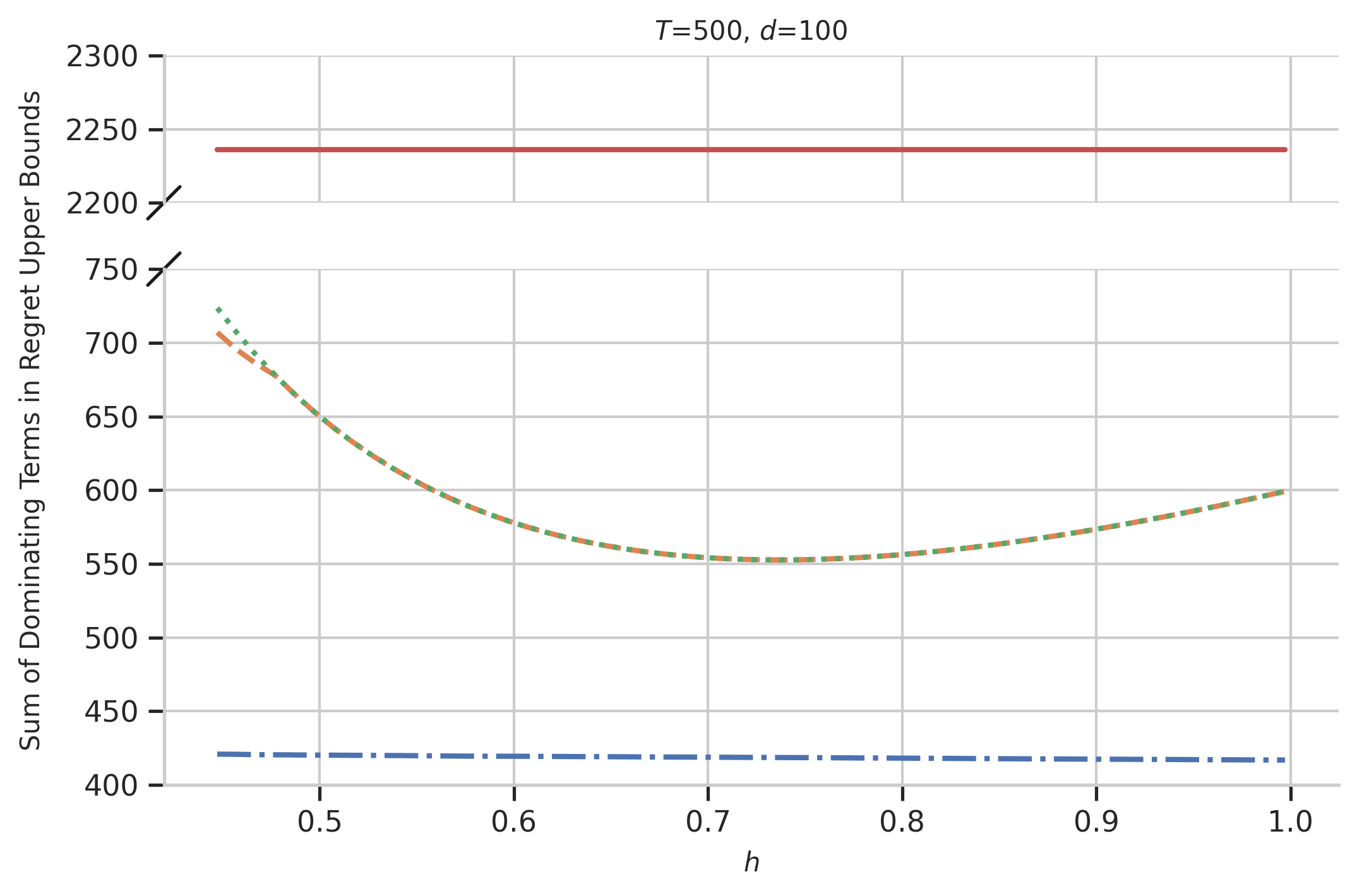}
            \caption{}
    \label{fig:upper_bound_comparison_T500}
    \end{subfigure}
    \hfill
        \begin{subfigure}[b]{0.54\textwidth}
        \centering
        \includegraphics[width=\textwidth]{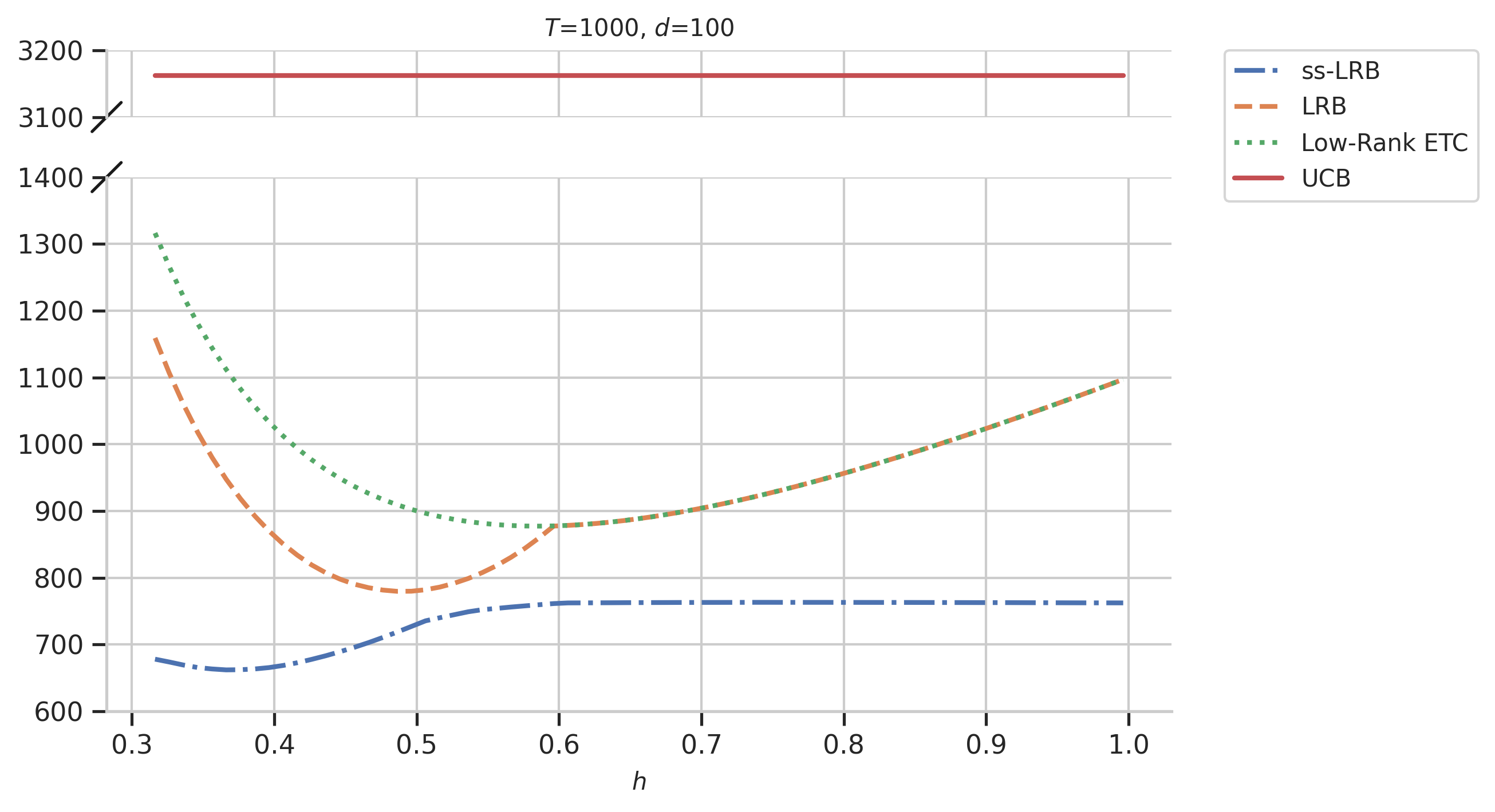}
            \caption{}
    \label{fig:upper_bound_comparison_T1000}
    \end{subfigure}
    \caption{Regret upper bound comparison among \textsf{ss-LRB}, \textsf{LRB}, \textsf{Low-rank ETC}, and \text{UCB} under different values of $h$ for \emph{additional} time horizons (a) $T = 500$ and (b) $T= 1000$.}
    \label{fig:app_upper_bound_comparison}
\end{figure}

\subsection{Experiment Details of Figure \ref{fig:synth_fig}}
\label{app:synth_param}
\paragraph{Additional Data Generation Process Details.} We use the underlying model $\bB^* = \bU\bV^{\top}$, where $\bU$ and $\bV$ are random matrices of size 100 by 3 with entries drawn independently and uniformly from $[0,1]$. We set the observation noise variance to be $ 0.1^2$.

\paragraph{Parameters.} Bandit algorithms require the decision-maker to specify a variety of input parameters that are often unknown in practice. In Figure \ref{fig:synth_fig}, we sweep a set of parameters for our algorithm as shown below (with a slight abuse of notation, we use $f$ to denote the total number of forced samples in the experiment). Note that, in these simulations, the gap between the largest arm and the second largest arm is at most 0.16, with mean 0.034 and standard error 0.003. The magnitude of $h$ is more than an order of magnitude larger, which validates the claims at the opening of Section \ref{sec:alg} about the fact that our targeted set may not necessarily contain just a single arm. For the benchmark \textsf{ss-UCB} algorithm, we use parameters suggested in computational experiments by the authors of \citep{bayati2020unreasonable} on the number of subsampled arms. We report the best regret achieved by the best $(f,h)$ combination in Figure~\ref{fig:synth_fig}.

\begin{table}[H]
    \centering
    \begin{tabular}{|c|c|p{6cm}|c|c|}
    \hline
    Time Horizon & Method & grid of $f$ & grid of $h$ & best $(f, h)$ \\
    \hline
        \multirow{2}{*}{$T=1000$} & \textsf{ss-LRB}-10 & 35, 40, 45, 50& 0.7, 0.8, 0.9, 1, 1.1 & (45, 1) \\ 
         & \textsf{ss-LRB}-20 & 70, 100, 130, 160& 0.7, 0.8, 0.9, 1, 1.1 &(70, 1) \\ 
         & \textsf{ss-LRB}-30 & 50, 100, 150, 200& 0.7, 0.8, 0.9, 1, 1.1 &(100, 0.9) \\ 
         & \textsf{ss-LRB}-40 & 100, 150, 200, 250& 0.7, 0.8, 0.9, 1, 1.1 & (100, 0.9) \\ 
         & \textsf{ss-LRB}-50 & 100, 150, 200, 250, 300&  0.7, 0.8, 0.9, 1, 1.1 & (100, 0.9) \\ 
         & \textsf{ss-LRB}-60 & 100, 150, 200, 250, 300, 350& 0.7, 0.8, 0.9, 1, 1.1 & (100, 1) \\ 
         & \textsf{LRB} & 150, 200, 225, 250, 300, 350, 400, 450, 500, 550 & 0.6, 0.8, 1, 1.2, 1.4, 1.6, 2 & (225, 1) \\
         \hline
        \multirow{2}{*}{$T=2000$} & \textsf{ss-LRB}-10 & 35, 40, 45, 50& 0.7, 0.8, 0.9, 1, 1.1 & (45, 1.1) \\
         & \textsf{ss-LRB}-20 & 70, 100, 130, 160& 0.7, 0.8, 0.9, 1, 1.1 & (100, 0.9) \\
         & \textsf{ss-LRB}-30 & 50, 100, 150, 200 & 0.7, 0.8, 0.9, 1, 1.1 & (150, 1) \\
         & \textsf{ss-LRB}-40 &  100, 150, 200, 250& 0.7, 0.8, 0.9, 1, 1.1 & (150, 0.9) \\
         & \textsf{ss-LRB}-50 &  100, 150, 200, 250, 300&  0.7, 0.8, 0.9, 1, 1.1 & (150, 1) \\
         & \textsf{ss-LRB}-60 & 100, 150, 200, 250, 300, 350& 0.7, 0.8, 0.9, 1, 1.1 & (250, 0.8) \\
         & \textsf{LRB} & 150, 200, 225, 250, 300, 350, 400, 450, 500, 550 & 0.6, 0.8, 1, 1.2, 1.4, 1.6, 2 & (225, 1)\\
        \hline
    \end{tabular}
    % \caption{Caption}
    \label{tab:synth_fig}
\end{table}

\subsection{Parameters used in Figure \ref{fig:real_figure_final_reg}}
\label{app:real_param}
As described in Section \ref{sec: data-driven heuristics}, the parameters used in Figure~\ref{fig:real_figure_final_reg} are tuned on a pre-specified grid using a related historical dataset. We vary the submatrix size $m$ among the values of 10, 15, 20, 25, 30, 35. We also test the full matrix size. We vary the filtering resolution $h$ among the values of 1, 2, 3, 4, 5. The forced-sampling sizes $f$ tested are as follows:

for $m=10$, we vary the forced-sampling size $f$ among the values of 35, 40, 45, 50, 70;

for $m=15$, we vary the forced-sampling size $f$ among the values of 50, 100, 125, 150, 175;

for $m=20$, we vary the forced-sampling size $f$ among the values of 50, 100, 150, 200, 250;

for $m=25$, we vary the forced-sampling size $f$ among the values of 50, 100, 150, 200, 250, 300;

for $m=30$, we vary the forced-sampling size $f$ among the values of 50, 100, 150, 200, 250, 300;

for $m=35$, we vary the forced-sampling size $f$ among the values of 100, 150, 200, 250, 300, 350.

The selected parameters are listed in the table below:

\begin{table}[H]
    \centering
    \begin{tabular}{|c|c|c|c|}
         \hline
         Time Horizon & $m$ & $f$ & $h$  \\
         \hline
         $T = 2000$ & 25 & 200 & 3 \\
         \hline
        $T = 3000$ & 25 & 150 & 4 \\
         \hline
        $T = 4000$ & 25 & 150 & 4 \\
         \hline
        $T = 5000$ & 25 & 150 & 4 \\
         \hline
        $T = 6000$ & 25 & 150 & 4 \\
         \hline
    \end{tabular}
    % \caption{Caption}
    \label{tab:real_fig}
\end{table}

\subsection{Real Data on Music Streaming}
\label{app:real-data}

\paragraph{Music Streaming Services.} With the development of fast internet and stable mobile connections, music entertainment industry has been reshaped by streaming services. Apps such as Spotify and Apple Music are very popular nowadays. In fact, according to Recording Industry Association of America (RIAA), the percentage of streaming in 2019 rose to 80\% of the U.S. music market in 2019, while it only accounts for 7\% in 2010. Further, paid streaming subscriptions rocketed to 611 million by the middle of 2019 from 1.5 million. Streaming has also surpassed both digital downloads and physical products and accounts for 80\% of the market \citep{RIAA}. 

\paragraph{Dataset.} We use a publicly available dataset collectively supplied by the Revenue Management and Pricing (RMP) Section of INFORMS and NetEase Cloud Music, one of the largest music streaming companies in China. The dataset contains more than 57 million impressions/displays of music content cards recommended to a random sample of 2 million users from November 1st, 2019 to November 30th, 2019. For each impression, the dataset provides the corresponding user activities, such as clicks, likes, shares, follows, whether the user commented the impression, whether the user viewed the comments or whether the user visited the creator's homepage. The dataset also contains information on each user, such as the location, age, gender, number of months registered, number of followed, and the activity intensity level. It also has information on each creator, such as gender, number of months registered, number of followers and number of followed, anonymized creator type and activity intensity level. More details on the dataset can be found in \cite{zhang2022netease}. 

\textbf{Distinction from a related formulation.}
    Note that a related setting which is often considered in bandit problems is the two-sided market where the goal is to recommend content creators to users, or to find a user cohort to recommend to when a new content creator arrives at the platform. There is usually one-sided uncertainty, such that the uncertain side is viewed as the unknown context, and the well-understood side is viewed as arms. For example, when a new user with unknown context joins, their interactions with well-understood content creators (arms) let us find a suitable content creator for the user. Similarly, when a new creator with unknown context joins, their interaction with well-understood users (arms) let us find a suitable user cohort for the new content creator. Prior papers (e.g. \cite{bastani2022learning}) have studied this problem through the lens of low-rank matrix-estimation. 
    
    In contrast, as we have mentioned in Example \ref{ex:ads}, we consider the case that an online platform wants to target advertisement campaign at one pair of user cohort and content creator segment that interact the most. The amount of interaction is used as the proxy for the rewards the advertiser can collect if the pair is targeted for the campaign.
    
    Our music streaming dataset fits this specific setting as we synthesize the scenario that the advertiser wants to target ads campaign on the music streaming platform. Hence, although this dataset can also be used to analyze the related formulation as aforementioned, our focus is not on recommending music from content creators to users, but rather to target ads campaign at a pair of user cohort and content creator segment.

\textbf{Different modes of advertising campaign.}
    There are different modes of advertising campaigns: a natural and common mode is a one-sided optimization, when advertising is driven by either user targeting (via cookie tracking on the internet) or brand spent on content (via TV advertising); and another is the mode we are considering, namely to jointly optimize over both user cohort and content type. 

In a one-sided optimization, an advertiser can specify a content type that they think is relevant based on prior experiences (e.g., the content keyword), and the platform (e.g. Youtube) decides which user cohort to target the specified content. They cannot target every cohort due to advertising budget and ads position constraints. As another setting, when the user cohort has been determined, for example, a cohort that has geographical location and/or device type (mobile/tablet/pc) specified, the platform needs to optimize for the kind of content they would like to target this cohort with. 

The mode that we have considered, namely to optimize jointly over both user cohort and content type, is needed when the advertiser does not pre-specify user cohort or content type. One industry example is a new type of product called Performance Max Campaign offered by Google. In this product, advertisers provide their budget and conversion objectives, and it is up to Google AI to decide for which user cohort (e.g., clustered by locations, activity intensity levels, interests, search keywords etc.) and content type (e.g., clustered by content creator activity intensity levels, platforms, media formats) to deliver the ads. This automatic way of doing advertising campaign will help unlock new audience for the ads and add new experience for the users. Such a joint mode has also been mentioned as a motivation in \cite{kveton2017stochastic}, where it tries to help a marketer design a campaign that finds a pair of product and user segment which maximizes the click-through-rate (CTR) of an ads.

We can also transform one-sided optimization into joint optimization. For example, we consider the scenario when the advertiser has specified a keyword for a particular content type. We can sub-divide content that belong to this type into smaller clusters and jointly optimize for the smaller clusters and the user cohorts to achieve more precise targeting. 

The full advertising campaign cycle is more intricate and sophisticated. It involves bidding, budget optimization, user clustering, content grouping and so on. We focus on one component of the process, which is when we have pre-defined user cohorts and content types, and treat the amount of interaction (e.g. CTR) as a proxy for ads revenue, as discussed in Example \ref{ex:ads}.

\subsection{Experiment Details of Section~\ref{subsec:contextual-model}}
\label{app:context}
\paragraph{Additional Data Generation Process Details.} Entries in $\bU$ and $\bV_k$ follow normal distribution $\mathcal{N}(0,1)$ independently. The observation noise $\epsilon_{t}$ is independently drawn from $\mathcal{N}(0, 0.1)$. Arguments in the context vector $X$ also follow $\mathcal{N}(0,1)$ independently.

\paragraph{Regrets of \textsf{LRB} under Different Experiment Parameters.} We tried many different combinations of the number of forced samples $f$ and the filtering resolution $h$, whose cumulative regrets (at $T=500$) are displayed in Figure \ref{fig:heatmap}. It can be seen that there is at most a factor of 2 impact on the regret of \textsf{LRB} comparing to the best performant one shown in figure \ref{fig:cumulative} with ($h$,$f$) = ($20$,$35$), as parameters $h$ and $f$ vary. In comparison to \textsf{OFUL}, such impact is negligible, which shows that our \textsf{LRB} Algorithm is robust under a wide range of parameters in terms of performing better than the benchmark \textsf{OFUL}.

\begin{figure}
     \centering
         \includegraphics[scale=0.6]{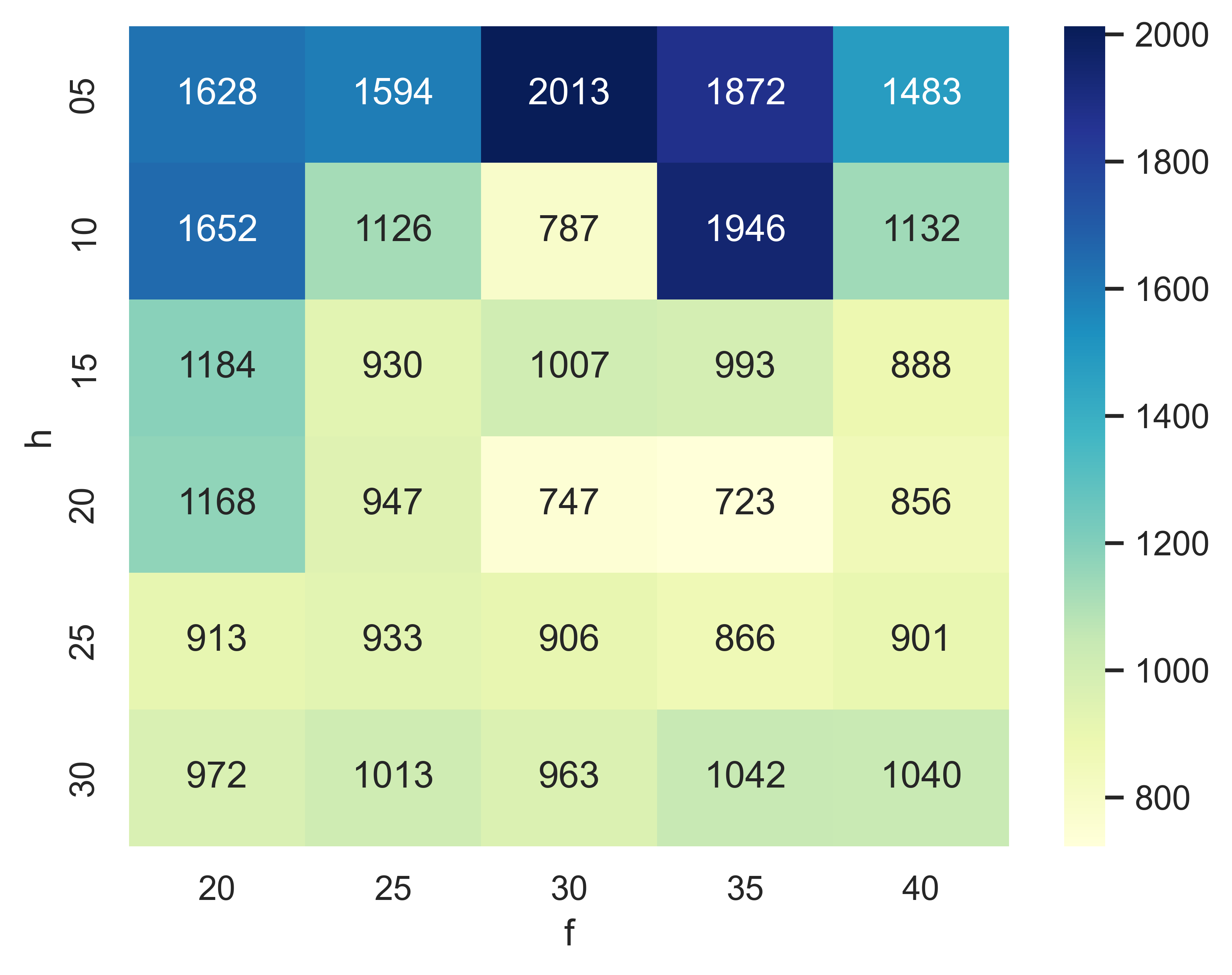}
        \caption{Average cumulative regrets of the \textsf{LRB} Algorithm \ref{alg:low-rank} at different combinations of number of forced samples $f$ and \resolution{} $h$.}
        \label{fig:heatmap}
\end{figure}

\section{Detailed Comparison against \cite{jun2019bilinear} and \cite{lu2021low}}
\label{app:comparison}

In this section, we discuss our work versus \cite{jun2019bilinear} and \cite{lu2021low}. We focus primarily on discussing \cite{jun2019bilinear}, and the comparison to \cite{lu2021low} will be analogous and discussed briefly at the end of this section. 

As an overview, \cite{jun2019bilinear} propose a two-stage algorithm ``Explore-Subspace-Then-Refine" (\textsf{ESTR}) to solve the so-called ``bilinear bandit" problem where the underlying unknown parameters of dimension $p = d^2$ form a low-rank matrix. This is a linear bandit problem that involves infinite number of contextual arms.

In a nutshell, as we have mentioned in Section \ref{sec:lit}, \cite{jun2019bilinear} consider a more general setting and their algorithm can be adapted to solve our problems, by taking their $\mathcal{X}$ and $\mathcal{Z}$ to be the set of canonical basis vectors. However, when narrowing down to the specific non-contextual setting we focus on, our algorithm has better theoretical and empirical performance. Specifically, when adapting to our non-contextual setting, their theoretical bound is $O(d^{3/2}\sqrt{T})$, which is already bigger than that of our standard non-contextual bandit benchmark, $O(d\sqrt{T})$. We have shown no worse performance against the standard non-contextual bandit benchmark for all time horizons and improvement when the time horizon is short even if we do not know the structure of the near-optimal function $g(h)$. We show the empirical comparison below where we talk about the difficulty of adapting \cite{jun2019bilinear} in the short horizons under our setting, given that their focus is on longer horizons. We leave it as future work to explore how our algorithm can be adapted to leveraging contextual information and how such an adaptation compares with \cite{jun2019bilinear} under the linear bandit setting.

\textbf{Difficulty of adapting \cite{jun2019bilinear} to short horizons.} From the algorithm described in Figure 1 in \cite{jun2019bilinear}, they consider a long-enough horizon such that they can at least pull each arm the same number of times. So they are considering a regime with $T$ at least bigger than $d^2$. Indeed, \cite{jun2019bilinear} and \cite{lu2021low} both require the number of exploration periods to be of the order $\Theta(d^{3/2}\sqrt{T})$ (Corollary 2 in the former and Theorem 4 in the latter), which exceeds the order of $T$ when $T \leq d^3$. That is, their algorithm is equivalent to random sampling and thus incur linear regret when $T\leq d^3$. In contrast, our two-phase algorithm can be fairly effective and improves from linear regret when $T$ is small as shown in the aforementioned Proposition \ref{prop: compare i.i.d. bandits} and our discussion on the subsampling version in Section \ref{sec:subsamp}.

\textbf{Empirical performance.} We show empirically in Figure~\ref{fig:synth_fig_jun} that the algorithm \textsf{ESTR-BM} proposed by \cite{jun2019bilinear} cannot emulate the performance of our algorithms in short horizons $T=1000$ and $T=2000$ under the same synthetic experiment set-up described in Section~\ref{sec:synth}. In fact, Figure~\ref{fig:jun_cumulative_regret} shows that \textsf{ESTR-BM} incurs linear regret for the short horizons we consider. We only show results averaged over 20 trials since \textsf{ESTR-BM} takes a long time to run when the matrix dimension is large.

\begin{figure}[h]    \centering
    \includegraphics[scale = 0.6]{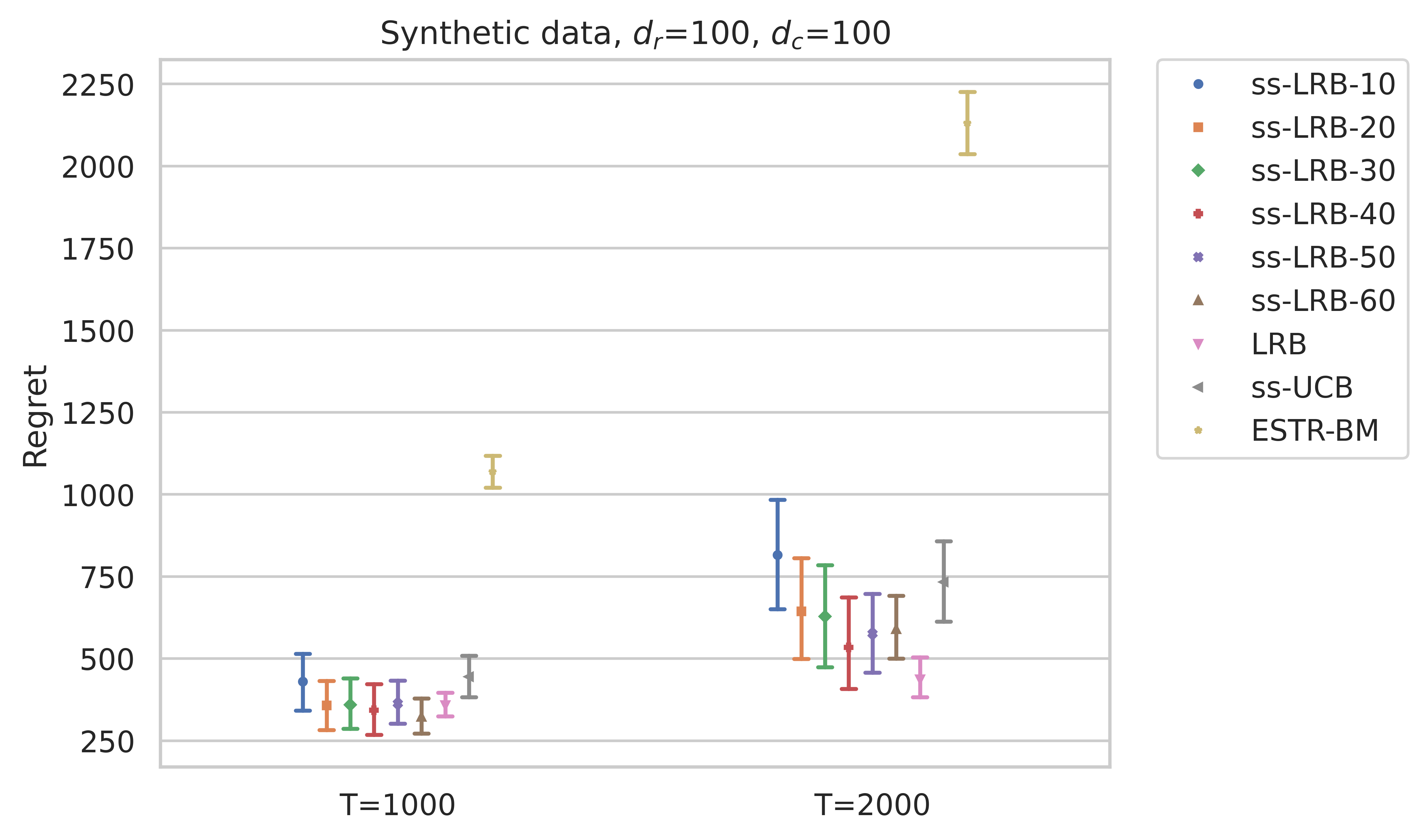}
    \caption{20 trials of the synthetic experiment described in Section \ref{sec:synth} with comparison against the \textsf{ESTR-BM} algorithm proposed in \cite{jun2019bilinear}. This experiment is previously shown in Figure \ref{fig:synth_fig} with 200 trials.}
    \label{fig:synth_fig_jun}
\end{figure}

\begin{figure}[h]
\begin{subfigure}[b]{0.4\textwidth}
\centering
 \includegraphics[width=\textwidth]{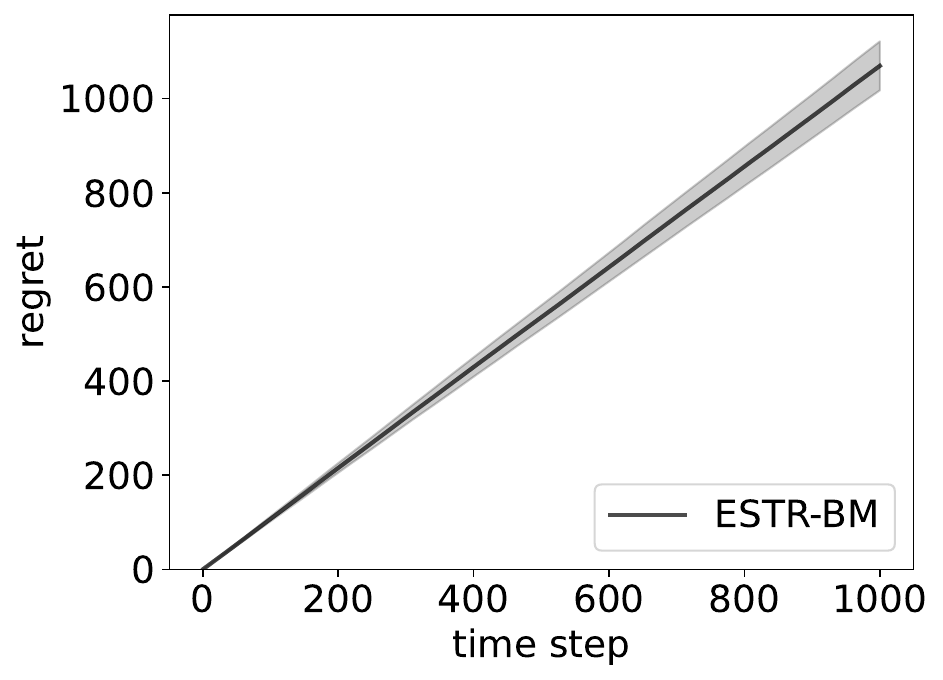}
\caption{}
 \label{fig:Jun-T1000}
\end{subfigure}     
\hfill
\begin{subfigure}[b]{0.4\textwidth}
\centering
 \includegraphics[width=\textwidth]{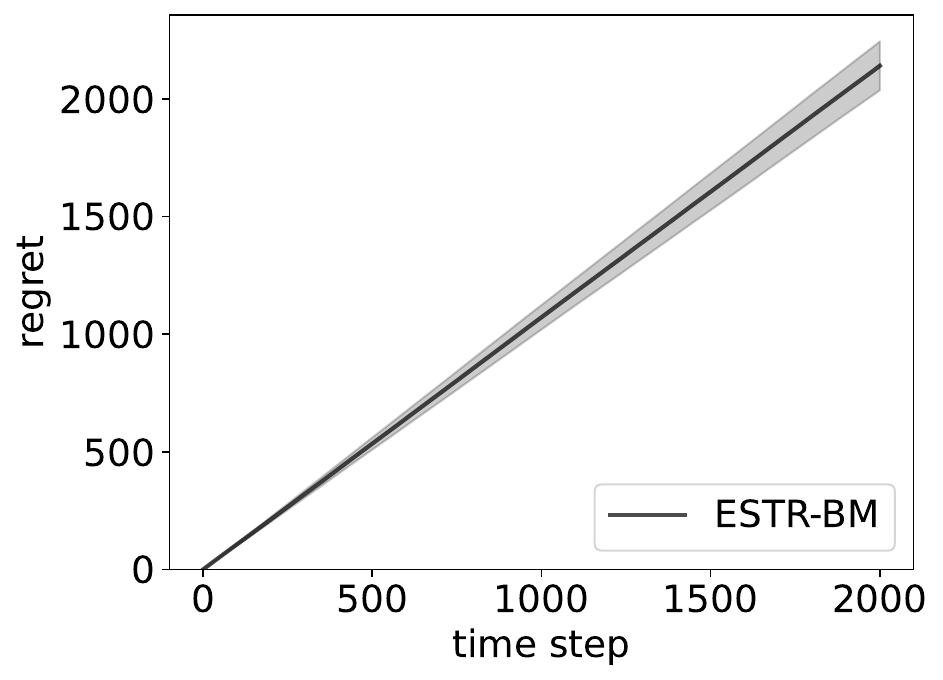}
\caption{}
\label{fig:Jun-T2000}
\end{subfigure}
\caption{Cumulative regret of the \textsf{ESTR-BM} algorithm proposed in \cite{jun2019bilinear} under total experiment time horizon $T=1000$ and $T=2000$ respectively.}
        \label{fig:jun_cumulative_regret}
\end{figure}

\textbf{Our \textsf{LRB} is a more straightforward and interpretable design for non-contextual settings.} Our way of dealing directly with observed entries helps preserve more information of the arms. To see this, we describe in more details of how their algorithm proceeds first. In their algorithm, a low-rank estimation is performed in the first stage to learn the left and right subspace spanned by the learned left and right latent feature vectors of the unknown matrix. The left and right arm sets are then projected onto the left and right subspaces respectively to obtain new arm sets, which is fed into their second-stage linear bandit algorithm called \textsf{lowOFUL} (a variant of \textsf{OFUL} that utilizes a ridge regression regularizing each coordinate of the estimation differently). Such a projection of the arm set helps reduce the dimension of the learning problem faced by \textsf{lowOFUL} from $p = d^2$ to $k = O(d\rrank)$ to speed up the learning process. The radius of the confidence ellipsoid is informed by the matrix estimation error bound and will be discussed in more details later.

Our arm elimination process is more straightforward and interpretable such that the arms in the targeted sets are not altered or rotated. Furthermore, inaccurate matrix estimation will give rise to a not-so-ideal rotated armset in their algorithm, whereas our filtering mechanism makes sure that after $O(\log T)$ forced-sample periods, the best arm is in the targeted set. The estimation error will not directly propagate to another phase since our ``targeted exploration + exploitation" phase has its own estimation subroutine. The filtering resolution $h$ allows us to accommodate low-rank estimation error by controlling the targeted set size.

\textbf{Comparison with \cite{lu2021low}.} Their work builds on \cite{jun2019bilinear} and analyzes the generalized linear bandit problem. Their result applies to more general action sets and relaxes incoherence and bounded eigenvalue assumptions made in \cite{jun2019bilinear}. However, when casting our problem as a special case of their set-up, we again run into the same argument we made when comparing to \cite{jun2019bilinear}. In particular, their regret bound is also $O(d^{3/2}\sqrt{T})$, which cannot emulate the benchmark bounds in our non-contextual setting.

\section{Benefit of Targeted Exploration + Exploitation Shown Empirically}
\label{app:benefit_filtering}

We have compared with \textsf{Low-rank ETC} in terms of the theoretical upper bounds in Remark~\ref{rmk:benefit_ucb} and Section~\ref{sec:visual} to show the benefit of targeted exploration + exploitation. To recall, \textsf{Low-rank ETC} is a naive explore-then-commit style algorithm that only undergoes the ``pure exploration" phases and uses the forced samples to construct a low-rank estimator; it then commits to the best estimated entry without targeted exploration + exploitation. We empirically show that the design of our \textsf{LRB} and \textsf{ss-LRB} is better  in Figure~\ref{fig:synth_fig_no_filtering}, by testing a wide range of ``pure exploration" phase lengths for \textsf{Low-rank ETC} to make sure that its under-performance is not due to insufficient exploration
 (that is, we let $f$ range from 100 to 1925, using a step size of 25, and we report the lowest regret).

\begin{figure}[H]    \centering
    \includegraphics[scale = 0.6]{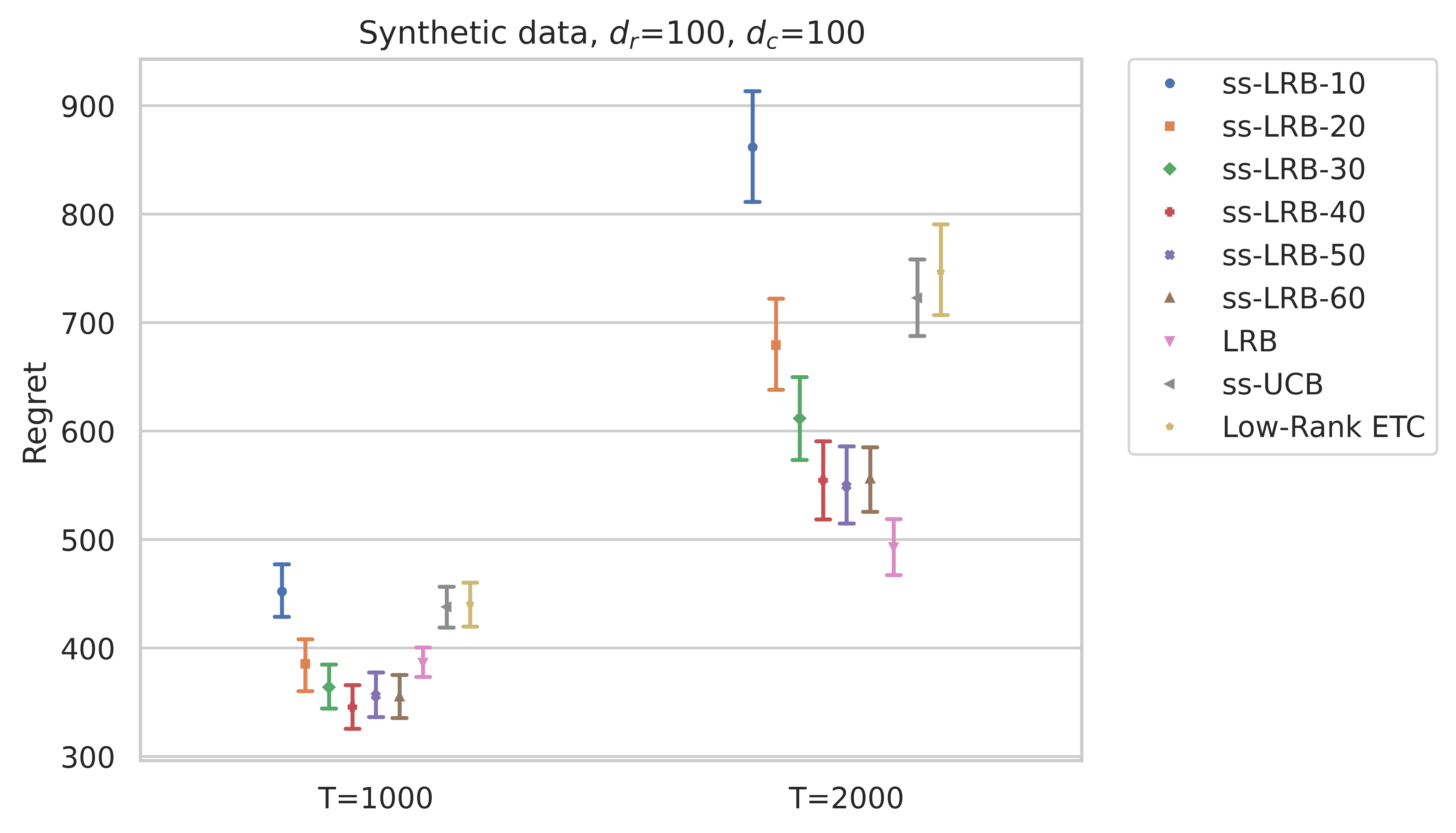}
    \caption{Synthetic experiment described in Section \ref{sec:synth} with comparison against a naive explore-then-commit algorithm \textsf{Low-rank ETC} that commits to the biggest estimated arm based on a low-rank estimation without the filtering mechanism nor the targeted exploration + exploitation. Part of this experiment is previously shown in Figure \ref{fig:synth_fig}.}
    \label{fig:synth_fig_no_filtering}
\end{figure}

Intuitively, the targeted exploration + exploitation brings in benefits since we utilize the additional information we have: the best arm is inside our targeted set parameterized by the filtering resolution $h$ (see Lemma~\ref{lemma:candidates}), but it is not necessarily estimated to be the best. As aforementioned, $h$ allows us to accommodate low-rank estimation error by controlling the targeted set size. Hence, we need to work with the targeted set given by our filtering mechanism, and not the entire arm set nor the estimated best arm based on the low-rank estimator.

\end{APPENDICES}

%%%%%%%%%%%%%%%%%
\end{document}